\documentclass{article} 
\usepackage{iclr2026_conference,times}

\usepackage{amsmath,amsfonts,bm,amsthm}

\def\eqref#1{equation~\ref{#1}}
\def\Eqref#1{Equation~\ref{#1}}

\def\1{\bm{1}}

\DeclareMathAlphabet{\mathsfit}{\encodingdefault}{\sfdefault}{m}{sl}
\SetMathAlphabet{\mathsfit}{bold}{\encodingdefault}{\sfdefault}{bx}{n}

\newcommand{\E}{\mathbb{E}}

\newcommand{\R}{\mathbb{R}}

\newcommand{\KL}{D_{\mathrm{KL}}}

\DeclareMathOperator*{\argmin}{arg\,min}

\newtheorem{theorem}{Theorem}
\newtheorem{lemma}{Lemma}

\newtheorem{proposition}{Proposition}
\newtheorem{definition}{Definition}
\newtheorem{remark}{Remark}

\newtheorem{assumption}{Assumption}

\usepackage{hyperref}
\usepackage{cleveref}
\usepackage{url}
\usepackage{multirow}
\usepackage{graphicx}
\usepackage{amssymb}
\usepackage{enumitem} 
\usepackage{tikz}
\usepackage{textgreek}
\usepackage{booktabs}
\usetikzlibrary{arrows.meta,backgrounds,positioning,calc}
\usepackage{array}
\newcolumntype{C}[1]{>{\centering\arraybackslash}p{#1}}

\newcommand{\seg}{\mathrm{seg}}
\newcommand{\Prob}{\mathbb{P}}
\newcommand{\Cseg}{\mathcal{C}}
\newcommand{\taumax}{\tau_{\max}}
\newcommand{\taumin}{\tau_{\min}}

\newcommand{\occq}{\bar\mu_T^{q}}
\newcommand{\Hclass}{\mathcal{H}}
\newcommand{\Risk}{\mathcal{R}}
\newcommand{\eRisk}{\widehat{\mathcal{R}}}
\newcommand{\norm}[1]{\lVert #1 \rVert}
\newcommand{\inner}[2]{\langle #1, #2 \rangle}
\newcommand{\bignorm}[1]{\big\lVert #1 \big\rVert}
\newcommand{\smin}{s_{\min}}
\newcommand{\dmax}{d_{\max}}
\DeclareMathOperator{\tr}{tr}
\DeclareMathOperator{\supp}{supp}
\newcommand{\Lzero}{$L_0$}

\newcommand{\Lone}{$L_1$}

\newcommand{\stp}[1]{\smallskip\noindent\textit{Step #1.}\ } 
\newcommand{\substp}[1]{\noindent\emph{#1}\ } 

\newenvironment{manualtheorem}[2]{%
  \renewcommand{\thetheorem}{\ref{#1}}
  \addtocounter{theorem}{-1}
  \begin{theorem}[#2]
}{%
  \end{theorem}
}

\newenvironment{manualproposition}[2]{%
  \renewcommand{\theproposition}{\ref{#1}}
  \addtocounter{proposition}{-1}
  \begin{proposition}[#2]
}{%
  \end{proposition}
}

\newenvironment{manualassumption}[2]{%
  \renewcommand{\theassumption}{\ref{#1}}
  \addtocounter{assumption}{-1}
  \begin{assumption}[#2]
}{%
  \end{assumption}
}

\title{Identifiability Guarantees for Drivers and Dynamics of Delayed Physical Systems}

\author{Julien Boussard$^{1,2, \dagger}$
\And
Antoine Debouchage$^3$ \\
\And
Th\'{e}o Saulus$^{2, 4}$ \\
}

\iclrfinalcopy 

\begin{document}

\maketitle

\vspace{-1cm}
\begin{center}
$^1$School of Computer Science, McGill University, Canada \quad $^2$Mila - Quebec AI Institute, Canada \quad $^3$LaMMe, Universit\'{e} \'{E}vry Paris-Saclay, France \quad $^4$ DIRO, Université de Montréal, Canada \\ 
$^\dagger$\texttt{julien.boussard@mila.quebec}
\end{center}
\vspace{0.3cm}

\begin{abstract}
Learning the dynamics of a physical system directly from observations is a key problem in natural sciences, where physical consistency and interpretability are essential. 
A wide range of methods have been proposed, including physics-informed neural networks, which are powerful but do not guarantee identifiability of the dynamics, symbolic regression, which requires a set of precomputed operations, and causal discovery, which is more principled but usually relies on strong assumptions that physical systems may violate. In this work, we develop a theory-grounded method and prove that under a set of permissive assumptions, the structural drivers and drift of stochastic delayed differential equations are identifiable. Our method outperforms others on a benchmark for driver identifiability, and on a second benchmark to evaluate physical consistency of the learned dynamics. 

\end{abstract}

\section{Introduction}

The evolution of many physical systems depends not only on their current state but also on their past states. For example, the climate system 
and ecological or neural population dynamics are driven by delayed feedbacks between variables \citep{CHEKROUN2024134058, climatemodels_sddes, neural_sddes}. 
Lagged dynamical systems explicitly model both instantaneous and time-delayed dependencies, which improves the predictive capacity and interpretability of the learned model.
Such systems are formally described by stochastic delayed differential equations (SDDEs) \citep{sfdeMohammed82, Scheutzow1984} which allow for delayed connections, as opposed to stochastic differential equations (SDEs) used to describe Markovian dynamical systems \citep{Ito44, ito1951stochastic}. 

Lagged connections in physical systems such as the climate system typically lead to oscillations and chaotic behavior, characterized by ergodicity and chaoticity \citep{GHIL20082111, RevModPhys.92.035002}. Ergodic means that long-time statistics are identical across almost all initial conditions \citep{ergodicity}, while chaotic means that an infinitesimally small difference between two initial conditions will lead to two paths diverging exponentially fast \citep{ergodicity_chaos}. 



Scientific machine learning (ML) has mostly focused on modeling SDE-driven systems. 
Neural ordinary differential equations (ODEs) \citep{neural_odes, dupont2019augmentedneuralodes} allow the direct approximation of the gradient of a physical system from irregular sparse observations \citep{zhang2020approximationcapabilitiesneuralodes}. They have been extended to capture delay \citep{eggen2023delaysdenetdeeplearningapproach, Oh_2025}, but still do not learn the long-term statistics of the system, leading to unrealistic long-horizon trajectories. 
Indeed, \citet{statistical_accurate} have shown that data-driven neural parameterizations of chaotic ordinary or partial differential equations (PDEs) fail to generalize when modeling ergodic chaotic systems, despite low test error. 
Physics-informed neural networks \citep{Yin_2021}, or symbolic regression methods \citep{feynman2, sindy} improve generalization but require a known parametric form or a set of precomputed functions and are computationally expensive. 
Motivated by the above limitations, we aim to learn the true drivers and drift of an SDDE. They must therefore be identifiable, i.e. uniquely recoverable from observations. 
Causal discovery methods have emerged as a promising avenue to infer cause–effect relationships from time-series observations and have been extended to complex, nonlinear, lagged causal structures \citep{runge2022discoveringcontemporaneouslaggedcausal, yao2022learning, lippe22citris, hickman2025causalclimateemulationbayesian, brouillard2024causalrepresentationlearningtemporal}. 
To identify the causal graph, these methods typically rely on strong constraints such as acyclicity \citep{pamfil2020dynotearsstructurelearningtimeseries} or no instantaneous connections, which physical systems often violate. 
Another approach is to embed traditional regularization and feature selection methods into Neural ODEs (NODEs, \citealp{cnode}), with identifiability guarantees derived for Markovian systems \citep{bellot2022ngm}.

In this work, we provide identifiability guarantees for ergodic, chaotic SDDE-driven systems, both at population level and in finite-sample settings. We use $L_0$ regularization \citep{louizos2018l0}, which has been proposed in the disentanglement literature \citep{lachapelle_disentanglement} and to regularize neural ODEs \citep{path_reg}. We show that under a set of permissive assumptions, it is sufficient for identifying the correct lagged causal connections and drift of SDDEs.


\textbf{Main contributions}: 
\begin{itemize} 
    \item We develop a theory-grounded method (\Lzero-Neural DDE, \Lzero-NDDE) and provide novel identifiability proofs for lagged and instantaneous drivers and drift of SDDE systems. 
    \item We demonstrate that \Lzero-NDDE outperforms other data-driven methods for identifying structural drivers of dynamical systems on a comprehensive benchmark. 
    \item We show that \Lzero-NDDE accurately captures long-term statistics on 40 simulated low-dimensional chaotic systems from the Dysts dataset \citep{gilpin2023chaosinterpretablebenchmarkforecasting}.
\end{itemize}

\section{Related work}

Several causal discovery approaches have been proposed for time-series data. 
Building upon Granger causality \citep{granger_causality}, which fails to capture nonlinear dynamics \citep{gc_nonlinear}, methods such as Neural Granger Causality (NGC) \citep{neural_granger_causality}, CUTS+ \citep{cheng2023cutshighdimensionalcausaldiscovery} or Temporal Causal Discovery Framework (TCDF) \citep{make1010019} use various deep-learning techniques to learn nonlinear autoregressive links and identify lagged causal links directly from multivariate sequences. 
Constraint-based methods such as PCMCI+ \citep{runge2022discoveringcontemporaneouslaggedcausal} or F-PCMCI \citep{castri2023enhancingcausaldiscoveryrobot} infer causal structure using conditional independence tests between lagged variables, while permutation-based methods (GRaSP, \citealp{pmlr-v180-lam22a}) search over permutations of variables and iteratively prune spurious edges to identify sparse causal graphs.
Score-based methods such as DYNOTEARS \citep{pamfil2020dynotearsstructurelearningtimeseries} infer causal relationships by optimizing a score function with differentiable constraints that ensure identifiability of the causal graph \citep{zheng2018dagstearscontinuousoptimization, brouillard2024causalrepresentationlearningtemporal, hickman2025causalclimateemulationbayesian}.
Noise-based methods (VARLiNGAM, \citealp{JMLR:v11:hyvarinen10a}; RCD, \citealp{pmlr-v108-maeda20a}) leverage the statistical independence of the noise term with respect to the inputs to identify causal connections. Finally, cross-mapping techniques discover lagged causal relationships by aligning lagged time-series (TSCI, \citealp{butler2024tangentspacecausalinference}). 

These causal methods typically discretize time and rely on strong assumptions about the causal graph (e.g., acyclicity) which often violate the nature of the continuous-time dynamics \citep{tagliapietra2025causalstructurelearningdynamical}. To overcome this issue, a growing body of work has proposed to integrate graphical modelling approaches into Neural ODEs to identify the Jacobian of the physical system. Penalized regression has been used to estimate parameters of differential equations, with consistency guarantees \citep{ramsay_param_estimation, raissi2017physicsinformeddeeplearning, wenk2019odinodeinformedregressionparameter}. More recently, \citet{bellot2022ngm} derived identifiability results for drivers of SDEs and proposed to use adaptive group lasso (AGL, \citealp{adaptive_lasso}) to select input features and identify drivers. This work has been applied to SDDEs without identifiability guarantees \citep{lagged_agl}. 

In our work, we propose to use \Lzero-induced sparsity \citep{louizos2018l0} on the input features to identify the instantaneous and lagged drivers of SDDEs. We derive novel theoretical guarantees for the identifiability of the dynamics of SDDEs, under a discrete, finite number of observations. 
\Lzero-induced sparsity has been used in causal representation learning \citep{lachapelle_disentanglement}, and for regularizing neural ODEs (PathReg, \citealp{path_reg}). 


\section{Method}
We consider the following continuous-time stochastic model, with delays $0 < \tau_1 < \dots < \tau_q$,
\begin{gather}
    dX_t = G(X_t,\, X_{t-\tau_1}, \dots, X_{t-\tau_q})\,dt + \Sigma\,dW_t, \nonumber
    \\
    G_j(x_0, x_1, \dots, x_q) \;=\; f_j\!\left(m_j^0 \odot x_0,\; m_j^1 \odot x_1, \dots,\; m_j^q \odot x_q\right),
\quad M \in \{0,1\}^{D \times D \times (q+1)}
    \label{eq:sde_model}
\end{gather}
where $W$ is a standard $D$-dimensional Brownian motion and $\Sigma \in \R^{D \times D}$ is the diffusion matrix ($\Sigma\Sigma^\top$ being the noise covariance per unit time). $M$ is a sparse binary tensor representing the dynamical drivers at timesteps $0, \tau_1, \dots, \tau_q$.
We note $Z_t := (X_t, X_{t-\tau_1},\dots, X_{t - \tau_q})$ to simplify the notation when needed.
We assume here that there is a maximum time lag $\taumax$ for the lagged interactions, which happen within the timesteps $0, \tau_1, \dots, \tau_q$. $f_j$ is the nonlinear function giving the drift (the infinitesimal conditional mean rate of change) of $X_{t, j}$ as a function of its drivers. We learn the functions $\widetilde f_j (\tilde m_j^0 \odot x_0,\; \tilde m_j^1 \odot x_1, \dots,\; \tilde  m_j^q \odot x_q) =: \widetilde G_j(x_0, \dots, x_q)$ to approximate this system.

In this section, we first show that the drivers and the dynamics, i.e. the matrix $M$ and function $G$, can in principle be recovered from a finite number of observations. We first give assumptions under which the SDDE (\ref{eq:sde_model}) admits a unique strong solution, which is a prerequisite to identifiability.
We then show that $G$ and $M$ are identifiable at the population level, either under a hard sparsity constraint (\Cref{thm:sde_identifiability}) or with an $L_0$ penalty (\Cref{thm:penalized}). Under additional assumptions, we prove that the system is geometrically ergodic (\Cref{thm:ergodic}) and provide a finite-sample identifiability theorem (\Cref{thm:finite-sample}).
Proofs are given in \Cref{app:proofs} and we provide a proof map (\Cref{fig:proof-map}) to help the reader understand the construction of the proofs.



\subsection{Existence and uniqueness of a solution}

Because the system is not Markovian, the evolution of the system depends on the whole history of length $\taumax$ rather than the current value only. This is represented by the segment space and the segment process \citep{mao2007, HMS2011}:
\[
\Cseg \;:=\; C\!\left([-\taumax, 0];\, \R^D\right), 
\qquad
X^{\mathrm{seg}}_t(s) \;:=\; X_{t+s}, 
\quad 
s \in [-\taumax, 0],
\]
where $\Cseg$ is the space of continuous functions from $[-\taumax, 0]$ to $\R^D$, equipped with the supremum norm $\norm{\varphi}_\Cseg = \sup_{s \in [-\taumax,0]} \norm{\varphi(s)}$ for any given path $\varphi \in \Cseg$.
The segment process $(X_t^{\seg})_{t \ge 0}$ is then a Markov process on the (infinite-dimensional) space $\Cseg$. 
Accordingly, the initial condition is a $\Cseg$-valued random variable $X_0^{\mathrm{seg}}$, i.e. all values observed in the previous time interval $[-\taumax,0]$.

We impose standard assumptions on the underlying process: we require that the observed history does not anticipate future noise, that the noise acts in all dimensions (non-degenerate), and that the drift maintains finite energy along the trajectory, i.e. the system has finite energy.

\begin{assumption}[Data-Generating Process]
    \label{assu:dgp}
    The true process $X$ solves the SDDE \Eqref{eq:sde_model} on $\mathbb{R}^D$: 
    \begin{itemize}[nosep, leftmargin=2em]
        \item[(i)] \textbf{Initial conditions:} The initial state $X_0^{\mathrm{seg}}$ is independent of $W$, and $\mathbb{E}\|X_0^{\mathrm{seg}}\|_\Cseg^2 < \infty$.
        \item[(ii)] \textbf{Non-degenerate noise:} The diffusion matrix $\Sigma \in \mathbb{R}^{D \times D}$ is invertible.
        \item[(iii)] \textbf{Finite energy:} For every finite $T>0$, $\mathbb{E}\int_0^T \|G(X_t, X_{t - \tau_1}, \dots, X_{t - \tau_q})\|^2\,dt < \infty$.
    \end{itemize}
\end{assumption}

To guarantee that the problem is well-posed, we need to control the drift locally and globally. Locally, it must be regular enough for trajectories to be unique, and globally it must not diverge  in finite time to infinity. We thus impose the drift to be locally Lipschitz --- a condition generally satisfied in most physical systems --- and a radial Khasminskii condition \citep{khasminskii2012} which requires that, outside a ball whose radius can be as large as desired, the drift does not push the current state away from $c$ faster than linearly and prevents explosiveness in finite time. 
It is much weaker than \emph{dissipativity}, which asks the drift to actively pull the state back towards $c$ when it is far away, like friction or damping in a physical system.

\begin{assumption}[Regularity]
    \label{assu:regularity}
    To guarantee well-posedness and stability, we assume:
    \begin{itemize}[nosep, leftmargin=2em]
        \item[(i)] \textbf{Local Lipschitzness:} $\{f_j\}_{j \leq D}$ are locally Lipschitz continuous in all of their arguments jointly.
        \item[(ii)] \textbf{Radial Khasminskii condition:} There exist constants $c \in \mathbb{R}^D$, $R > 0$, and $\beta > 0$ such that, writing $z = (x_0, x_1, \dots, x_q) \in \R^{D(q+1)}$ and $c^{(q)} := (c, \dots, c) \in \R^{D(q+1)}$ for the concatenation of $q+1$ copies of $c$, the true drift $G$ satisfies
        \begin{equation}
            \label{eq:khasminskii}
            \langle x_0-c, G(z)\rangle \le \beta(1 + \|z-c^{(q)}\|^2) \quad \text{for all } z \text{ with } \|z-c^{(q)}\| \ge R.
        \end{equation}
    \end{itemize}
\end{assumption}


These two sets of assumptions are permissive and encompass a wide range of SDDE-driven systems. We now establish the existence of a unique strong solution that does not explode in finite time, deferring the proof to \Cref{sec:proof_strongsol}.

\begin{proposition}[Non-explosion and strong solution]
    \label{prop:strong_solution}
    Under Assumptions \ref{assu:dgp} and \ref{assu:regularity}, the SDDE defined by \Eqref{eq:sde_model} has a unique strong solution for every initial condition with $\mathbb{E}\|X_0^{\mathrm{seg}} \|_\Cseg^2 < \infty$ without explosion i.e. it does not diverge to infinity in a finite amount of time. 
\end{proposition}


\subsection{Identifiability of the dynamical drivers}
\label{identifiability}

We aim to identify the drift on the regions of the state space that the process actually visits. Since the drift depends on the lagged state $Z_t$ and we observe a trajectory over $[0, T]$, the natural measure of ``where the data lives'' is the time-averaged law of $Z_t$. We thus define the \emph{joint occupation measure} of the 
states entering the drift: for a Borel set $A \subseteq \R^{D(q+1)}$ and a fixed $T > \taumax$,
\begin{equation}\label{eq:occq}
\occq(A) \;:=\; \frac{1}{T - \taumax} \int_{\taumax}^{T}
\Prob\!\left(Z_t \in A \right) dt .
\end{equation}

A parameter of a model is identifiable if it is theoretically possible to uniquely determine it from observations. We thus consider an idealized learner: its model class must be able to represent the true system and minimize the population risk (Assumption  \ref{assu:risk_minimization}). 
Although we never fully minimize the risk in practice, this is a standard assumption. 



\begin{assumption}[Model class and population risk minimization]
    \label{assu:risk_minimization}
    Let $T>\taumax$ be fixed.
    \begin{itemize}[nosep, leftmargin=2em]
        \item[(i)] \textbf{Model class:} The model class $\mathcal{H}^q$ consists of pairs $(\widetilde G, \tilde M)$ of candidate masks $\tilde{M} \in \{0,1\}^{D \times D \times (q+1)}$ and masked drifts $\widetilde G_j(x_0, x_1, \dots, x_q) = \widetilde f_j(\tilde m_j^0 \odot x_0, \tilde m_j^1 \odot x_1, \dots, \tilde m_j^q \odot x_q)$ with continuous $\widetilde f_j$,
        and the \emph{same known lags} $\tau_1 < \dots < \tau_q$ as the data-generating process.
        It contains the true pair $(G, M)$, and every $\widetilde G \in \Hclass^q$ has finite risk:
        \begin{equation}\label{eq:risk-identity}
        \Risk_T(\widetilde G) \;:=\; \E \int_{\taumax}^{T}
        \bignorm{\widetilde G(Z_t) - G(Z_t)}^2\, dt = (T - \taumax) \norm{\tilde{G} - G}_{L^2(\occq)}^2 \;<\; \infty.
        \end{equation}
        \item[(ii)] \textbf{Risk minimization:} The learned pair $(\widetilde G, \tilde M) \in \Hclass^q$ minimises $\Risk_T$ over $\Hclass^q$.
    \end{itemize}
\end{assumption}

When minimizing the risk, an input along which $G_j$ is constant could be kept or dropped without changing the drift, we therefore require every true edge to actually vary (Assumption \ref{ass:faithful}).

\begin{assumption}[Faithfulness]\label{ass:faithful} 
For every $a \in \{0, 1, \dots, q\}$ and every pair $(i,j)$ with $M^a_{ij} = 1$, there exist
$z, z' \in \supp(\occq) = \R^{D(q+1)}$ \emph{differing only in the $i$-th coordinate of the $a$-th block argument}
such that $G_j(z) \ne G_j(z')$. We show that $\supp(\occq) = \R^{D(q+1)}$ in \Cref{lem:joint-support}.
\end{assumption}

Conversely, to avoid retaining spurious edges on which $\widetilde G_j$ does not actually depend, we bound the number of learned edges. This entails that the sparsity level of the system is known a priori, a standard assumption in identifiability \citep{lachapelle_disentanglement}, and we relax it in \Cref{penalized_identifiability}. 


\begin{assumption}[Structural Sparsity]
    \label{assu:sparsity}
    The learned mask is no denser than the true mask:
    \begin{equation}
        \|\tilde{M}\|_0 \le \|M\|_0.
    \end{equation}
\end{assumption}


\begin{theorem}[Identifiability]
    \label{thm:sde_identifiability}
    Under Assumptions~
    \ref{assu:dgp}--\ref{assu:sparsity}
    , $G$ and $M$ are identifiable: 
    \begin{align}
        \widetilde G = G \quad \text{on }\R^{D(q+1)},
        \qquad
        \tilde M = M.
    \end{align}
\end{theorem}

The proof of this theorem is given in \Cref{subsec:sde_identifiability}.

\subsection{\texorpdfstring{$L_0$}{L0}-penalized identifiability}
\label{penalized_identifiability}

We now relax Assumption~\ref{assu:sparsity}: instead of a hard constraint on the number of edges, it is possible to use a penalty if dropping a true edge to minimize sparsity leads to a higher increase in risk.
Define, for each lag $a \in \{0, \dots, q\}$ and pair $(i,j)$ with $M^a_{ij} = 1$, the per-link
\emph{signal strength}
\begin{equation}\label{eq:signal}
s_{ij}^{a} \;:=\; \bignorm{\, G_j - \Pi_{-(i,a)} G_j \,}_{L_2(\occq)},
\end{equation}
where $\Pi_{-(i,a)}$ is the $L_2(\occq)$-orthogonal projection onto the closed subspace of functions
of all coordinates \emph{except} coordinate $i$ of block $a$. 
In words, $s_{ij}^a$ is the smallest error achievable by any model that ignores this input.
Set
\begin{equation}
\smin \;:=\; \min\left\{ s_{ij}^a \;:\; M^a_{ij} = 1 \right\},
\qquad
\dmax \;:=\; \max_{j} \sum_{a=0}^q \norm{m_j^a}_0 .
\end{equation}
Define, for a constant $\lambda > 0$, the penalized risk:
\begin{equation}
    \label{eq:penalized_risk}
    J_\lambda(G, M) = \mathcal{R}_T(G) + \lambda \sum_{a=0}^q \norm{M^a}_0
\end{equation}

\begin{theorem}[Identifiability with the penalized objective]\label{thm:penalized}
    Let Assumptions~\ref{assu:dgp}, \ref{assu:regularity} and \ref{assu:risk_minimization}(i) hold. Suppose $M \neq 0$ and $\smin > 0$ i.e. the system has at least one true driver, and every true driver has a measurable impact on the dynamics. This replaces Assumption~\ref{ass:faithful} (faithfulness).
    If
    \begin{equation}\label{eq:lambda-window}
    0 \;<\; \lambda \;<\; \frac{(T - \taumax) \smin^2}{\dmax},
    \end{equation}
    then every minimizer of $J_\lambda$ over $\Hclass^q$ satisfies $\tilde M = M$ and $\widetilde G = G$ \text{on }$\R^{D(q+1)}$. Moreover, every feasible pair with $\tilde M \neq M$ satisfies
    \[J_\lambda(\tilde{G}, \tilde{M}) - J_\lambda(G, M) \ge \gamma(\lambda) := \min(\lambda, (T - \taumax) s_{\min}^2 - \lambda d_{\max}) > 0\]
\end{theorem}
Note that Assumption~\ref{assu:sparsity} (sparsity constraint) is not needed here.

\subsection{Ergodicity of the true process}

The results above assume that we can compute and minimize $\Risk_T$. In practice, however, we want to learn the system from discrete observations of a single long trajectory. To this aim, we need the system to be \emph{ergodic}, so that trajectories converge to a unique statistical equilibrium.
This is the natural notion of predictability for chaotic systems: individual trajectories cannot be forecast over long horizons, but their long-run statistics are well defined and do not depend on initial conditions.

        
        

\begin{assumption}[Delay-adapted dissipativity]\label{ass:dissip}
There exist $c \in \R^D$, $R, \alpha, C_R \ge 0$ and $\beta_0, \dots, \beta_q, \varepsilon_1, \dots, \varepsilon_q \ge 0$ such that, for all $z = (x_0, \dots, x_q) \in \R^{D(q+1)}$,
\begin{align}
\inner{x_0 - c}{G(z)} &\le \alpha - \beta_0 \norm{x_0 - c}^2 + \textstyle\sum_{a=1}^q \beta_a \norm{x_a - c}^2
  && \text{if } \norm{x_0 - c} \ge R, \label{eq:delay-dissip}\\
\inner{x_0 - c}{G(z)} &\le C_R + \textstyle\sum_{a=1}^q \varepsilon_a \norm{x_a - c}^2
  && \text{if } \norm{x_0 - c} \le R, \label{eq:ball-control}\\
B &:= \textstyle\sum_{a=1}^q (\beta_a + \varepsilon_a) < \beta_0 . && \label{eq:strict-domination}
\end{align}
\end{assumption}

Far from $c$, \Eqref{eq:delay-dissip} forces the current state to be pulled back at rate $\beta_0$, while past states may push outwards at rates $\beta_a$. Near $c$, \Eqref{eq:ball-control} only forces the drift to stay radially controlled: no restoring condition is imposed on the dynamics inside the ball, which can be chaotic.
Strict domination (\Eqref{eq:strict-domination}) is the delay counterpart of stability: the restoring force of the present must dominate the destabilising influence of the past.
This assumption holds for most damped systems with a bounded attractor \citep{lorenz1963, lorenz1996}, for any bounded delayed feedback \citep{mackeyglass1977}, and for any linear delayed feedback once the instantaneous damping is superlinear (e.g. a delayed oscillator model of ENSO \citep{suarezschopf1988}).

\begin{theorem}[Ergodicity: existence, uniqueness, exponential convergence]\label{thm:ergodic}
Under Assumptions~\ref{assu:dgp}, \ref{assu:regularity} and \ref{ass:dissip}, the segment process admits a unique invariant probability measure $\pi^{\seg}$ on $\Cseg$.
Write $P_t(\varphi,\cdot):=\mathcal{L}(X_t^{\seg}\mid X_0^{\seg}=\varphi)$ and $V(\varphi):=\norm{\varphi-c}_{\Cseg}^{2}$, where $c$ denotes the constant segment with value $c$.
The invariant measure satisfies
\begin{equation}
\int_{\Cseg}\norm{\eta-c}_{\Cseg}^{2p}\,\pi^{\seg}(d\eta)<\infty,
\qquad p\ge1 \text{ an integer}.
\end{equation}
There exist $C\ge1$ and $\gamma>0$ such that, $\forall \epsilon>0$, the metric $d(\varphi,\psi):=1\wedge\epsilon^{-1}\norm{\varphi-\psi}_{\Cseg}$ satisfies
\begin{equation}\label{eq:ergodic-tv-wasserstein}
\mathcal{W}_d\bigl(P_t(\varphi,\cdot),\pi^{\seg}\bigr)
\le \norm{P_t(\varphi,\cdot)-\pi^{\seg}}_{\mathrm{TV}}
\le C e^{-\gamma t}\bigl(1+V(\varphi)\bigr),
\qquad t\ge0,\quad \varphi\in\Cseg.
\end{equation}
\end{theorem}
In words, the system effectively forgets about its initial condition, and reaches exponentially quickly a regime that is uniquely determined and well behaved.

\subsection{Finite-sample identifiability}
\label{sec:finite-sample}

This section characterizes identifiability when only a finite number of samples are available. 
We assume data is observed on a deterministic time grid $0 = t_0 < t_1 < \dots < t_n = T$, with step sizes $\Delta_k := t_{k+1} - t_k$ bounded by $\Delta_{\max} := \max_{k<n} \Delta_k \le 1$. We also assume that the lagged connections happen within the observed time steps. 
Consequently, we restrict our analysis to the effective horizon $T_\circ := T - \taumax > 0$, denoting the sampled sequence as $Z_k := Z_{t_k}$. All subsequent summations over $k < n$ apply strictly to the usable window $t_k \ge \taumax$.



\begin{assumption}[Finite-sample model class]
\label{ass:class}
Let $\Hclass_n^q$ denote the restricted class of deterministic models derived from Assumption~\ref{assu:risk_minimization}(i) satisfying the following two conditions:
\begin{itemize}
    \item[(i)] \textbf{Local finite covering:} For every radius $\rho \ge 1$ and tolerance $\epsilon > 0$, we can select a finite number of representative functions from within $\Hclass_n^q$ (an internal $\epsilon$-net) such that every function in the class is within a maximum distance of $\epsilon$ from at least one representative over the compact set $K_\rho := \{z \in \R^{D(q+1)} : \norm{z - c^{(q)}} \le \rho\}$.
    The minimal number of representatives required to cover the class this way is denoted by $N_n(\rho, \epsilon)$.
    \item[(ii)] \textbf{Stationary envelope:} The pointwise diameter of the class, defined as $\mathcal{E}(z) := \sup_{\tilde{G}_1, \tilde{G}_2 \in \Hclass_n^q} \norm{\tilde{G}_1(z) - \tilde{G}_2(z)}$, has a finite second moment under the stationary measure: $m_{\mathcal{E}}^2 := \int \mathcal{E}(z)^2 \mu_\infty^q(dz) < \infty$.
\end{itemize}
\end{assumption}

This ensures that each bounded region can be covered by a finite number of candidate functions (i), and that candidates do not lead to wildly different equilibrium (ii).

The empirical risk is defined below, and we assume that a measurable global minimizer exists:
\begin{align}
    \eRisk_n(\tilde{G}) &:= \sum_k \Delta_k \left\| \frac{X_{t_{k+1}} - X_{t_k}}{\Delta_k} - \tilde{G}(Z_k) \right\|^2 \nonumber \\
    (\widehat{G}, \widehat{M}) &\in \argmin_{(\tilde{G}, \tilde{M}) \in \Hclass_n^q} \left\{ \eRisk_n(\tilde{G}) + \lambda \sum_{a=0}^q \norm{\tilde{M}^a}_0 \right\} \label{eq:empirical-risk}
\end{align}

\begin{theorem}[Finite-sample exact support recovery ($L_0$ estimator)]
\label{thm:finite-sample}
Let Assumptions~\ref{assu:dgp}, \ref{assu:regularity}, \ref{ass:dissip} and \ref{ass:class} hold,
with $T > \taumax$. For a fixed $\delta \in (0, 1)$, there exists a bound $\varepsilon_n(\delta; \rho, \epsilon)$ dependent on two deterministic parameters $\rho$ and $\epsilon$, such that the uniform deviation is bounded with probability at least $1 - \delta$:
\begin{equation}
    \sup_{\tilde{G} \in \Hclass_n^q} \left| \left[ \eRisk_n(\tilde{G}) - \eRisk_n(G) \right] - \Risk_T(\tilde{G}) \right| \le \varepsilon_n(\delta; \rho, \epsilon),
\end{equation}
The explicit bound $\varepsilon_n(\delta; \rho, \epsilon)$ and details on the selection of $\rho$ and $\epsilon$ are given in (\ref{eq:finite-choices})--(\ref{eq:finite-error}).

Suppose further that the true mask $M \neq 0$, the minimum signal strength $\smin > 0$, and the penalty parameter $\lambda$ satisfies $0 < \lambda < T_\circ \smin^2 / \dmax$. This ensures a strictly positive margin $\gamma(\lambda) := \min\{\lambda, T_\circ \smin^2 - \lambda \dmax\} > 0$.

If the error bound is strictly less than the margin ($\varepsilon_n(\delta; \rho, \epsilon) < \gamma(\lambda)$), every measurable minimizer of the empirical objective achieves exact support recovery with probability at least $1 - \delta$:
\begin{equation}
    \Prob_\pi \left( \widehat{M}^a = M^a \text{ for all } a=0,\dots,q, \quad \norm{\widehat{G} - G}_{L^2(\mu_\infty^q)}^2 \le \frac{\varepsilon_n(\delta; \rho, \epsilon)}{T_\circ} \right) \ge 1 - \delta
\end{equation}
\end{theorem}

Setting $\lambda = \lambda_0 T_\circ$ (where $0 < \lambda_0 < \smin^2 / \dmax$) allows us to identify the mask exactly and the drift up to a given level of error from finitely many observations. Achieving consistency inherently requires joint conditions on the observation duration, mesh size, and model class complexity; increasing the number of samples alone is not sufficient. The bound $\varepsilon_n$ is explicit, but it should be read as a qualitative consistency statement:
it contains factors $1/\delta$ and a local envelope $E_\rho$ that may grow with the radius $\rho$, itself growing with $T_\circ$.

\subsection{Implementation and optimization}
\label{implementation}

The \Lzero-norm used in our identifiability theorems (\ref{thm:sde_identifiability}, \ref{thm:finite-sample}) is not differentiable. To allow for differentiable optimization, we implement the relaxed $L_0$ penalization proposed by \citet{louizos2018l0} and parametrize the matrix $M$ as the sigmoid of differentiable logit parameters. 
The functions $\widetilde f$ are parametrized with simple MLPs, with $\text{tanh}$ activation functions and an affine output layer. With discrete, regularly sampled observations, we train the model to predict $X_{t+1} - X_t$ given $X_t, X_{t-1}, ..., X_{t-\tau}$. With irregular observations, it is possible to choose any numerical integrator, as when training Neural ODEs \citep{neural_odes}. Here, we use a simple Euler integrator, as all datasets in this paper have regular discrete temporal resolutions. We standardize each dimension by the estimated gradient mean and standard deviation. Otherwise, the method may learn to predict the mean for smooth dimensions, and mask off all of their potential drivers. 
We use the Adam optimizer \citep{kingma2015adam}. We additionally propose to use gradient penalty \citep{gradient_penalty}, to regularize our network and ensure that it does not overfit to spurious correlations. 
Full implementation details are provided in \Cref{sec:implementation_details}. As the datasets considered in this paper are low-dimensional, and the architectures are relatively small, we run all experiments on CPUs. The code for implementing the methods can be found here \url{https://anonymous.4open.science/r/ODEDriverIdentifiability-C325/README.md}.

Assumption~\ref{ass:dissip} requires the true drift to be unbounded, and Assumption~\ref{ass:class} requires the drift to belong to the model class. However, MLPs with $\tanh$ activation functions are bounded by construction. To satisfy both assumptions needed for finite-sample identifiability, we parametrize $\widetilde{G}$ as $\widetilde{G}_{R}(z) = \widetilde f(z) - K \cdot(u - \Pi_{R, c}(u))$, where $R$ is an arbitrarily large radius parameter, $c$ the center of the system, $K=\mathrm{diag}(\kappa_1, \dots, \kappa_D)$ a diagonal matrix, and $\Pi_{R, c}(u)$ a projection onto the cube of center $c$ and radius $R$. In practice, we can choose $R$ large enough so that the cube encompasses all our training observations, $c$ to be equal to zero. This correction is seen as a linear restoring force outside a cube and can be applied at inference only. More details on this modification
are given in \ref{subsec:learned-ergodic}.

\section{Datasets and Experiments}
\label{experiments}

\subsection{CausalDynamics Dataset}

The CausalDynamics dataset \citep{herdeanu2025causaldynamicslargescalebenchmarkstructural} contains more than 14,000 simulated dynamical systems grouped in different classes: ``Simple'' three-dimensional chaotic dynamical systems (e.g. the Lorenz system), with or without added noise (high noise level, 2) and confounders (585 graphs); ``Coupled'' systems generated from simulated interactions between simple MLPs, of dimension 3--10,  with added noise and with or without confounders, lagged interactions, or standardization (14096  graphs); and 2 simulated ``Climate'' datasets, one corresponding to a simplified coupled ocean-atmosphere model in dimension 10, and one simulating the decoupled El-Ni\~{n}o Southern Oscillation (ENSO), an important mode of variability of the climate system (12 graphs). 

For each graph, 10 timeseries of 1000 timesteps are generated with different random seeds. The area under the receiver-operator curve (AUROC), area under the precision-recall curve (AUPRC) and structural Hamming distance (SHD) are computed for each time series, before being averaged across the system class (e.g. ``Simple, Confounder, No noise''). 
Hyperparameter search is done for all evaluated methods by maximizing AUROC on the ``Simple, No confounder, No noise, Lorenz84'' graph, which is then excluded from the results. The same set of parameters is used for evaluation on all other graphs. 
We implement \Lzero-NDDE and additional baselines: AGL \citep{bellot2022ngm, lagged_agl} and C-NODE \citep{cnode}, which uses the \Lone-penalty. We extend C-NODE to SDDEs similarly as Neural ODEs, by adding neural networks to capture lagged connections. All methods share the same core architecture. Hyperparameter search is conducted for each method separately, following the benchmark procedure. 
Details on the hyperparameter search, and the final set of parameters are reported in \Cref{sec:hyperparams}. 

\subsection{The Dysts Benchmark}
\label{subsec:dysts_benchmark}

To evaluate the quality of the learned dynamics, we use the Dysts benchmark \citep{gilpin2023chaosinterpretablebenchmarkforecasting}. It allows us to generate trajectories for low-dimensional chaotic systems (e.g. ``Lorenz'' system), which we use to evaluate \Lzero-NDDE, C-NODE, AGL, as well as PathReg, which uses \Lzero \space to regularize entire input-output paths, and regular NODEs. 
In total, we use the 40 systems for which the Jacobian is given by the dataset, which is needed for our metric computation (\Cref{subsec:dysts_metric_computation}). 
For each system, we generate one noisy training trajectory of 50 Lyapunov timesteps, with stochastic noise forcing of level 0.05, and 3 test trajectories of 250 Lyapunov timesteps, with different initial conditions. The Lyapunov timestep is the timescale over which a chaotic system loses its predictability \citep{LYAPUNOV01031992}.
These trajectories are long enough to densely cover the system's space. 
We train each method on the training trajectory, generate 3 trajectories using the 3 test initial conditions, and compare the predicted trajectories to the true unseen test trajectories for evaluation. 

We compute two metrics to evaluate short-term accuracy: the normalized root mean-square error (NRMSE) averaged over the first 5 Lyapunov timesteps, and the valid Lyapunov prediction time (VLPT), the number of Lyapunov timesteps before the predicted trajectories diverge from the true trajectory by more than 1 standardized unit. We compute two metrics to evaluate the reconstructed systems' attractor geometry and chaoticity: the first Lyapunov exponent squared error (FLEE), indicative of the level of chaoticity of the system, and the squared error between the true and estimated Kaplan-Yorke dimension ($D_{KY}$, \citealp{dky_error}), a proxy for the system's attractor dimension. We compute two metrics to evaluate the long-term statistics of the predicted trajectories: the log-spectral distance (LSD), to evaluate whether the oscillation frequencies are accurately captured, and the state-space distance ($D_{stsp}$, \citealp{dstsp}), a divergence between the true and estimated equilibrium distribution. 
We report, for each metric and method, the median value across all datasets and the average fractional rank, to assess whether one method systematically outperforms another. More details on the metric computation are given in \Cref{subsec:dysts_metric_computation}. 

We choose 5 systems at random for hyperparameter search and we select, for each method, the set of parameters that maximizes the VLPT before evaluating the method on the other systems. A detailed description of the hyperparameter search and final parameters is given in \Cref{subsec:hyperparam_search_dysts}.


\section{Results}
\label{results}

\subsection{CausalDynamics Benchmark}
\label{causaldynamicsresults}

Results on the CausalDynamics benchmark are reported in \Cref{tab:causaldynamics_results}. \Lzero-NDDE achieves lower average rank on all three metrics. It generally has higher AUROC/AUPRC than competing methods on the ``Simple'', ``Coupled'' and the ``Climate'' datasets, and consistently achieves low SHD, although the hyperparameters, only selected on the ``Simple lorenz84'' dataset, lead to too dense graphs on the ``Coupled'' datasets. 
Some results lead to ``degenerate'' metrics (e.g. AUROC of 0.50), as some methods sometimes fully sparsify or do not sparsify at all, and thus do not rank the edges. NGC and AGL achieve the lowest SHD on the MAOOAM dataset despite low AUROC, as they do not sparsify the learned graph. 
C-NODE performs well in general, despite not providing any identifiability guarantee. C-NODE also does not explicitly remove input features and even when a low weight is given to an input feature, the neural network can compensate for it. 

\begin{table*}[hbt!]\centering
\resizebox{\textwidth}{!}{
\begin{tabular}
{c c ccc ccccc cc c}
\toprule
& \multirow{2}{*}{\textbf{Experiment}}  & \multicolumn{3}{c}{\textbf{Simple}} & \multicolumn{5}{c}{\textbf{Coupled}} & \multicolumn{2}{c}{\textbf{Climate}} & \multirow{2}{*}{\textbf{Avg. rank}} \\
\cmidrule(lr){3-5} \cmidrule(lr){6-10} \cmidrule(lr){11-12}
& & Default & Confounder & Noise & Default & Noise & Confounder & Time-lag & Standardize & MAOOAM & ENSO \\
\midrule

 \multirow{9}{*}{$\vcenter{\hbox{\rotatebox{90}{\textbf{SHD}}}}$}
& PCMCI+ & 41.04 & 23.02 & 47.64 & 224.80 & 183.90 & 324.63 & 327.72 & 228.32 & 80.00 & 529.36 & 7.30 \\
& F-PCMCI & 35.30 & 21.07 & 45.09 & 192.90 & \underline{149.60} & 195.74 & 350.61 & \underline{201.79} & 130.00 & 530.27 & 5.80 \\
& NGC & 28.91 & 19.96  & 28.84 & 840.95 & 842.55 & 670.53 & 793.67 & 840.26 & \underline{31.00} & 337.09 & 7.90 \\
& CUTS+ & 48.11 & 24.04 & 61.32 & \textbf{152.00} & 150.50 & 272.68 & 247.22 & 310.63 & 130.00 & 608.73 & 7.50 \\
& RCD & 61.85 & 26.74 & 61.38 & \underline{157.05} & 155.65 & \underline{136.53} & \textbf{201.11} & \textbf{159.84} & 130.00 & 665.36 & 6.80 \\
& GRaSP & 59.04 & 27.09 & 60.18 & 215.94 & 842.55 & \textbf{136.00} & 793.67 & 840.26 & 126.00 & 666.27 & 9.70 \\
 & AGL & 26.60 & \underline{16.96} & \underline{21.28} & 319.04$^*$ & 222.41$^*$ & 246.37$^*$ & 315.54$^*$ & 312.31$^*$ & \textbf{30.00}$^*$ & 337.09 & 5.55 \\
 & C-NODE (L1) & \textbf{24.07} & 17.10 & \textbf{20.50} & 253.91$^*$ & 176.30$^*$ & 179.58$^*$ & 273.57$^*$ & 258.09 & 42.00 & \underline{335.36} & \underline{4.50} \\
& \textbf{\Lzero-NDDE (Ours)} & \underline{26.42} & \textbf{16.25} & 23.67 & 205.35 & \textbf{144.20} & 143.33 & \underline{222.34} & 212.11 & 40.00 & \textbf{325.54} & \textbf{2.50} \\

\midrule

 \multirow{9}{*}{$\vcenter{\hbox{\rotatebox{90}{\textbf{AUROC / AUPRC} }}}$} 

& PCMCI+ & .52 / .71  & .49 / .59 & .50 / .69 & \textbf{.67} / .25 & \underline{.64} / .25 & .58 / .20 & .58 / .24 & \textbf{.69} / .27 & .69 / .88 & \underline{.57 / .70} & 4.10 / 5.25 \\
& F-PCMCI & .51 / .70 & .50 / .59  & .52 / .70 & \textbf{.67} / .27 & .57 / .21 & .55 / .19 & \underline{.59} / .24 & \underline{.68} / .28 & .50 / .81 & \underline{.57 / .70} & 4.80 / 6.05 \\
& NGC & .50 / .69 & .50 / .58  & .50 / .68 & .50 / .15 & .50 / .15 & .50 / .16 & .50 / .20 & .50 / .15 & .50 / .81 & .50 / .67 & 9.85 / 10.90 \\
& CUTS+ & .50 / .69 & .50 / .58 & .50 / .68 & .50 / .15 & .50 / .15 & .49 / .16 &  .50 / .20 &  .50 / .15 &  .50 / .81 & .50 / .67 & 10.10 / 10.90 \\
& RCD & .50 / .69 & .50 / .58 & .50 / .68 & .50 / .15 & .50 / .15 & .51 / .18 & .50 / .20 & .50 / .16 &  .50 / .81 &  .50 / .67 & 9.55 / 10.25 \\
& GRaSP & .52 / .71 & .55 / .64 & .49 / .68 & .49 / .15 & .50 / \textbf{.50} & .50 / .17 & .50 / \textbf{.50} & .50 / \textbf{.50} & .48 / .81 & .50 / .50 & 9.60 / 6.70 \\
 & AGL & .52 / .71$^*$ & \underline{.56} / .64 & .53 / .74$^*$ & .53 / .33$^*$ & .53 / .33$^*$ & .54 /  .35$^*$ &.52 / .35$^*$ & .54 / .34$^*$ & .56 / .86$^*$ & .54 / .69$^*$ & 5.80 / 4.35 \\
  & C-NODE (L1) & .64 / \underline{.79} & .54 / \underline{.65} & \underline{.61 / .79}$^*$ & .64 / \textbf{.43} & .63 / .43 & \textbf{.67} / \underline{.47} & \underline{.59} / .43 & \underline{.68} / .47 & \textbf{.88 / .97} & .49 / .66$^*$ & \underline{3.45 / 3.35} \\
& \textbf{\Lzero-NDDE (Ours)} & \textbf{.64 / .80} & \textbf{.58 / .66} & \textbf{.63 / .80} & .64 / \textbf{.43} & \textbf{.65} / .44 & \textbf{.67} / \underline{.47} & \textbf{.60} / .43 & .66 / \underline{.45} & \underline{.86 / .96} & \textbf{.58 / .75} & \textbf{1.75 / 2.10} \\
\bottomrule
\end{tabular}
}
\caption{\textbf{\Lzero-NDDE outperforms competing methods on the CausalDynamics benchmark.} SHD (lower is better) and AUROC / AUPRC (higher is better) are reported for \Lzero-NDDE (Ours), AGL, C-NODE and the benchmark causal discovery methods. The best values are bold, the second best values are underlined. $^*$ indicates statistical significance between AGL or C-NODE and \Lzero-NDDE, calculated with a Welch's t-test (p-value 0.05). The benchmark only provides the mean values for the causal discovery methods, and we thus do not run statistical significance tests for these methods. For AUROC and AUPRC, statistical significance results always match, so we report only one $^*$ for each pair. The last column reports the average rank for each method and metric. For readability, methods that never score best are not shown here, but in the Appendix (\Cref{tab:full_causal_dyn_new})}
\label{tab:causaldynamics_results}
\end{table*}

\subsection{Dysts Benchmark}
\label{dystsresults}

\begin{table*}[hbt!]\centering

\resizebox{\textwidth}{!}{
\begin{tabular}{c cc cc cc cc cc cc cc}

\toprule
\multirow{2}{*}{\textbf{Method}} & \multicolumn{2}{c}{NRMSE ($\downarrow$)} & \multicolumn{2}{c}{VLPT ($\uparrow$)} & \multicolumn{2}{c}{FLEE ($\downarrow$)} & \multicolumn{2}{c}{$D_{KY}$ ($\downarrow$)} & \multicolumn{2}{c}{LSD ($\downarrow$)} & \multicolumn{2}{c}{$D_{stsp}$ ($\downarrow$)} & \multicolumn{2}{c}{SHD ($\downarrow$)}\\
\cmidrule(lr){2-3} \cmidrule(lr){4-5} \cmidrule(lr){6-7} \cmidrule(lr){8-9} \cmidrule(lr){10-11} \cmidrule(lr){12-13} \cmidrule(lr){14-15}
 & Med. & Rank & Med. & Rank & Med. & Rank & Med. & Rank & Med. & Rank & Med. & Rank & Med. & Rank \\
\midrule
AGL & \textbf{1.15} & \underline{2.69} & 2.85 & 3.14 & 0.25 & 3.47 & 0.88$^*$ & 3.82 & 14.71$^*$ & 2.82 & 3e$^6$$^*$ & 4.37 & 0.44$^*$ & 4.46 \\
PathReg & \underline{1.19} & \textbf{2.49} & 4.33 & 3.24 & 0.25 & 3.78 & 0.89$^*$ & 3.9 & 16.99 & 3.02 & 3e$^6$$^*$ & 4.62 & 0.22 & 2.70 \\
C-NODE & 1.53 & 4.05 & 4.61 & 3.39 & 0.20 & 3.00 & 0.72$^*$ & 3.17 & 37.79$^*$ & 4.08 & 34.60$^*$ & 2.86 & \underline{0.17} & \underline{2.67} \\
NODE & 1.20 & 2.94 & \underline{5.12} & \underline{2.63} & \textbf{0.16} & \underline{2.51} & \underline{0.076} & \underline{2.17} & \textbf{13.99} & \underline{2.57} & \underline{4.29} & \underline{1.62} & 0.22 & 2.70 \\
\Lzero-NDDE & \underline{1.19} & 2.83 & \textbf{6.54} & \textbf{2.60} & \textbf{0.16} & \textbf{2.25} & \textbf{0.075} & \textbf{1.94} & \underline{14.75} & \textbf{2.50} & \textbf{2.50} & \textbf{1.54} & \textbf{0.11} & \textbf{2.45} \\

\bottomrule

\end{tabular}
}
\caption{\textbf{\Lzero-NDDE outperforms other feature selection methods.} \Lzero-NDDE ranks best in five of the metrics. The difference with NODE is non significant, as NODE also performs very well, but \Lzero-NDDE consistently ranks better than the other methods performing input feature selection, especially on attractor dimension reconstruction ($D_{KY}$) and convergence in distribution ($D_{stsp}$). $^*$ indicates statistical significance with the \Lzero-NDDE metric, obtained using a permutation test (p-value 0.05). For each metric, best results are bolded and second best are underlined.}
\label{tab:dysts_results}
\end{table*} 

Results are reported in \Cref{tab:dysts_results}. \Lzero-NDDE and NODE consistently rank better than other methods in all metrics except the NRMSE. The median $D_{stsp}$ is extremely high for AGL and PathReg, as this metric explodes when the two distributions are very different. It is an artefact of $D_{stsp}$. Additionally, we show predicted versus true trajectories for \Lzero-NDDE, NODE, and C-NODE on 26 test systems, in \Cref{fig:full_dysts_results}. Both NODE and \Lzero-NDDE are capable of modelling the systems very well. 

To explain why NODE and \Lzero-NDDE outperform other methods, we report in \Cref{fig:dysts_statistical_association} the relationship between each metric and the SHD (upper half) and training MSE (lower half). We run a Kendall-$\tau$ test \citep{kendall_rank}, standard for non-Gaussian low-samples regimes, to measure the strength of association between the two sets of values. We find that $D_{KY}$, $D_{stsp}$ and VLPT have a statistically significant association with both SHD and training MSE, indicating that on this benchmark, methods that better learn the drivers, and better fit the training data perform better. C-NODE, PathReg, and AGL all apply a regularization on the weights of the neural network, preventing these methods from reaching low training error. PathReg tends to sparsify specific input-output paths without masking off input features, thus not performing input feature selection. This is also why we do not evaluate it on CausalDynamics. With no weight penalty, neural ODE achieves low training error and thus good reconstruction performance. 
\Lzero-penalty does not penalize the weights of the network but instead turns the inputs on/off, allowing \Lzero-NDDE to achieve low train error while masking off irrelevant input features. With the lowest SHD, \Lzero-NDDE thus outperforms other methods. 

\section{Discussion}

We hope that the theoretical framework proposed in this paper lays the foundations for the development of robust methods for modelling lagged dynamical systems. 
One limitation of our method is that we can only optimize for an approximate \Lzero-penalty, as the true \Lzero-loss is not differentiable, although existing methods also rely on approximations. We show empirical improvements on two benchmarks that evaluate the methods mostly on SDE-driven systems as, to the best of our knowledge, no comprehensive benchmark exists for evaluating lagged systems modelling. 

Despite these promising results, future work is needed to apply this method to real physical systems. Our method's computational requirements scale quadratically with the input dimension. To apply it to high-dimensional systems, it would need to be combined with an identifiable dimensionality reduction technique \citep{hizli2024identifyinglatentstatetransition}.  
Our method learns probabilities for the input-output mapping, but the learned drift is deterministic. Extending our method to be fully probabilistic could allow us to better represent uncertainty during long-term autoregressive rollouts. 
We see several exciting applications of our method. Identifying correct drivers in dynamical systems could open the door to scientific discovery across several domains. Our method could also be integrated into physical models to robustly parametrize SDDE-driven subprocesses. It could also be used to stabilize gradient estimates in differentiable models of chaotic systems \citep{metz2022gradientsneed}.

\section*{Acknowledgements}

We thank Julia Kaltenborn and David Rolnick for their feedback on the manuscript. We thank Benjamin Herdeanu, Carla Roesch and Juan Nathaniel for helping us run the CausalDynamics benchmark. This research was supported in part by the Canada CIFAR AI Chairs program. 

This research was enabled in part by compute resources provided by Mila (\url{mila.quebec}).


\section*{AI use statement}

In this work, we used generative AI tools for providing critical ingredients for proving mathematical claims, assisting in the writing of proofs, and implementing methods. 
We have not used generative AI tools for generating synthetic data sets, helping develop theoretical models or conceptual frameworks, formulating mathematical claims, proposing or refining hypotheses, designing or providing feedback on research methodology or experiments, assisting with translation, cleaning and reformatting datasets, supporting qualitative and thematic data analysis, or interpreting results.
Additionally, we used generative AI tools to identify relevant literature. 
We have reviewed all AI-assisted work. 
For the mathematical proofs, we first sketched the proofs ourselves before asking an LLM to review them and suggest corrections or improvements. We have reviewed every suggestion manually before incorporating them into the final version of the proofs. For implementing methods, we carefully verified and tested all code used for the experiments in this paper. All the identified literature and references were manually checked, and we read every paper that our work builds upon or takes inspiration from; e.g., we did not rely on an LLM-generated summary of other research papers to design this paper, or to decide whether to include or not include references. 
We take responsibility for the final content of this work, including text, claims or artifacts produced with the aid of generative AI.

\section*{Reproducibility statement}

The code to implement the methods has been shared in \Cref{implementation}. This code uses poetry to manage the environment. Detailed parameters, as well as hyperparameter search procedures are given in the main text as well as the appendix. As mentioned in the main text, all experiments have been run on CPU, allowing for easy reproducibility i.e. without intensive computational requirements. The metric computation is described in details in the appendix (\Cref{subsec:dysts_metric_computation}). The data used in this paper is taken from peer-reviewed, publicly available benchmarks.

\bibliography{iclr2026_conference}
\bibliographystyle{iclr2026_conference}

\newpage
\appendix

%
%

\appendix
\section{Proofs}
\label{app:proofs}

We dedicate this section to proving the different lemmas, propositions and theorems from the main content of the paper as well as some auxiliary lemmas and corollaries. We first establish that the model is well posed: 
Proposition~\ref{prop:strong_solution} gives a unique, non-exploding strong solution. We then treat identifiability at the population level, first for the exact-risk formulation (Theorem~\ref{thm:sde_identifiability}) and then for the penalized objective (Theorem~\ref{thm:penalized}).
The remaining three subsections provide the probabilistic machinery to go from the population result to the finite-sample one: uniform moment bounds, ergodicity of the true process, and exponential mixing of the sampled lagged state.
These feed the finite-sample recovery guarantee (Theorem~\ref{thm:finite-sample}). We provide a graph illustrating the relationships between the assumptions, lemmas, propositions, and theorems in Figure~\ref{fig:proof-map}.

\subsection{Notation}
\label{subsec:notation}

We denote the true data-generating objects using standard symbols: $G$ for the drift, $G_j$ for its coordinates, $f_j$ for the link functions, and $M$ for the binary mask. Candidates within the model class are indicated by a tilde (e.g., $\tilde{G}$), whereas empirical minimizers are denoted by a hat (e.g., $\hat{G}$). Matrix indices follow the mask variables, where $M^a_{ij}=1$ indicates that the driver coordinate $i$ at lag $\tau_a$ influences the target coordinate $j$. The lag index $a \in \{0, \dots, q\}$ specifies the lag block, with $a=0$ representing the instantaneous interactions. We let $m^a_j$ denote the $j$-th column of $M^a$, and define $S_j = \{(i,a) : M^a_{ij}=1\}$ as the active set of drivers for coordinate $j$. The delay sequence is strictly ordered as $0 < \tau_1 < \dots < \tau_q$, with the maximum delay defined as $\taumax = \tau_q$.

Because the drift evaluates the trajectory at multiple past times, the process $X$ is non-Markovian, though the process of its recent histories (segments) is Markovian. We therefore operate on the \emph{segment space} of continuous functions from $[-\taumax,0]$ to $\mathbb{R}^D$,
$$
\Cseg = C\!\left([-\taumax,0];\R^D\right), \qquad X_t^{\seg}(s) = X_{t+s},\quad s\in[-\taumax,0],
$$
equipped with the supremum norm $\norm{\cdot}_{\Cseg}$. Generic segments in this space are denoted by Greek letters such as $\varphi$, $\psi$, and $\eta$, with $X_0^{\seg}$ serving as the initial condition. The drift relies strictly on $q+1$ discrete time slices of a segment; thus, we define the \emph{lagged state} as $Z_t = (X_t, X_{t-\tau_1}, \dots, X_{t-\tau_q})$. Observations of this state are taken on a temporal grid $t_0 < \dots < t_n = T$, where we write $Z_k = Z_{t_k}$. The effective horizon over which the drift is evaluated is $T_\circ = T - \taumax$. Reference balls centered at the constant $c$ (introduced in our regularity and dissipativity assumptions) are denoted by $B_R \subset \Cseg$ in the segment space and $K_\rho \subset \R^{D(q+1)}$ in the lagged state space, and $c^{(q)} \in \R^{D(q+1)}$ denotes the concatenation of $q+1$ copies of $c$. When $q = 0$ (no delays) we use the conventions $\taumax = 0$ and $\Cseg = \R^D$, so that the segment process is $X$ itself.

We use subscripts on the probability operators to distinguish the initial state of the process. The probability measure and expectation for a solution initialized at a fixed segment $X_0^{\seg} = \varphi$ are written as $\Prob_\varphi$ and $\E_\varphi$, with the corresponding transition kernel $P_t(\varphi, \cdot)$ and semigroup $(P_t)_{t \ge 0}$. This specific notation is utilized for moment and coupling estimates that require uniform bounds over $\varphi \in B_R$. Conversely, when the process is initialized from equilibrium ($X_0^{\seg} \sim \pi^{\seg}$, independent of the driving noise $W$), we may use $\Prob_\pi$ and $\E_\pi$. Here, $\pi^{\seg}$ represents the invariant measure of the system, which forms the probabilistic foundation for our finite-sample analysis. Undecorated expectations $\E$ and probabilities $\Prob$ denote an arbitrary initial distribution consistent with the underlying data-generating process.

We compare candidate drifts in the $L^2$ space associated with the joint occupation measure $\occq$, which represents the time-averaged law of $Z_t$ over the interval $[\taumax, T]$. Identifiability is first established as an identity in $L^2(\occq)$ before being extended to all of $\R^{D(q+1)}$. Under stationary conditions, $\occq$ coincides with the stationary law $\mu_\infty^q$ of $Z_t$, rendering our bounding constants independent of the horizon $T$.

For a given candidate drift $\tilde{G}$ and mask $\tilde{M}$, the population risk is defined as $\Risk_T(\widetilde G) = T_\circ\norm{\widetilde G-G}_{L^2(\occq)}^2$. Its empirical counterpart is the Euler criterion $\eRisk_n$, leading to the penalized objective $J_\lambda(\widetilde G, \tilde M) = \Risk_T(\widetilde G) + \lambda\sum_a\norm{\tilde M^a}_0$. Because $\eRisk_n$ contains a quadratic-variation term identical across all candidates, we evaluate only the excess empirical risk $\eRisk_n(\widetilde G) - \eRisk_n(G)$ against the population risk $\Risk_T$. 
Although $\Risk_T$ involves the unknown $G$, it is determined by observable quantities up to a candidate-independent constant: for every $\widetilde G$ with finite risk, the contrast $\mathcal C_T(\widetilde G) := \E\bigl[\int_{\taumax}^T \norm{\widetilde G(Z_t)}^2\,dt - 2\int_{\taumax}^T \inner{\widetilde G(Z_t)}{dX_t}\bigr]$ satisfies $\mathcal C_T(\widetilde G) - \mathcal C_T(G) = \Risk_T(\widetilde G)$, as follows by substituting the SDDE and using that the square-integrable stochastic integral has zero mean.

Finally, successful recovery is governed by three primary structural quantities. The first is the signal strength $s^a_{ij} = \norm{G_j - \Pi_{-(i,a)}G_j}_{L^2(\occq)}$, where $\Pi_{-(i,a)}$ denotes the orthogonal projection onto the subspace of functions that do not depend on (i.e.\ do not use) coordinate $i$ in lag block $a$. The remaining two are $\smin$, the minimum signal strength across all true edges, and $\dmax$, the maximal in-degree of the system. These combine to define the margin $\gamma(\lambda) = \min\{\lambda, T_\circ\smin^2 - \lambda\dmax\}$, which acts as the minimum objective gap incurred by any incorrectly identified mask.


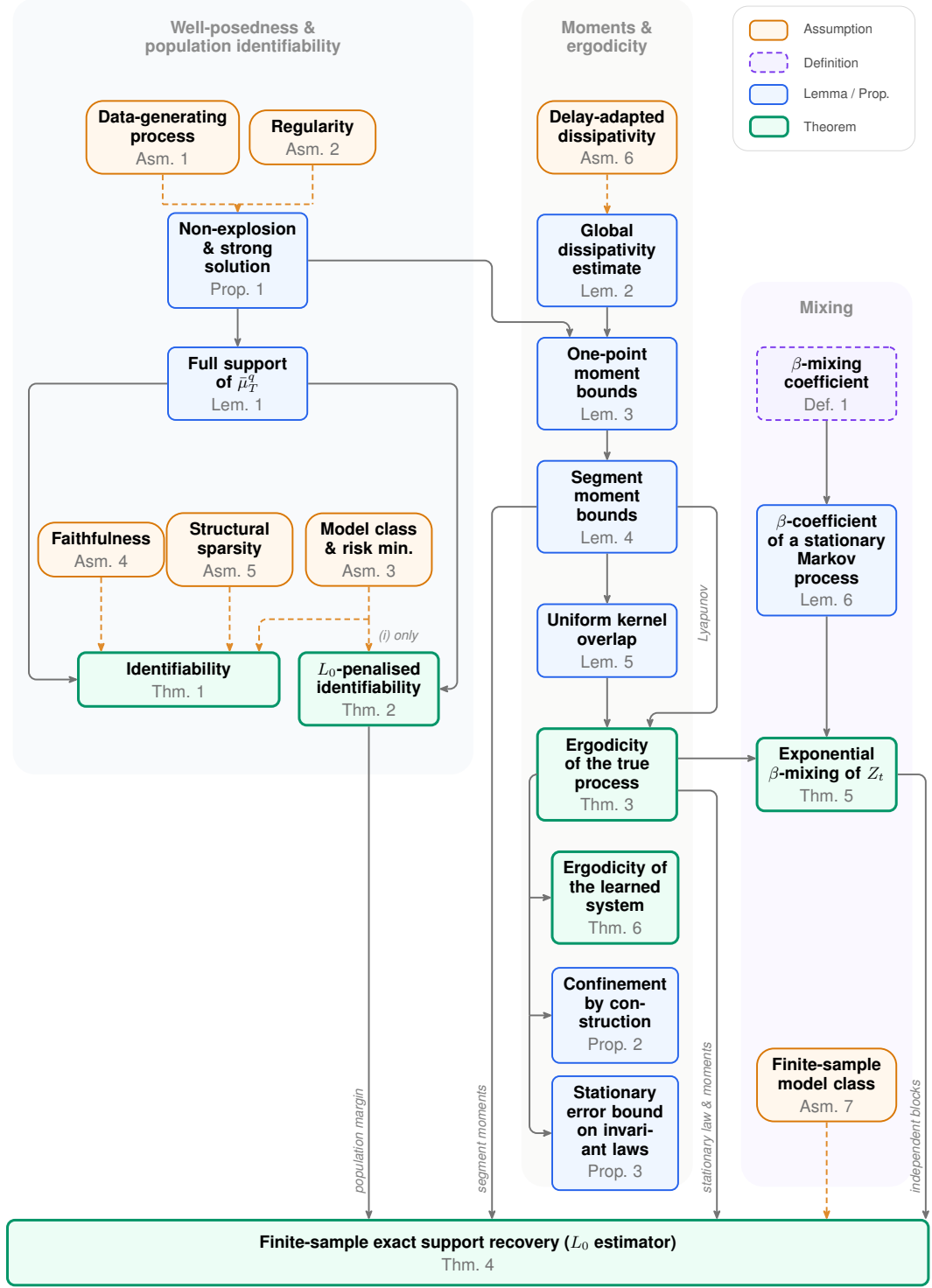
\begin{figure}[htp]
\centering
\definecolor{asmFill}{HTML}{FFF7ED}
\definecolor{asmLine}{HTML}{D97706}
\definecolor{lemFill}{HTML}{EFF6FF}
\definecolor{lemLine}{HTML}{2563EB}
\definecolor{dfnFill}{HTML}{F5F3FF}
\definecolor{dfnLine}{HTML}{7C3AED}
\definecolor{thmFill}{HTML}{ECFDF5}
\definecolor{thmLine}{HTML}{059669}
\definecolor{bandA}{HTML}{F8FAFC}
\definecolor{bandB}{HTML}{FAFAF9}
\definecolor{bandC}{HTML}{FAF8FF}
 
\def\nref#1{{\color{black!55}#1}}
 
\resizebox{\textwidth}{!}{%
\begin{tikzpicture}[
  box/.style    = {draw, rounded corners=4pt, line width=0.9pt, inner sep=5pt,
                   text width=2.25cm, align=center, font=\sffamily\footnotesize},
  asmn/.style   = {box, rounded corners=9pt, fill=asmFill, draw=asmLine, text width=2.0cm},
  wide/.style   = {text width=2.5cm},
  lemn/.style   = {box, fill=lemFill, draw=lemLine},
  dfnn/.style   = {box, fill=dfnFill, draw=dfnLine, densely dashed},
  thmn/.style   = {box, fill=thmFill, draw=thmLine, line width=1.4pt},
  banner/.style = {draw=thmLine, fill=thmFill, rounded corners=4pt, line width=1.4pt,
                   inner sep=8pt, minimum width=17.2cm, align=center,
                   font=\sffamily\footnotesize},
  swatch/.style = {draw, rounded corners=3pt, line width=0.9pt,
                   minimum width=7mm, minimum height=3.6mm, inner sep=0pt},
  ttl/.style    = {font=\sffamily\bfseries\small, text=black!45, align=center},
  note/.style   = {font=\sffamily\scriptsize\itshape, text=black!50},
  dep/.style    = {-{Stealth[round,length=2.2mm,width=1.7mm]}, draw=black!55,
                   line width=0.9pt, rounded corners=4pt},
  adep/.style   = {dep, draw=asmLine!85, densely dashed},
  bus/.style    = {draw=asmLine!85, densely dashed, line width=0.9pt, rounded corners=4pt},
  trunk/.style  = {draw=black!55, line width=0.9pt, rounded corners=4pt}
]
 
\begin{scope}[on background layer]
  \fill[bandA, rounded corners=12pt] (-2.30,  1.90) rectangle ( 6.30, -12.60);
  \fill[bandB, rounded corners=12pt] ( 7.20,  1.90) rectangle (10.40, -20.30);
  \fill[bandC, rounded corners=12pt] (11.30, -3.40) rectangle (14.50, -20.30);
\end{scope}
\node[ttl] at ( 2.00,  1.15) {Well-posedness \&\\population identifiability};
\node[ttl] at ( 8.80,  1.15) {Moments \&\\ergodicity};
\node[ttl] at (12.90, -3.90) {Mixing};
 
\node[asmn,wide] (A1)  at ( 0.50,  -0.70) {\textbf{Data-generating process}\\[2pt]\nref{Asm.~\ref{assu:dgp}}};
\node[asmn]      (A2)  at ( 3.30,  -0.70) {\textbf{Regularity}\\[2pt]\nref{Asm.~\ref{assu:regularity}}};
\node[lemn]      (P1)  at ( 1.90,  -3.00) {\textbf{Non-explosion \& strong solution}\\[2pt]\nref{Prop.~\ref{prop:strong_solution}}};
\node[lemn]      (LJS) at ( 1.90,  -5.30) {\textbf{Full support\\of $\occq$}\\[2pt]\nref{Lem.~\ref{lem:joint-support}}};
\node[asmn]      (A4)  at (-0.65,  -8.40) {\textbf{Faithfulness}\\[2pt]\nref{Asm.~\ref{ass:faithful}}};
\node[asmn]      (A5)  at ( 1.75,  -8.40) {\textbf{Structural sparsity}\\[2pt]\nref{Asm.~\ref{assu:sparsity}}};
\node[asmn]      (A3)  at ( 4.35,  -8.40) {\textbf{Model class \& risk min.}\\[2pt]\nref{Asm.~\ref{assu:risk_minimization}}};
\node[thmn, text width=3.4cm, anchor=north] (T1) at ( 0.80, -10.30)
  {\textbf{Identifiability}\\[2pt]\nref{Thm.~\ref{thm:sde_identifiability}}};
\node[thmn, anchor=north] (T2) at ( 4.35, -10.30)
  {\textbf{\Lzero-penalised identifiability}\\[2pt]\nref{Thm.~\ref{thm:penalized}}\\[2pt]
   };
 
\node[asmn,text width=2.25cm] (A6) at (8.80,  -0.70) {\textbf{Delay-adapted dissipativity}\\[2pt]\nref{Asm.~\ref{ass:dissip}}};
\node[lemn] (L2) at (8.80,  -3.00) {\textbf{Global\\dissipativity estimate}\\[2pt]\nref{Lem.~\ref{lem:dissip-global}}};
\node[lemn] (L3) at (8.80,  -5.30) {\textbf{One-point moment bounds}\\[2pt]\nref{Lem.~\ref{lem:moments-point}}};
\node[lemn] (L4) at (8.80,  -7.60) {\textbf{Segment moment bounds}\\[2pt]\nref{Lem.~\ref{lem:moments}}};
\node[lemn] (L5) at (8.80, -10.10) {\textbf{Uniform kernel\\overlap}\\[2pt]\nref{Lem.~\ref{lem:segment-overlap}}};
\node[thmn] (T3) at (8.80, -12.60) {\textbf{Ergodicity of the true process}\\[2pt]\nref{Thm.~\ref{thm:ergodic}}};
\node[thmn, text width=2.0cm] (T6)  at (8.95, -14.90) {\textbf{Ergodicity of the learned system}\\[2pt]\nref{Thm.~\ref{thm:transfer}}};
\node[lemn, text width=2.0cm] (PLD) at (8.95, -17.10) {\textbf{Confinement by construction}\\[2pt]\nref{Prop.~\ref{prop:learned-damping}}};
\node[lemn, text width=2.0cm] (PIC) at (8.95, -19.30) {\textbf{Stationary error bound on invariant laws}\\[2pt]\nref{Prop.~\ref{prop:learned-invariant-comparison}}};
 
\node[dfnn] (D1) at (12.90,  -5.30) {\textbf{$\beta$-mixing coefficient}\\[2pt]\nref{Def.~\ref{def:beta-mixing}}};
\node[lemn] (L6) at (12.90,  -8.60) {\textbf{$\beta$-coefficient of a stationary Markov process}\\[2pt]\nref{Lem.~\ref{lem:markov-beta}}};
\node[thmn] (T4) at (12.90, -12.60) {\textbf{Exponential $\beta$-mixing of $Z_t$}\\[2pt]\nref{Thm.~\ref{thm:exp-mixing}}};
\node[asmn,text width=2.25cm] (A7) at (12.90, -18.40) {\textbf{Finite-sample model class}\\[2pt]\nref{Asm.~\ref{ass:class}}};
 
\node[banner, anchor=north] (T5) at (6.20, -20.90)
  {\textbf{Finite-sample exact support recovery (\Lzero{} estimator)}\\[2pt]
   \nref{Thm.~\ref{thm:finite-sample}}};
 
\draw[bus]  (A1.south) -- (0.50,-1.95);
\draw[bus]  (A2.south) -- (3.30,-1.95);
\draw[bus]  (0.50,-1.95) -- (3.30,-1.95);
\draw[adep] (1.90,-1.95) -- (P1.north);
 
\draw[dep] (P1.south) -- (LJS.north);
 
\draw[dep]  (LJS.west) -- (-2.00,-5.30) |- (T1.west);
\draw[dep]  (LJS.east) -- ( 6.00,-5.30) |- (T2.east);
\draw[adep] (A4.south) -- (A4.south |- T1.north);
\draw[adep] (A5.south) -- (A5.south |- T1.north);
\draw[bus]  (A3.south) -- (4.35,-9.70);
\draw[adep] (4.35,-9.70) -- (T2.north);
\draw[adep] (4.35,-9.70) -| ($(T1.north)+(1.50,0)$);
\node[note, anchor=west] at (4.40,-10.02) {(i) only};
 
\draw[adep] (A6.south) -- (L2.north);
\draw[dep]  (L2.south) -- (L3.north);
\draw[dep]  (L3.south) -- (L4.north);
\draw[dep]  (L4.south) -- (L5.north);
\draw[dep]  (L5.south) -- (T3.north);
\draw[dep]  (P1.east)  -- (6.65,-3.00) -- (6.65,-4.15) -- (8.10,-4.15) -- (8.10,-4.15 |- L3.north);
\draw[dep]  (L4.east)  -- (10.85,-7.60) -- (10.85,-11.45) -- (9.60,-11.45) -- (9.60,-11.45 |- T3.north);
\node[note, rotate=90, anchor=south] at (10.85,-9.60) {Lyapunov};
\draw[trunk] (T3.west) -- (7.35,-12.60) -- (7.35,-17.10);
\draw[dep]   (7.35,-14.90) -- (T6.west);
\draw[dep]   (7.35,-17.10) -- (PLD.west);
\draw[dep]   (7.35,-17.10) |- (PIC.west);
 
\draw[dep] (D1.south) -- (L6.north);
\draw[dep] (L6.south) -- (T4.north);
\draw[dep] ($(T3.east)+(0,0.30)$) -- ($(T4.west)+(0,0.30)$);
 
\draw[dep]      (T2.south) -- (T2.south |- T5.north);
\draw[dep]      (L4.west)  -- ( 6.65, -7.60) -- ( 6.65,-20.90);
\draw[dep]      ($(T3.east)+(0,-0.30)$) -- (10.85,-12.90) -- (10.85,-20.90);
\draw[dep]      (T4.east)  -- (14.75,-12.60) -- (14.75,-20.90);
\draw[adep]     (A7.south) -- (12.90,-20.90);
 
\node[note, rotate=90, anchor=west] at ( 4.18,-20.55) {population margin};
\node[note, rotate=90, anchor=west] at ( 6.48,-20.55) {segment moments};
\node[note, rotate=90, anchor=west] at (10.68,-20.55) {stationary law \& moments};
\node[note, rotate=90, anchor=west] at (14.58,-20.55) {independent blocks};
 
\begin{scope}
  \draw[draw=black!15, fill=white, rounded corners=8pt] (11.15,1.75) rectangle (14.55,-0.95);
  \node[swatch, fill=asmFill, draw=asmLine, anchor=west] at (11.45, 1.30) {};
  \node[anchor=west, font=\sffamily\scriptsize, text=black!70] at (12.35, 1.30) {Assumption};
  \node[swatch, fill=dfnFill, draw=dfnLine, densely dashed, anchor=west] at (11.45, 0.70) {};
  \node[anchor=west, font=\sffamily\scriptsize, text=black!70] at (12.35, 0.70) {Definition};
  \node[swatch, fill=lemFill, draw=lemLine, anchor=west] at (11.45, 0.10) {};
  \node[anchor=west, font=\sffamily\scriptsize, text=black!70] at (12.35, 0.10) {Lemma / Prop.};
  \node[swatch, fill=thmFill, draw=thmLine, line width=1.4pt, anchor=west] at (11.45,-0.50) {};
  \node[anchor=west, font=\sffamily\scriptsize, text=black!70] at (12.35,-0.50) {Theorem};
\end{scope}

\end{tikzpicture}}
\caption{Dependency map of the theoretical results. 
Yellow boxes are assumptions, blue boxes are lemmas and propositions, and green boxes are theorems. 
Arrows denote when an assumption or a result is used in another proof. A use already implied by a chain of arrows is not drawn, unless it plays a separate role (labelled arrows).
The proofs are grouped vertically into the three groups of proofs in \Cref{app:proofs}, contributing the finite-sample guarantee at the bottom.
}
\label{fig:proof-map}
\end{figure}

\subsection{Non-explosion and existence of a strong solution}
\label{sec:proof_strongsol}

Since the process is not Markov, we cannot apply the classical existence theorems for SDEs directly. Instead, we exploit the fact that the delays are bounded away from zero and prove existence of a unique solution by induction over intervals of length $\taumin := \min_a \tau_a$ (the so-called method of steps, see e.g.\ \citealp{halelunel1993}). 
On each such interval the delayed arguments are already known from the previous interval, so that the SDDE reduces to a standard (non-delayed) It\^o equation with a random, time-dependent drift: we solve that by truncation on a closed ball and rule out explosion with a Khasminskii-type Lyapunov test.
The induction then extends the solution to all of $[0,\infty)$.
The proof of the following statement only uses the initial-condition clause of Assumption~\ref{assu:dgp} (independence of $X_0^{\seg}$ and $W$, and its finite second moment) together with Assumption~\ref{assu:regularity}.

\begin{manualproposition}{prop:strong_solution}{Non-explosion \& Strong solution}
    Under Assumptions \ref{assu:dgp} and \ref{assu:regularity}, the SDDE defined in \Eqref{eq:sde_model} has a unique strong solution for every initial condition with $\mathbb{E}\|X_0^{\mathrm{seg}} \|_\Cseg^2 < \infty$ without explosion i.e. it does not diverge to infinity in a finite amount of time. 
\end{manualproposition}
\begin{proof}
\stp{1: Initialisation on $[-\taumax,\taumin]$}
The initial segment $X_0^{\rm seg}$ provides the solution on $[-\taumax, 0]$. For $s \in [0, \taumin]$ and every lag $a \ge 1$, the delayed argument satisfies $s - \tau_a \le \taumin - \tau_a \le 0$, so $X_{s-\tau_a}$ is read from the initial segment.
Define
\[Y_s = (X_{s - \tau_1}, \dots, X_{s - \tau_q}), \quad s \in [0, \taumin]\]
This is a continuous, $\mathcal{F}_0$-measurable path with compact range $K_Y \subset \Bbb R^{Dq}$ (as a continuous function on the compact $[0, \taumin]$ composed with the continuous initial segment).
On this segment, the SDDE reduces to a standard (non-delayed) It\^o equation:
\[dX_s = b(s, X_s) \, ds + \Sigma \, dW_s, \quad X_0 = X_0^{\rm seg}(0)\]
with time-dependent drift $b(s, x) := G(x, Y_s)$.

  \substp{(a) Existence and uniqueness via truncated approximation.}
    For $n \in \Bbb N$, consider the metric projection onto the closed ball of radius $n$ centered at $c$, $\pi_n : \Bbb{R}^D \rightarrow \bar{B}_n(c)$ which is $1$-Lipschitz by convexity. Define the \textit{truncated drift}:
    \[b_n(s, x) := b(s, \pi_n(x)) = G(\pi_n(x), Y_s)\]
    Since $G$ is locally Lipschitz by Assumption~\ref{assu:regularity} and both $\pi_n$, $Y_s$ remain in a compact set ($\bar{B}_n(c) \times K_Y$), the truncated drift $b_n$ is globally Lipschitz in $x$ uniformly in $s \in [0, \taumin]$ with Lipschitz constant $L_n$ (random but $\mathcal{F}_0$-measurable).
    
    Since $L_n$ and $Y$ are $\mathcal F_0$-measurable while the Brownian motion $W$ is independent of $\mathcal F_0$ (Assumption~\ref{assu:dgp}), we may apply the Picard-Lindel\"of theorem for SDEs with globally Lipschitz coefficients \citep[Ch.~2]{mao2007} conditionally on $\mathcal F_0$, i.e.\ for each fixed realisation of the initial segment (alternatively, one may localise on the events $\{L_n \le m\} \in \mathcal F_0$, $m \in \Bbb N$). Hence the truncated equation
    \[dX_s^{(n)} = b_n(s, X_s^{(n)})\, ds + \Sigma \, dW_s, \quad X_0^{(n)} = X_0^{\rm seg}(0)\]
    admits a unique non-exploding strong solution on $[0, \taumin]$.
    
    Define the exit time from $\bar{B}_n(c)$, $\zeta_n := \inf \{s \ge 0, \|X_s^{(n)} - c \| \ge n\}$. For any $n \le m$, the pathwise uniqueness on $[0, \zeta_n]$ gives $X_s^{(n)} = X_s^{(m)}$ for all $s \le \zeta_n$ (since $b_n = b_m = b$ on the ball $\bar{B}_n(c)$). Thus, $\zeta_n$ is non-decreasing in $n$, and given the limit $\zeta = \lim_{n \rightarrow \infty} \zeta_n$, define the process $X_s := X_s^{(n)}$ for $s < \zeta_n$. This defines a unique strong solution on $[0, \zeta)$.

  \substp{(b) Non-explosion via Khasminskii's test (\citealp{khasminskii2012}; see also \citealp{mao2007}).}
    Set $\Phi(x) := 1 + \|x-c\|^2$. We work conditionally on $\mathcal F_0$, under which $Y$ and the constants below are fixed while $W$ remains a Brownian motion. By Dynkin's formula applied to $\Phi(X_{s \wedge \zeta_n})$:
    \[\mathbb{E}[\Phi(X_{s \wedge \zeta_n}) \mid \mathcal F_0] = \Phi(X_0) + \mathbb{E} \Bigl[\int_0^{s \wedge \zeta_n} \mathcal{L}_u \Phi(X_u) \, du \Bigm| \mathcal F_0\Bigr]\]
    where $\mathcal{L}_u\Phi(x) := 2\langle x-c, b(u,x)\rangle + \tr(\Sigma\Sigma^\top)$ is the (time-dependent) generator of the equation with frozen delayed arguments.

    Let $\rho_Y^2 := \sup_{s \in [0,\taumin]} \sum_{a=1}^q \|X_{s-\tau_a} - c\|^2 < \infty$, which is $\mathcal F_0$-measurable since $K_Y$ is compact, so that $\|(x_0, Y_s) - c^{(q)}\|^2 \le \|x_0 - c\|^2 + \rho_Y^2$.
    By Assumption~\ref{assu:regularity} (ii), whenever $\|(x_0,Y_s)-c^{(q)}\|\ge R$ we have $2\langle x_0-c, G(x_0,Y_s)\rangle \le 2\beta(1 + \|x_0-c\|^2 + \rho_Y^2) \le 2\beta(1+\rho_Y^2)\Phi(x_0)$. Otherwise $\|x_0-c\| < R$. For $\|x_0-c\|\le R$, set
    \[
    M_R := \sup\{\|G(x_0,y)\| : \|x_0-c\|\le R,\ y\in K_Y\} < \infty \quad\text{a.s.},
    \]
    finite by continuity of $G$ on the compact $\bar B_R(c)\times K_Y$, so that $2\langle x_0-c, G(x_0,Y_s)\rangle \le 2R M_R \le 2C_R\Phi(x_0)$ with $C_R := R M_R$, using $\Phi\ge 1$.
    Hence, for all $x\in\R^D$ and $s\in[0,\taumin]$,
    \[
    \mathcal{L}_s\Phi(x) = 2\langle x-c, G(x,Y_s)\rangle + \tr(\Sigma\Sigma^\top) \le C\,\Phi(x),
    \]
where $C := 2\max\bigl(\beta(1+\rho_Y^2), C_R\bigr) + \tr(\Sigma\Sigma^\top)$ is $\mathcal{F}_0$-measurable and a.s. finite.

    As $\int_0^{s\wedge\zeta_n}\mathcal L_u\Phi(X_u)\,du \le C\int_0^s \Phi(X_{u\wedge\zeta_n})\,du$, Gr\"onwall's inequality gives
    \[\mathbb{E}[\Phi(X_{s \wedge \zeta_n}) \mid \mathcal F_0] \le \Phi(X_0)e^{Cs}, \quad s \in [0, \taumin].\]
    Since $\Phi(X_{\zeta_n}) \ge 1 + n^2$ on $\{\zeta_n \le \taumin\}$, we have that
    \[(1+n^2)\mathbb{P}(\zeta_n \le \taumin \mid \mathcal F_0) \le \mathbb{E}[\Phi(X_{\zeta_n}) \mathbf{1}_{\zeta_n \le \taumin} \mid \mathcal F_0] \le \Phi(X_0) e^{C \taumin}.\]
    Since $\{\zeta \le \taumin\} \subseteq \{\zeta_n \le \taumin\}$, letting $n \rightarrow \infty$ yields $\mathbb{P}(\zeta \le \taumin \mid \mathcal F_0) = 0$ a.s., and integrating these conditional probabilities (which are bounded by one) gives $\mathbb{P}(\zeta \le \taumin) = 0$, i.e.\ the solution does not explode on $[0, \taumin]$. No integrability of $\Phi(X_0)e^{C\taumin}$ is needed.

\stp{2: Induction on $[k\taumin,(k+1)\taumin]$}
The delayed arguments now satisfy $s-\tau_a \le k\taumin$, hence lie on the
already-constructed path; parts (a)--(b) apply verbatim with $\mathcal F_0$ replaced
by $\mathcal F_{k\taumin}$.

Induction over $k$ yields a unique strong solution on $[0,\infty)$ with
$\sup_{s\le T}\norm{X_s}<\infty$ a.s.\ for every finite $T$.

Note: When $q = 0$ there are no delayed arguments: Step~1 applies directly on any finite interval $[0, T]$ (with $Y$ void and $\taumin$ replaced by $T$), and no induction is needed.
\end{proof}

\subsection{Population identifiability}
\label{subsec:sde_identifiability}

We prove population identifiability in two stages. First, we show that the drift and the mask are pinned down on the support of the joint occupation measure: matching the risk forces $\widetilde G = G$ there by continuity, and faithfulness upgrades this to equality of the masks.
Second, Lemma~\ref{lem:joint-support} shows that, under non-degenerate noise, the occupation measure has full support, so ``on the support'' is in fact ``everywhere'' and the faithfulness hypothesis reduces to plain faithfulness.

The proof follows the structure of the assumptions: a zero risk forces $\widetilde G = G$ on the support of $\occq = \R^{D(q+1)}$ which fills the whole domain from the non-degenerate noise (\Cref{lem:joint-support}), faithfulness prevents the learned mask from dropping a true edge, and sparsity prevents it from adding a spurious one. We compare these assumptions to those of \citet{bellot2022ngm} in \Cref{subsec:assumptions_bellot}.

\begin{manualtheorem}{thm:sde_identifiability}{Identifiability}
    Under Assumptions~\ref{assu:dgp}--\ref{assu:sparsity}, $G$ and $M$ are identifiable: 
    \begin{align}
    \widetilde G = G \quad \text{on } \supp(\occq) , \qquad \tilde M = M .
    \end{align}
    Moreover, under the non-degenerate noise of Assumption~\ref{assu:dgp} the occupation measure has full support, $\supp(\occq) = \R^{D(q+1)}$ by \Cref{lem:joint-support} (see below), so the drift equality extends to all of $\R^{D(q+1)}$.
\end{manualtheorem}
\begin{proof}
\stp{1: Drift identifiability on $\supp(\occq)$}
    The true pair $(G, M)$ is feasible by Assumption~\ref{assu:risk_minimization}(i) and has risk zero.
    Nonnegativity and optimality (Assumption~\ref{assu:risk_minimization}(ii)) therefore give $\Risk_T(\widetilde G) = 0$, hence by \Eqref{eq:risk-identity} $\widetilde G = G$ $\occq$-a.e.; $G$ is continuous by Assumption~\ref{assu:regularity} and $\widetilde G$ by Assumption~\ref{assu:risk_minimization}(i), and the complement of a $\occq$-null set meets every nonempty relatively open subset of $\supp(\occq)$, so equality on this dense subset of the (closed) support upgrades to equality on all of $\supp(\occq)$.

\stp{2: Mask identifiability on the support}
    If $M^a_{ij} = 1$ but $\tilde M^a_{ij} = 0$ --- that is, if a true edge were dropped --- then $\widetilde G_j$ is constant in coordinate $i$ of block $a$ on all of $\R^{D(q+1)}$ (its mask removes that input), so $G_j = \widetilde G_j$ on $\supp(\occq)$ takes equal values at the axis-aligned witnessing pair $z, z' \in \supp(\occq)$ of Assumption~\ref{ass:faithful}, which is a contradiction.
    Hence $\tilde M^a_{ij} \ge M^a_{ij}$ entry-wise for all $a$, i.e.\ $\sum_a \norm{\tilde M^a}_0 \ge \sum_a \norm{M^a}_0$ with equality iff $\tilde M^a = M^a$ for all $a$; Assumption~\ref{assu:sparsity} forces equality.
\end{proof}

It remains to remove the ``on the support'' qualifier. The next lemma shows that non-degenerate noise makes the joint law of the lagged state charge every open set, so its support is the whole space and the identity is global.

\begin{lemma}[Full support of the joint occupation measure]\label{lem:joint-support}
Under Assumptions~\ref{assu:dgp} and \ref{assu:regularity}, for every $t > \taumax$ the law of $\left(X_t, X_{t-\tau_1}, \dots, X_{t-\tau_q}\right)$ on $\R^{D(q+1)}$ has full support; consequently $\supp(\occq) = \R^{D(q+1)}$ for every $T > \taumax$.
\end{lemma}

\begin{proof}
We first prove the statement for a deterministic initial history and then integrate.

\stp{1: Deterministic initial segment and pilot path}
    Fix a deterministic initial segment $\varphi \in \Cseg$, a time $t > \taumax$, a target $z = (z_0, z_1, \ldots, z_q) \in \R^{D(q+1)}$ and $\varepsilon > 0$, and write $\Prob_\varphi$ for the law of the solution started from $X_0^{\seg} = \varphi$.
    We must show:
    \[
    \Prob_\varphi\left(\norm{X_t-z_0}<\varepsilon, \norm{X_{t-\tau_1}-z_1}<\varepsilon, \ldots, \norm{X_{t-\tau_q}-z_q} < \varepsilon\right) > 0.
    \]
    With the convention $\tau_0 = 0$, the times $t - \tau_a$ for $a = 0, 1, \ldots, q$ are distinct and strictly positive, since $t > \taumax = \tau_q$.
    Choose a path $h : [-\taumax, t] \to \R^D$ which is continuously differentiable on $[0,t]$ and satisfies
    \[
    h|_{[-\taumax, 0]} = \varphi, \qquad h(t - \tau_a) = z_a \quad (0 \le a \le q).
    \]
    Such a path exists because we are prescribing finitely many values at distinct times of $(0, t]$, together with the value $h(0) = \varphi(0)$ at the left endpoint (e.g.\ by spline interpolation).
    Note that $h$ need not be differentiable at $0$ and no regularity of $\varphi$ beyond continuity is required.

\stp{2: The driving control}
    Since $\Sigma$ is invertible by Assumption~\ref{assu:dgp}, we may define
    \[
    w_h(s) \;:=\; \Sigma^{-1}\left[h(s) - \varphi(0) - \int_0^s G\bigl(h(u), h(u - \tau_1), \ldots, h(u - \tau_q)\bigr)\, du\right], \qquad s \in [0,t].
    \]
    The integrand is continuous on $[0,t]$, so $w_h$ is continuously differentiable with $w_h(0) = 0$. In particular, $w_h$ has square-integrable derivative and belongs to the Cameron--Martin space of the Wiener measure on $C([0,t]; \R^D)$ (see \citealp[Ch.~1]{nualart2006}).
    By construction $h$ solves the integral equation associated with \Eqref{eq:sde_model} on $[0,t]$ driven by $w_h$ in place of $W$:
    \begin{equation}
    \label{eq:pilot_integral}
    h(s) = \varphi(0) + \int_0^s G\bigl(h(u), h(u-\tau_1), \ldots, h(u-\tau_q)\bigr)\, du + \Sigma\, w_h(s).
    \end{equation}
\stp{3: Localized Gr\"onwall estimate}
    Let
    \[
        K := \Bigl\{ (y_0, \ldots, y_q) \in \R^{D(q+1)} \;:\; \exists\, u \in [0,t] \text{ with } \max_{0 \le a \le q} \norm{y_a - h(u - \tau_a)} \le 1 \Bigr\},
    \]
    a compact subset of $\R^{D(q+1)}$, and let $L < \infty$ be a Lipschitz constant for $G$ on $K$, which exists by Assumption~\ref{assu:regularity} (i).
    Both $K$ and $L$ are deterministic as functions of $h$ alone.
    Choose $\eta > 0$ such that
    \begin{equation}
    \label{eq:eta_choice}
    \norm{\Sigma}_{\mathrm{op}}\, \eta \, e^{L \sqrt{q+1}\, t} \;<\; \min(1, \varepsilon),
    \end{equation}
    and define the Brownian tube event and the exit time
    \[
    A_\eta := \Bigl\{ \sup_{s \le t} \norm{W_s - w_h(s)} < \eta \Bigr\}, \qquad \sigma := t \wedge \inf\{ s \ge 0 : \norm{X_s - h(s)} \ge 1 \}.
    \]
    Since $X$ and $h$ agree with $\varphi$ on $[-\taumax, 0]$, we have $X_u - h(u) = 0$ for every $u \in [-\taumax, 0]$, so that for $s \in [0, \sigma]$
    \[
    \norm{X_{s - \tau_a} - h(s - \tau_a)} \;\le\; D(s) := \sup_{0 \le u \le s} \norm{X_u - h(u)}, \qquad 0 \le a \le q.
    \]
    Subtracting \Eqref{eq:pilot_integral} from the integral form of \Eqref{eq:sde_model} and using that both argument vectors lie in $K$ before time $\sigma$, we obtain on $A_\eta$, for every $s \le \sigma$,
    \[
    D(s) \;\le\; \norm{\Sigma}_{\mathrm{op}}\, \sup_{u \le t} \norm{W_u - w_h(u)} + L \sqrt{q+1} \int_0^s D(u)\, du \;\le\; \norm{\Sigma}_{\mathrm{op}}\, \eta + L \sqrt{q+1} \int_0^s D(u)\, du,
    \]
    the factor $\sqrt{q+1}$ coming from bounding the Euclidean norm on $\R^{D(q+1)}$ by $\sqrt{q+1}$ times the maximum of the blockwise norms.
    Gr\"onwall's inequality and \Eqref{eq:eta_choice} then give
    \[
    D(\sigma) \;\le\; \norm{\Sigma}_{\mathrm{op}}\, \eta\, e^{L\sqrt{q+1}\, t} \;<\; \min(1, \varepsilon).
    \]
    In particular $\norm{X_\sigma - h(\sigma)} < 1$ strictly, so by continuity of $s \mapsto X_s - h(s)$ the exit time cannot occur before $t$, i.e.\ $\sigma = t$ on $A_\eta$, and the estimate $\sup_{s \le t} \norm{X_s - h(s)} < \varepsilon$ holds on all of $[0,t]$.

\stp{4: Positivity of the Brownian tube probability}
    The path $w_h$ is continuously differentiable with $w_h(0) = 0$, hence lies in the Cameron--Martin space, so the law of $W - w_h$ on $C([0,t]; \R^D)$ is equivalent to the Wiener measure.
    Since the Brownian small-ball probability $\Prob(\sup_{s \le t} \norm{W_s} < \eta)$ is strictly positive, we conclude that $\Prob(A_\eta) > 0$ \citep[Sec.~3]{stroockvaradhan1972}.
    Combining with Step 3 and $h(t - \tau_a) = z_a$,
    \[
    \Prob_\varphi\Bigl( \max_{0 \le a \le q} \norm{X_{t - \tau_a} - z_a} < \varepsilon \Bigr) \;\ge\; \Prob(A_\eta) \;>\; 0.
    \]
\stp{5: Random initial history and conclusion}
    By Proposition~\ref{prop:strong_solution}, the solution is a measurable functional of $(\varphi, W)$, so $\varphi \mapsto \Prob_\varphi(\cdot)$ is a probability kernel; using Assumption~\ref{assu:dgp}, the initial segment $X_0^{\seg}$ is independent of $W$.
    Conditioning on $X_0^{\seg} = \varphi$ and integrating the strictly positive conditional probabilities of Step 4 against $\mathrm{Law}(X_0^{\seg})$ yields
    \[
    \Prob\Bigl( \max_{0 \le a \le q} \norm{X_{t - \tau_a} - z_a} < \varepsilon \Bigr) \;>\; 0.
    \]
    Since $z \in \R^{D(q+1)}$ and $\varepsilon > 0$ were arbitrary, the law of $(X_t, X_{t - \tau_1}, \ldots, X_{t - \tau_q})$ charges every nonempty open subset of $\R^{D(q+1)}$, i.e.\ it has full support.
    Finally, for every nonempty open $O \subset \R^{D(q+1)}$ the map $t \mapsto \Prob\bigl((X_t, X_{t-\tau_1}, \ldots, X_{t-\tau_q}) \in O\bigr)$ is strictly positive on $(\taumax, T]$, so its integral over that interval is strictly positive.
    The single endpoint $t = \taumax$, at which the lagged coordinate $X_{t - \tau_q} = X_0$ is still determined by the initial history, has zero Lebesgue measure and does not affect the integral.
    Hence $\supp(\occq) = \R^{D(q+1)}$.
\end{proof}

\subsection{Identifiability with the penalized objective}
\label{subsec:penalized}

The penalized objective decouples across output coordinates, so it is enough to argue one coordinate $j$ at a time. For each $j$, we compare the true active set $S_j$ against a candidate $\tilde S_j$ in two exhaustive cases; the minimum-signal condition $\smin>0$ makes dropping a true input strictly costly, and this yields the sufficient $\lambda$-window \Eqref{eq:lambda-window}.

\begin{manualtheorem}{thm:penalized}{Identifiability with the penalized objective}
Let Assumptions~\ref{assu:dgp}, \ref{assu:regularity} and \ref{assu:risk_minimization}(i) hold. Suppose $M \neq 0$ and $\smin > 0$ i.e. the system has at least one true driver, and every true driver has a measurable impact on the dynamics. This replaces Assumption~\ref{ass:faithful} (faithfulness).
    If
    \begin{equation}\label{eq:lambda-window}
    0 \;<\; \lambda \;<\; \frac{(T - \taumax) \smin^2}{\dmax},
    \end{equation}
    then every minimizer of $J_\lambda$ over $\Hclass^q$ satisfies $\tilde M = M$ and $\widetilde G = G$ \text{on }$\R^{D(q+1)}$. Moreover, every feasible pair with $\tilde M \neq M$ satisfies
    \[J_\lambda(\tilde{G}, \tilde{M}) - J_\lambda(G, M) \ge \gamma(\lambda) := \min(\lambda, (T - \taumax) s_{\min}^2 - \lambda d_{\max}) > 0\]
\end{manualtheorem}
\begin{proof}
The true pair $(G , M)$ is feasible by Assumption~\ref{assu:risk_minimization}(i) and satisfies $\Risk_T(G) = 0$, so $J_\lambda(G , M) = \lambda \sum_{j = 1}^{D} | S_j |$, where $S_j := \{ (i , a) : M^a_{ij} = 1 \}$ is the true active set of coordinate $j$ and $| S_j | = \sum_{a = 0}^{q} \norm{m^a_j}_0 \le \dmax$.
Let $(\widetilde G , \tilde M) \in \Hclass^q$ be any feasible pair and write $\tilde S_j := \{ (i , a) : \tilde M^a_{ij} = 1 \}$.
Since the squared $L_2(\occq)$ risk and the $\ell_0$ penalty of \Eqref{eq:penalized_risk} are both sums over the output coordinates, the objective gap decomposes as
\begin{equation}
\label{eq:pen-decouple}
J_\lambda(\widetilde G , \tilde M) - J_\lambda(G , M) \;=\; \sum_{j = 1}^{D} \Delta_j , \qquad \Delta_j \;:=\; (T - \taumax) \, \norm{G_j - \widetilde G_j}_{L_2(\occq)}^2 + \lambda \bigl( | \tilde S_j | - | S_j | \bigr) .
\end{equation}
This decomposes the \emph{value} of the gap at a fixed feasible pair.
It therefore suffices to bound $\Delta_j$ from below in the two exhaustive cases and to sum.

Consider first a coordinate whose mask keeps every true input, $\tilde S_j \supseteq S_j$.
Then $| \tilde S_j | \ge | S_j |$, and since the risk term is nonnegative we obtain $\Delta_j \ge 0$; if the containment is strict then $| \tilde S_j | \ge | S_j | + 1$ and
\begin{equation}
\label{eq:pen-caseA}
\Delta_j \;\ge\; \lambda \;>\; 0 .
\end{equation}
Consider next a coordinate whose mask omits some true input $(i , a) \in S_j \setminus \tilde S_j$, that is $\tilde M^a_{ij} = 0$ while $M^a_{ij} = 1$.
By the masked parametrization of Assumption~\ref{assu:risk_minimization}(i), the candidate $\widetilde G_j = \widetilde f_j(\tilde m^0_j \odot x_0 , \ldots , \tilde m^q_j \odot x_q)$ does not depend on coordinate $i$ of block $a$, so it belongs to the closed subspace onto which $\Pi_{-(i,a)}$ projects (the $L_2(\occq)$ space of the $\sigma$-algebra generated by the remaining coordinates; note that $G_j \in L_2(\occq)$ by the finite-energy clause of Assumption~\ref{assu:dgp}, hence $\widetilde G_j \in L_2(\occq)$ by finite risk).

Since $\Pi_{-(i,a)} G_j$ is the nearest point of that subspace to $G_j$, the signal strength \Eqref{eq:signal} lower-bounds the distance to any of its elements, and in particular
\[
\norm{G_j - \widetilde G_j}_{L_2(\occq)} \;\ge\; \bignorm{\, G_j - \Pi_{-(i,a)} G_j \,}_{L_2(\occq)} \;=\; s^a_{ij} \;\ge\; \smin .
\]
Bounding the penalty term by $| \tilde S_j | \ge 0$ and $| S_j | \le \dmax$ gives
\begin{equation}
\label{eq:pen-caseB}
\Delta_j \;\ge\; (T - \taumax) \, ( s^a_{ij} )^2 - \lambda \, | S_j | \;\ge\; (T - \taumax) \, \smin^2 - \lambda \dmax ,
\end{equation}
which is strictly positive under the window \Eqref{eq:lambda-window}.

Now let $(\widetilde G , \tilde M)$ be feasible with $\tilde M \neq M$ and let $j$ be a coordinate in which the two masks differ.
That coordinate falls into exactly one of the two cases: either it omits a true input, and \Eqref{eq:pen-caseB} applies, or its mask strictly contains $S_j$, and \Eqref{eq:pen-caseA} applies.
In both cases $\Delta_j \ge \gamma(\lambda)$, while every remaining coordinate satisfies $\Delta_j \ge 0$ under \Eqref{eq:lambda-window}.
Summing in \Eqref{eq:pen-decouple} yields $J_\lambda(\widetilde G , \tilde M) - J_\lambda(G , M) \ge \gamma(\lambda) > 0$, so no pair with an incorrect mask can be a minimizer and every minimizer satisfies $\tilde M = M$.
Given $\tilde M = M$ the true pair is feasible with objective $\lambda \sum_j | S_j |$, so any minimizer has $\Risk_T(\widetilde G) = 0$, that is $\widetilde G = G$ $\occq$-almost everywhere.
$G$ is continuous by Assumption~\ref{assu:regularity} (i) and $\widetilde G$ by Assumption~\ref{assu:risk_minimization}(i), and $\supp(\occq) = \R^{D(q+1)}$ by Lemma~\ref{lem:joint-support}, so $\widetilde G = G$ everywhere, exactly as in the proof of \Cref{thm:sde_identifiability}.
\end{proof}
\subsection{Uniform moment bounds}
\label{subsec:moments}

In order to prove some ergodicity results in \Cref{subsec:ergodicity}, we will need some estimates on the moments of the process both point-wise and in the segment space. This section provides: first a bound on the one-point moments $\E\norm{X_t-c}^{2p}$, obtained from It\^o's formula and a Halanay comparison to handle the delayed feedback, second a lifting of these to the segment norm via the Burkholder--Davis--Gundy inequality \citep{karatzas1991brownian}.

We first give a small lemma combining the radial bound and the residual bound from the Assumption~\ref{ass:dissip} to get a bound on the whole space.

\begin{lemma}[Global dissipativity estimate]\label{lem:dissip-global}
    Under Assumption~\ref{ass:dissip}, set $b_a := \beta_a + \varepsilon_a$ for $a = 1, \dots, q$, so that $B = \sum_{a=1}^q b_a < \beta_0$ by \Eqref{eq:strict-domination}. Combining the radial bound outside $\bar B_R(c)$ with the residual bound inside it, we obtain, for all $(x_0 , x_1 , \dots , x_q) \in \R^{D(q+1)}$,
    \begin{equation}
    \label{eq:dissip-explicit}
    \langle x_0 - c , \, G(x_0 , x_1 , \dots , x_q) \rangle \;\le\; A -\beta_0 \norm{x_0 - c}^2 + \sum_{a=1}^{q} b_a \norm{x_a - c}^2,
    \end{equation}
    where $A := \max\bigl\{ \alpha , \; C_R + \beta_0 R^2 \bigr\}$ and $C_R$ is the constant defined in \Eqref{eq:ball-control}.
\end{lemma}
\begin{proof}
    Outside the ball, \Eqref{eq:dissip-explicit} follows from the radial bound \Eqref{eq:delay-dissip}, since $A \ge \alpha$ and $b_a \ge \beta_a$.
    Inside it, $\beta_0 \norm{x_0 - c}^2 \le \beta_0 R^2$, so adding and subtracting this term in the residual bound \Eqref{eq:ball-control} and using $b_a \ge \varepsilon_a$ gives \Eqref{eq:dissip-explicit} with the constant $C_R + \beta_0 R^2 \le A$.
\end{proof}

\begin{lemma}[One-point moment bounds]\label{lem:moments-point}
    Suppose that Assumptions~\ref{assu:dgp}, \ref{assu:regularity} and \ref{ass:dissip} hold.
    For every integer $p\ge 1$ with $\E\norm{X_0^{\seg}-c}_{\mathcal{C}}^{2p}<\infty$ there are constants $\gamma_p > 0$ and $K_p^\star < \infty$, depending only on $\Sigma$, $p$, the delays and the constants from the assumption bounds, such that
    \begin{equation}
        \mathbb{E} \norm{X_t - c}^{2p} \le \mathbb{E}\norm{X_0^\seg - c}_\Cseg^{2p} \, e^{-\gamma_p t} + K_p^\star
    \end{equation}
\end{lemma}
\begin{proof}
Write $Q := \Sigma\Sigma^\top$, $\Phi_1(x) := \norm{x-c}^2$ and $\Phi_p := \Phi_1^{\,p}$, so that $\Phi_p(X_t) = \norm{X_t - c}^{2p}$, and recall $b_a = \beta_a + \varepsilon_a$ and $B = \sum_{a=1}^{q} b_a$ from \Cref{lem:dissip-global}, and set $M_0 := \E\norm{X_0^{\seg}-c}_\Cseg^{2p} < \infty$.

\stp{1: Generator bound for $\Phi_p$}
    Since $\nabla \Phi_p = 2p \, \Phi_1^{p-1} (x - c)$ and $D^2 \Phi_p = 2p \, \Phi_1^{p-1} I + 4p(p-1) \, \Phi_1^{p-2} (x-c)(x-c)^\top$, It\^o's formula applied to \Eqref{eq:sde_model} gives
    \[
    \mathcal{L}\Phi_p = 2p \, \Phi_1^{p-1}(x_0) \, \langle x_0 - c , G \rangle + p \, \Phi_1^{p-1}(x_0) \tr(Q) + 2p(p-1) \, \Phi_1^{p-2}(x_0) \bignorm{\Sigma^\top (x_0 - c)}^2 .
    \]
    Note that the last term vanishes when $p = 1$.
    Bounding $\norm{\Sigma^\top(x_0-c)}^2 \le \norm{\Sigma}_{\rm op}^2 \, \Phi_1(x_0)$ and inserting \Eqref{eq:dissip-explicit} yields,
    \begin{equation}\label{eq:gen-raw}
    \mathcal{L}\Phi_p \;\le\; -2p\beta_0 \, \Phi_p(x_0) + 2p \sum_{a=1}^{q} b_a \, \Phi_1^{p-1}(x_0) \Phi_1(x_a) + k_p \, \Phi_1^{p-1}(x_0),
    \end{equation}
    with $k_p := 2pA + p \tr(Q) + 2p(p-1) \norm{\Sigma}_{\rm op}^2$.
    The delayed terms $\Phi_1(x_a)$ are the reason a plain Gr\"onwall argument does not apply and a delay comparison is needed using Halanay's inequality \citep{halanay1966}.

\stp{2: Scaled Young inequality and the dissipativity margin}
    For $p > 1$, Young's inequality with conjugate exponents $p/(p-1)$ and $p$ gives
    \[
    \Phi_1^{p-1}(x_0) \, \Phi_1(x_a) \;\le\; \tfrac{p-1}{p} \, \Phi_p(x_0) + \tfrac1p \, \Phi_p(x_a) ,
    \]
    so the cross terms in \Eqref{eq:gen-raw} are bounded by $2(p-1) B \, \Phi_p(x_0) + 2 \sum_a b_a \, \Phi_p(x_a)$.
    
    The lower-order term $k_p \, \Phi_1^{p-1}(x_0)$ is treated separately using a scale $\eta_p$. Fix any
    \[
    \eta_p \in \bigl( 0 , \, 2p(\beta_0 - B) \bigr) , \qquad
    C_p := \sup_{y \ge 0} \bigl( k_p \, y^{p-1} - \eta_p \, y^{p} \bigr) = \frac{k_p^{\,p} \, (p-1)^{p-1}}{p^{\,p} \, \eta_p^{\,p-1}} < \infty ,
    \]
    so that $k_p \, \Phi_1^{p-1}(x_0) \le \eta_p \, \Phi_p(x_0) + C_p$; for $p = 1$ this term is already the constant $k_1$ and we set $\eta_1 := 0$, $C_1 := k_1$.
    
    Substituting both bounds into \Eqref{eq:gen-raw} yields the delay-dissipative estimate
    \begin{equation}
    \label{eq:gen-final}
    \mathcal{L}\Phi_p \;\le\; -a_p \, \Phi_p(x_0) + 2 \sum_{a=1}^{q} b_a \, \Phi_p(x_a) + C_p ,
    \qquad a_p := 2p\beta_0 - 2(p-1)B - \eta_p .
    \end{equation}
    The role of the free parameter is that the margin required by the comparison step is now automatic:
    \[
    a_p - 2B \;=\; 2p(\beta_0 - B) - \eta_p \;>\; 0 \qquad \text{for every } p \ge 1 ,
    \]
    by the choice of $\eta_p$.
    Applying an unscaled Young inequality to the lower-order term would instead add $k_p$ to the coefficient of $\Phi_p(x_0)$ and can make $a_p$ negative, which is what must be avoided when using Halanay inequality.

\stp{3: A priori integrability}
    Before applying the inequality, we must show that $m_p(t) := \E \Phi_p(X_t)$ is finite and locally absolutely continuous.
    Let $\zeta_n := \inf\{ s \ge 0 : \norm{X_s - c} \ge n \}$, which increases to $+\infty$ a.s.\ by Proposition~\ref{prop:strong_solution}, and set
    \[
    S_n(t) := \max\Bigl\{ \norm{X_0^{\seg} - c}_\Cseg^{2p} , \; \sup_{u \le t} \Phi_p(X_{u \wedge \zeta_n}) \Bigr\} ,
    \]
    so that $\E S_n(t) \le M_0 + n^{2p} < \infty$ and $S_n(s)$ dominates every delayed term $\Phi_p(X_{s - \tau_a})$ for $s \le t$ --- those with $s < \tau_a$ being read off the initial segment.
    
    It\^o's formula for $\Phi_p(X_{u \wedge \zeta_n})$ together with \Eqref{eq:gen-final}, after discarding the nonpositive term $-a_p \Phi_p$, gives the following estimate
    \[
    \E S_n(t) \;\le\; 2 M_0 + C_p \, t + 2B \int_0^t \E S_n(s) \, ds + \E \sup_{u \le t} \bigl| N_{u \wedge \zeta_n} \bigr| .
    \]
    where we wrote $N_t := \int_0^{t} 2p \, \Phi_1^{p-1}(X_s) (X_s - c)^\top \Sigma \, dW_s$ for the martingale part.
    
    Since $d[N]_s \le 4p^2 \norm{\Sigma}_{\rm op}^2 \, \Phi_1^{2p-1}(X_s) \, ds$, the Burkholder--Davis--Gundy inequality, the factorisation $\Phi_1^{2p-1} = \Phi_p \cdot \Phi_1^{p-1}$ and Young's inequality $\sqrt{uv} \le \tfrac14 u + v$ bound the last term by $\tfrac12 \E S_n(t) + C \int_0^t ( 1 + \E S_n(s) ) \, ds$, with $C$ depending only on $p$ and $\norm{\Sigma}_{\rm op}$.
    Absorbing $\tfrac12 \E S_n(t)$ on the left and applying Gr\"onwall's lemma gives a bound uniform in $n$ on every finite horizon; letting $n \to \infty$ and using Fatou's lemma,
    \[
    \E \sup_{u \le H} \norm{X_u - c}^{2p} < \infty \qquad \text{for every finite } H .
    \]
    Consequently, by the same Burkholder--Davis--Gundy bound, $\E \sup_{u \le H} |N_u| < \infty$, so $N$ is a true martingale on finite horizons and expectations may be taken in It\^o's formula without stopping; the positive part of $\mathcal{L}\Phi_p$ is integrable by \Eqref{eq:gen-final} and It\^o's identity then controls its negative part, so $m_p$ is locally absolutely continuous and, for almost every $t > 0$,
    \begin{equation}
    \label{eq:delay-ineq}
    m_p'(t) \;\le\; -a_p \, m_p(t) + 2 \sum_{a=1}^{q} b_a \, m_p(t - \tau_a) + C_p .
    \end{equation}
\stp{4: Halanay comparison}
    Consider the characteristic equation
    \begin{equation}\label{eq:gamma-root}
    \gamma_p + 2 \sum_{a=1}^{q} b_a \, e^{\gamma_p \tau_a} \;=\; a_p .
    \end{equation}
    Its left-hand side is strictly increasing in $\gamma_p$, equals $2B < a_p$ at $\gamma_p = 0$ by Step~2, and tends to infinity, so \Eqref{eq:gamma-root} has a unique positive root $\gamma_p \in (0 , a_p)$; when all $b_a$ vanish it reduces to $\gamma_p = a_p$.
    Put $K_p^\star := C_p / (a_p - 2B)$ and
    \[
    w(t) := M_0 \, e^{-\gamma_p t} + K_p^\star , \qquad t \ge -\taumax .
    \]
    A direct computation using \Eqref{eq:gamma-root} shows that $w$ satisfies \Eqref{eq:delay-ineq} with equality, and $w(t) \ge M_0 \ge m_p(t)$ on $[-\taumax , 0]$ since $e^{-\gamma_p t} \ge 1$ there.
    Set $u := m_p - w$, so that $u \le 0$ on $[-\taumax , 0]$ and $u'(t) \le -a_p u(t) + 2 \sum_a b_a \, u(t - \tau_a)$ almost everywhere.
    On $[0 , \taumin]$ all delayed values $u(t - \tau_a)$ are nonpositive, hence $(e^{a_p t} u)' \le 0$ and $u(t) \le u(0) e^{-a_p t} \le 0$; since $\taumin$ is the smallest delay, induction over the intervals $[k \taumin , (k+1) \taumin]$ propagates $u \le 0$ to all of $[0,\infty)$ (for $q = 0$ ordinary Gr\"onwall suffices).
    Therefore $m_p(t) \le w(t)$ for every $t \ge 0$, that is
    \[
    \E \norm{X_t - c}^{2p} \;\le\; \E\norm{X_0^{\seg} - c}_\Cseg^{2p} \, e^{-\gamma_p t} + K_p^\star , \qquad t \ge 0 ,
    \]
    which is the announced bound, with $\gamma_p$ the positive root of \Eqref{eq:gamma-root} and $K_p^\star = C_p / (a_p - 2B)$.
\end{proof}

With the one-point bound in hand, we lift to the segment norm and record the stationary and finite-horizon consequences (parts (b) and (c) of the lemma).

\begin{lemma}[Segment moment bounds]\label{lem:moments}
    Under the assumptions of \Cref{lem:moments-point} and for the same range of $p$, write $U_p(\varphi):=\norm{\varphi-c}_\Cseg^{2p}$ for $\varphi\in\Cseg$ and let $(P_t)_{t\ge0}$ denote the semigroup of the segment process on $\Cseg$.
    \begin{enumerate}
        \item[(a)] There are constants $A_p^{\seg},B_p^{\seg}<\infty$ and $\gamma_p>0$ such that
        \begin{equation}\label{eq:moment-segment}
        P_tU_p(\phi)=\E_\phi\norm{X_t^{\seg}-c}_{\Cseg}^{2p}\;\le\;A_p^{\seg}\,e^{-\gamma_p t}\,U_p(\phi)+B_p^{\seg},\qquad t\ge0,\ \phi\in\Cseg .
        \end{equation}
        In particular $M_{2p}(\phi):=\sup_{t\ge0}\E_\phi\norm{X_t^{\seg}-c}_{\Cseg}^{2p}\le A_p^{\seg}\,U_p(\phi)+B_p^{\seg}<\infty$.
        \item[(b)] If the segment process admits an invariant probability measure $\pi^{\seg}$ on $\Cseg$, then
        \begin{equation}\label{eq:moment-stationary}
        \int_{\Cseg}\norm{\eta-c}_{\Cseg}^{2p}\,\pi^{\seg}(d\eta)<\infty .
        \end{equation}
        \item[(c)] For every finite horizon $H$ there is $C_{p,H}<\infty$ with
        \begin{equation}\label{eq:moment-horizon}
        \E_\phi\sup_{t\le H}\norm{X_t^{\seg}-c}_{\Cseg}^{2p}\;\le\;C_{p,H}\bigl(1+U_p(\phi)\bigr),\qquad \phi\in\Cseg,
        \end{equation}
        which is uniform over any ball $B_R:=\{\phi\in\Cseg:\norm{\phi-c}_{\Cseg}\le R\}$.
    \end{enumerate}
\end{lemma}

\begin{proof}
    Recall $\Phi_1(x)=\norm{x-c}^2$, $\Phi_p=\Phi_1^{\,p}$ and $\sigma^2:=\norm{\Sigma}_{\rm op}^2$, and note that $U_p(X_t^{\seg})=\sup_{u\in[t-\taumax,\,t]}\Phi_p(X_u)$, so that (a) is exactly an upgrade of the one-point control of \Cref{lem:moments-point} to a supremum over a window of length $\taumax$.
    We keep the constants $b_i$, $B=\sum_{i=1}^qb_i$, $C_p$, $\gamma_p$ and $K_p^\star$ of that lemma, together with its generator bound \Eqref{eq:gen-final}.
    All suprema below are first taken up to the localising time $\theta_N:=\inf\{s:\norm{X_s-c}>N\}$; the resulting constants do not depend on $N$, and $\theta_N\uparrow\infty$ a.s.\ by \Cref{prop:strong_solution}, so monotone convergence removes the localisation at the end of each step.
    The finite-horizon integrability established in the proof of \Cref{lem:moments-point} guarantees that every expectation written below is finite before the limit is taken.
    When $q=0$ we use the conventions $\taumax=0$ and $\Cseg=\R^D$: then $U_p(X_t^{\seg})=\Phi_p(X_t)$, part (a) is \Cref{lem:moments-point} itself, and in Step~5 the windows of length $\taumax$ are replaced by windows of unit length, to which the window estimate \Eqref{eq:seg-window} of Step~2 applies verbatim.

\stp{1: A one-point envelope valid on negative times}
    Set $m_p(s):=\E_\phi\Phi_p(X_s)$ for $s\ge-\taumax$.
    For $s\in[-\taumax,0]$ the path is the deterministic initial segment, so $m_p(s)\le U_p(\phi)$.
    Since moreover $m_p(s)\le U_p(\phi)e^{-\gamma_ps}+K_p^\star$ for $s\ge0$ by \Cref{lem:moments-point}, the inequality
    \begin{equation}\label{eq:envelope}
    m_p(s)\;\le\;U_p(\phi)\,e^{-\gamma_ps}+K_p^\star
    \end{equation}
    holds for every $s\ge-\taumax$.

\stp{2: Window estimate}
    Fix $0\le v<w$ with $w-v\le\taumax$.
    In integrated form, for $u\in[v,w]$,
    \[
    \Phi_p(X_u)=\Phi_p(X_v)+\int_v^u\mathcal{L}\Phi_p(X_s)\,ds+(N_u-N_v),
    \]
    with $N_s:=\int_0^s2p\,\Phi_1^{p-1}(X_r)\,(X_r-c)^\top\Sigma\,dW_r$ as in the previous lemma.
    
    Discarding the favourable term $-a_p\Phi_p\le0$ in \Eqref{eq:gen-final}, taking the supremum over $u\in[v,w]$ and then expectations,
    \begin{equation}\label{eq:seg-sup}
    \E\sup_{v\le u\le w}\Phi_p(X_u)\;\le\;m_p(v)+2\sum_{i=1}^{q}b_i\int_v^wm_p(s-\tau_i)\,ds+C_p(w-v)+\E\sup_{v\le u\le w}\bigl|N_u-N_v\bigr| .
    \end{equation}
    Re-anchored at $v$, the increment $u\mapsto N_u-N_v$ is a continuous local martingale started at $0$, so the Burkholder--Davis--Gundy inequality at exponent $1$ applies with a universal constant $C_{\mathrm{BDG}}$.
    Using $d[N]_s\le4p^2\sigma^2\Phi_1^{2p-1}(X_s)\,ds$ and the factorisation $\Phi_1^{2p-1}=\Phi_p\cdot\Phi_1^{p-1}$,
    \[
    \E\sup_{v\le u\le w}\bigl|N_u-N_v\bigr|\;\le\;2p\sigma C_{\mathrm{BDG}}\;\E\left[\Bigl(\sup_{v\le u\le w}\Phi_p(X_u)\Bigr)^{1/2}\Bigl(\int_v^w\Phi_1^{p-1}(X_s)\,ds\Bigr)^{1/2}\right] .
    \]
    Young's inequality $ab\le\tfrac12a^2+\tfrac12b^2$, with the constant absorbed into the second factor and $C_N:=(2p\sigma C_{\mathrm{BDG}})^2$, gives
    \begin{equation}\label{eq:seg-mart}
    \E\sup_{v\le u\le w}\bigl|N_u-N_v\bigr|\;\le\;\tfrac12\,\E\sup_{v\le u\le w}\Phi_p(X_u)+\tfrac{C_N}{2}\int_v^w\E\,\Phi_1^{p-1}(X_s)\,ds ,
    \end{equation}
    and $\Phi_1^{p-1}\le\tfrac{p-1}{p}\Phi_p+\tfrac1p\le\Phi_p+1$ bounds the last integral by $\int_v^w(1+m_p(s))\,ds$.
    Substituting \Eqref{eq:seg-mart} into \Eqref{eq:seg-sup} and absorbing the half-supremum into the left-hand side,
    \begin{equation}\label{eq:seg-window}
    \E\sup_{v\le u\le w}\Phi_p(X_u)\;\le\;2m_p(v)+4\sum_{i=1}^{q}b_i\int_v^wm_p(s-\tau_i)\,ds+2C_p(w-v)+C_N\int_v^w\bigl(1+m_p(s)\bigr)\,ds .
    \end{equation}

\stp{3: Geometric drift for the segment process}
    Let $t\ge\taumax$ and apply \Eqref{eq:seg-window} with $v=t-\taumax$ and $w=t$.
    Every time argument occurring on its right-hand side lies in $[t-2\taumax,\,t]$, so \Eqref{eq:envelope} bounds each of them by $e^{2\gamma_p\taumax}U_p(\phi)e^{-\gamma_pt}+K_p^\star$.
    Writing $\Lambda_p:=2+4B\taumax+C_N\taumax$, we obtain \Eqref{eq:moment-segment} for $t\ge\taumax$ with
    \[
    A_p^{\seg}:=\Lambda_p\,e^{2\gamma_p\taumax},\qquad B_p^{\seg}:=\Lambda_p K_p^\star+2C_p\taumax+C_N\taumax .
    \]
    For $0\le t<\taumax$ the window $[t-\taumax,\,t]$ straddles the origin, and the initial history is not an It\^o process on its negative part.
    We therefore split
    \[
    U_p(X_t^{\seg})=\sup_{u\in[t-\taumax,\,t]}\Phi_p(X_u)\;\le\;U_p(\phi)+\sup_{0\le u\le t}\Phi_p(X_u),
    \]
    bound the first term directly and the second by \Eqref{eq:seg-window} with $v=0$, $w=t$, whose length is at most $\taumax$.
    All time arguments in the latter lie in $[-\taumax,t]\subseteq[t-2\taumax,t]$, so the computation above bounds the expectation of the second term by $A_p^{\seg}e^{-\gamma_pt}U_p(\phi)+B_p^{\seg}$, while the first term is at most $e^{\gamma_p\taumax}e^{-\gamma_pt}U_p(\phi)$ since $e^{-\gamma_pt}\ge e^{-\gamma_p\taumax}$ on this range.
    Replacing $A_p^{\seg}$ by $A_p^{\seg}+e^{\gamma_p\taumax}$ therefore extends \Eqref{eq:moment-segment} to all $t\ge0$.
    Taking the supremum over $t\ge0$ gives the uniform bound $M_{2p}(\phi)\le A_p^{\seg}U_p(\phi)+B_p^{\seg}$ announced in (a), established without any growth assumption on $G$.

\stp{4: Stationary integrability (b)}
    Choose $h>0$ so that $\kappa:=A_p^{\seg}e^{-\gamma_ph}<1$ and write $\varrho:=B_p^{\seg}$, so that \Eqref{eq:moment-segment} reads $P_hU_p\le\kappa U_p+\varrho$ pointwise on $\Cseg$.
    One cannot integrate this against $\pi^{\seg}$ and cancel the two sides, since $\int U_p\,d\pi^{\seg}$ is not yet known to be finite.
    Therefore, set $u_N:=U_p\wedge N$, a bounded measurable function.
    Concavity of $v\mapsto v\wedge N$ and Jensen's inequality give $P_hu_N\le(P_hU_p)\wedge N\le(\kappa U_p+\varrho)\wedge N$, and invariance of $\pi^{\seg}$ applied to the bounded function $u_N$ yields
    \[
    \int_{\Cseg}\Bigl[u_N-\bigl((\kappa U_p+\varrho)\wedge N\bigr)+\varrho\Bigr]\,d\pi^{\seg}\;\le\;\varrho .
    \]
    The integrand is nonnegative, because $\kappa\le1$ implies $(\kappa U_p+\varrho)\wedge N\le(U_p\wedge N)+\varrho=u_N+\varrho$, and it converges pointwise to $(1-\kappa)U_p$ as $N\to\infty$.
    Fatou's lemma therefore gives $(1-\kappa)\int_{\Cseg}U_p\,d\pi^{\seg}\le\varrho<\infty$, which is \Eqref{eq:moment-stationary}.
    The argument applies separately for each $p$ and uses no moment of $\pi^{\seg}$ in its hypotheses.

\stp{5: Finite horizon (c)}
    Cover $[0,H]$ by the $n:=\lceil H/\taumax\rceil$ windows $[k\taumax,(k+1)\taumax]\cap[0,H]$, $k=0,\dots,n-1$, each of length at most $\taumax$.
    Since the supremum over a union is at most the sum of the suprema,
    \[
    \E_\phi\sup_{t\le H}U_p(X_t^{\seg})=\E_\phi\sup_{u\in[-\taumax,\,H]}\Phi_p(X_u)\;\le\;U_p(\phi)+\sum_{k=0}^{n-1}\E_\phi\sup_{u\in[k\taumax,(k+1)\taumax]\wedge H}\Phi_p(X_u),
    \]
    and each summand is bounded by \Eqref{eq:seg-window} together with the envelope \Eqref{eq:envelope}, which gives $m_p(s)\le U_p(\phi)+K_p^\star$ for every $s\ge-\taumax$.
    Collecting the $n$ contributions yields \Eqref{eq:moment-horizon} with a constant $C_{p,H}$ depending only on $p$, $H$, the delays and the constants of \Cref{lem:moments-point}, and in particular uniform over $\phi\in B_R$ for each fixed $R$.
\end{proof}

\subsection{Ergodicity of the true process}
\label{subsec:ergodicity}

Having those moment estimates both in the segment space and point-wise at hand, it is now possible to show that the true process admits a unique invariant probability and that we have geometric ergodicity in Wasserstein-$d$ distance for some distance on the segment space. The proof is based on the application of the generalized Harris theorem for delayed SDEs by \citet{HMS2011}.

Besides the Lyapunov condition provided by \Cref{lem:moments}, Harris' theorem needs the transition laws started from any two histories in a bounded set to  overlap. For SDDEs this is the tricky part: indeed, the two initial segments may differ over the whole window $[-\taumax, 0]$, while the noise only acts on the present. The idea of the proof is to steer one copy of the process onto the other with a control acting on $[0, h]$. After time $h > \taumax$ the two segments coincide, and since $\Sigma$ is invertible, Girsanov's theorem shows that the price paid for this control is a change of measure with a bounded energy.

\begin{lemma}[Uniform overlap of the segment transition kernels]\label{lem:segment-overlap}
    Under Assumptions~\ref{assu:dgp}, \ref{assu:regularity} and \ref{ass:dissip}, fix $h>\taumax$ and $R_0>0$, and set $B_{R_0}:=\{\varphi\in\Cseg:\norm{\varphi-c}_{\Cseg}\le R_0\}$.
    There exists $\delta>0$ depending on $h$ and $R_0$ (and on $G$ through a local Lipschitz constant, see \Eqref{eq:ergodic-control-energy}) such that
    \begin{equation}\label{eq:segment-overlap}
        \sup_{\varphi,\psi\in B_{R_0}}
        \norm{P_h(\varphi,\cdot)-P_h(\psi,\cdot)}_{\mathrm{TV}}
        \le 1-\delta.
    \end{equation}
\end{lemma}
\begin{proof}
    Fix $\varphi,\psi\in B_{R_0}$, let $X$ start from $\varphi$, and put $t_1:=h-\taumax>0$ and $v:=\norm{\varphi-\psi}_{\Cseg}\le2R_0$.
    Define the deterministic difference pilot path
    \[
        g(t):=
        \begin{cases}
        \varphi(t)-\psi(t), & -\taumax\le t\le0,\\
        (1-t/t_1)\bigl(\varphi(0)-\psi(0)\bigr), & 0\le t\le t_1,\\
        0, & t_1\le t\le h.
        \end{cases}
    \]
    Then $g$ is continuous, $\sup_t\norm{g(t)}\le v$, and its restriction to $[0,h]$ is absolutely continuous with $\norm{g'(t)}\le(v/t_1)\mathbf{1}_{(0,t_1)}(t)$ almost everywhere.

\stp{1: Localized control}
    The finite-horizon estimate in \Cref{lem:moments} gives
    \[
    \sup_{\varphi\in B_{R_0}}
    \E_\varphi\sup_{0\le t\le h}\norm{X_t-c}^{2}<\infty.
    \]
    Consequently, we can choose a deterministic $M>R_0$ such that, with $\sigma:=\inf\{t\ge0:\norm{X_t-c}\ge M\}$,
    \[
    \inf_{\varphi\in B_{R_0}}\Prob_\varphi(\sigma>h)\ge\tfrac12.
    \]
    Write $Z_t:=(X_t,X_{t-\tau_1},\dots,X_{t-\tau_q})$ and $g_t^{\mathrm{arg}}:=(g(t),g(t-\tau_1),\dots,g(t-\tau_q))$.
    Define the progressively measurable control
    \begin{equation}\label{eq:ergodic-local-control}
    u_t:=\mathbf{1}_{\{t<\sigma\}}\Sigma^{-1}
    \bigl[G(Z_t)-G(Z_t-g_t^{\mathrm{arg}})-g'(t)\bigr],
    \qquad 0\le t\le h.
    \end{equation}
    All arguments in this expression belong to a fixed compact subset of $\R^{D(q+1)}$, uniformly over $\varphi,\psi\in B_{R_0}$ in the finite times $t < \sigma$.
    Let $L$ be a Lipschitz constant of $G$ on a ball containing this set.
    Since $\norm{g_t^{\mathrm{arg}}}\le\sqrt{q+1}\,v$, the control has the deterministic energy bound
    \begin{equation}\label{eq:ergodic-control-energy}
    \int_0^h\norm{u_t}^{2}\,dt
    \le 8R_0^2\norm{\Sigma^{-1}}_{\mathrm{op}}^2
    \left((q+1)L^2h+\frac1{t_1}\right)
    =:K_{h,R_0}<\infty.
    \end{equation}
    
    Set $Y_t=X_t-g(t)$ up to $\sigma$, and, when $\sigma<h$, continue $Y$ after $\sigma$ using the original SDDE with the same Brownian motion and its segment at $\sigma$ as initial condition.
    Strong existence and pathwise uniqueness justify this continuation.
    By substitution, $Y$ satisfies
    \[
    dY_t=G(Y_t,Y_{t-\tau_1},\dots,Y_{t-\tau_q})\,dt
    +\Sigma\,dW_t+\Sigma u_t\,dt,
    \qquad Y_0^{\seg}=\psi.
    \]
    On the event $E:=\{\sigma>h\}$, the two terminal segments agree because $g=0$ on $[h-\taumax,h]$.
    Thus
    \[
    Y_h^{\seg}=X_h^{\seg}\quad\text{on }E,
    \qquad \text{with } \Prob(E)\ge\tfrac12.
    \]
\stp{2: Correction of the controlled law}
    Define
    \[
    \Lambda:=\exp\left(-\int_0^h u_t\cdot dW_t
    -\frac12\int_0^h\norm{u_t}^{2}\,dt\right),
    \qquad d\mathbb{Q}:=\Lambda\,d\Prob.
    \]
    The deterministic bound \Eqref{eq:ergodic-control-energy} implies Novikov's condition.
    Under $\mathbb{Q}$, the process $W_t+\int_0^t u_s\,ds$ is Brownian, so uniqueness in law gives $\mathcal{L}_{\mathbb{Q}}(Y_h^{\seg})=P_h(\psi,\cdot)$.
    The exponential martingale with integrand $u$ also has expectation one, and therefore
    \[
    \E[\Lambda^{-1}]
    =\E\left[
    \exp\left(\int_0^h u_t\cdot dW_t-\frac12\int_0^h\norm{u_t}^{2}\,dt\right)
    \exp\left(\int_0^h\norm{u_t}^{2}\,dt\right)
    \right]
    \le e^{K_{h,R_0}}.
    \]
    Set $\ell:=e^{-K_{h,R_0}}/4$.
    Markov's inequality gives $\Prob(\Lambda<\ell)\le1/4$, hence
    \[
    \E\bigl[\mathbf{1}_E\min(1,\Lambda)\bigr]
    \ge \ell\,\Prob\bigl(E\cap\{\Lambda\ge\ell\}\bigr)
    \ge \frac{e^{-K_{h,R_0}}}{16}.
    \]
    For a Borel set $A\subseteq\Cseg$, define
    \[
    \nu(A):=\E\bigl[\mathbf{1}_E\min(1,\Lambda)
    \mathbf{1}_{\{X_h^{\seg}\in A\}}\bigr].
    \]
    Since $X_h^{\seg}=Y_h^{\seg}$ on $E$, this measure is dominated by both $P_h(\varphi,\cdot)$ and $P_h(\psi,\cdot)$.
    Its mass is at least $e^{-K_{h,R_0}}/16$, which proves \Eqref{eq:segment-overlap} with $\delta := e^{-K_{h,R_0}}/16$.
\end{proof}

\begin{manualtheorem}{thm:ergodic}{Ergodicity: existence, uniqueness, exponential convergence}
Under Assumptions~\ref{assu:dgp}, \ref{assu:regularity} and \ref{ass:dissip}, the segment process admits a unique invariant probability measure $\pi^{\seg}$ on $\Cseg$.
Write $P_t(\varphi,\cdot):=\mathcal{L}(X_t^{\seg}\mid X_0^{\seg}=\varphi)$ and $V(\varphi):=\norm{\varphi-c}_{\Cseg}^{2}$, where $c$ also denotes the constant segment with value $c$.
The invariant measure satisfies
\[
\int_{\Cseg}\norm{\eta-c}_{\Cseg}^{2p}\,\pi^{\seg}(d\eta)<\infty,
\qquad p\ge1 \text{ an integer}.
\]
There exist $C\ge1$ and $\gamma>0$ such that, for every $\epsilon>0$, the metric $d(\varphi,\psi):=1\wedge\epsilon^{-1}\norm{\varphi-\psi}_{\Cseg}$ satisfies
\begin{equation*}
\mathcal{W}_d\bigl(P_t(\varphi,\cdot),\pi^{\seg}\bigr)
\le \norm{P_t(\varphi,\cdot)-\pi^{\seg}}_{\mathrm{TV}}
\le C e^{-\gamma t}\bigl(1+V(\varphi)\bigr),
\qquad t\ge0,\quad \varphi\in\Cseg.
\tag{\ref{eq:ergodic-tv-wasserstein}}
\end{equation*}
\end{manualtheorem}



\begin{proof}
By \Cref{prop:strong_solution}, time-homogeneity and pathwise uniqueness define a Markov semigroup $(P_t)_{t\ge0}$ on $\Cseg$.

\stp{1: Lyapunov estimate}
    For $V(\varphi)=\norm{\varphi-c}_{\Cseg}^{2}$, the decaying segment-moment bound in \Cref{lem:moments} provides constants $A_V\ge1$, $B_V<\infty$ and $\gamma_V>0$ such that
    \begin{equation}\label{eq:ergodic-sup-lyapunov}
    P_tV(\varphi)\le A_Ve^{-\gamma_Vt}V(\varphi)+B_V,
    \qquad t\ge0.
    \end{equation}
    Choose $h>\taumax$ sufficiently large that $A_Ve^{-\gamma_Vh}<1$.
    Every sublevel set $\{V\le S\}$ is the supremum-norm ball $B_{\sqrt S}$, and \Cref{lem:segment-overlap} gives uniform transition overlap on this set for $P_h$.
    Since \Cref{lem:segment-overlap} holds for every radius $R_0$ at the same time $h$, this overlap is available on every sublevel set of $V$, however large; in particular, the size of the level set required in Harris' theorem (of order $B_V/(1-A_Ve^{-\gamma_Vh})$) imposes no restriction.

\stp{2: Harris' theorem in total variation}
    The estimate \Eqref{eq:ergodic-sup-lyapunov} is a Lyapunov condition, and \Eqref{eq:segment-overlap} is precisely the pairwise total-variation small-set condition in \citet[Theorem~1.5]{HMS2011}.
    The choice of $h$ satisfies the time requirement $h>\gamma_V^{-1}\log A_V$ in that theorem.
    It follows that $(P_t)_{t\ge0}$ has a unique invariant probability measure $\pi^{\seg}$ and that, for some $C\ge1$ and $\gamma>0$,
    \[
    \norm{P_t(\varphi,\cdot)-\pi^{\seg}}_{\mathrm{TV}}
    \le C e^{-\gamma t}\bigl(1+V(\varphi)\bigr),
    \qquad t\ge0,\quad\varphi\in\Cseg.
    \]
    Only pairwise overlap is required here; no common minorising measure for an entire sublevel set is asserted.

\stp{3: Stationary moments and Wasserstein convergence}
    Now that an invariant measure exists, the stationary-integrability part of \Cref{lem:moments} gives all the stated polynomial moments.
    For any two probability measures $\mu,\nu$ on $\Cseg$, a maximal coupling has probability $\norm{\mu-\nu}_{\mathrm{TV}}$ of unequal coordinates.
    Since $d\le1$ and $d(\varphi,\varphi)=0$, this coupling gives $\mathcal{W}_d(\mu,\nu)\le\norm{\mu-\nu}_{\mathrm{TV}}$.
    Combining this inequality with the preceding total-variation estimate proves \Eqref{eq:ergodic-tv-wasserstein}.
\end{proof}

\begin{remark}[Stabilisation of the identifiability constants]\label{rem:stabilise}
Under Theorem~\ref{thm:ergodic}, averaging \Eqref{eq:ergodic-tv-wasserstein} over $t \in [\taumax, T]$ against the initial law and applying the evaluation map at the lags gives
\[
\norm{\occq - \mu_\infty^q}_{\mathrm{TV}} \le \frac{C\bigl(1 + \E V(X_0^{\seg})\bigr)}{\gamma (T - \taumax)}\bigl(e^{-\gamma\taumax} - e^{-\gamma T}\bigr),
\]
where $\mu_\infty^q$ is the stationary joint law of $(X_0, X_{-\tau_1}, \dots, X_{-\tau_q})$ under $\pi^{\mathrm{seg}}$. Total-variation convergence alone does not control squared errors of an unbounded drift (the stationary drift need not even be square-integrable). If, in addition, $\norm{G}^2$ is uniformly integrable with respect to the family $\{\occq\}_{T > \taumax} \cup \{\mu_\infty^q\}$, then the signal strengths $s_{ij}^a$ and hence $\smin$ stabilize as $T \to \infty$ to their stationary counterparts,
and the $\lambda$ window \Eqref{eq:lambda-window} scales linearly in $T$ with an asymptotically constant slope. This is what makes the finite-horizon constants of \Cref{thm:penalized} meaningful uniformly in $T$, rather than per-$(T, \text{initial law})$. The finite-sample analysis of \Cref{subsec:finite-sample} avoids this issue by working directly under the stationary law and imposing a finite stationary drift energy in \Eqref{eq:finite-stationary-energy} explicitly.
\end{remark}

\subsection{Mixing of the sampled process}
\label{subsec:mixing}

The finite-sample analysis of \Cref{subsec:finite-sample} requires the sampled lagged state to forget its past at an exponential rate, in the sense of $\beta$-mixing.
Since $\beta$-mixing is a total-variation property, a Wasserstein contraction of the segment process would not be enough.
\Cref{thm:ergodic}, however, already provides geometric ergodicity in total variation: despite the infinite-dimensional state space, the noise is additive and non-degenerate, so that for $h>\taumax$ the kernels $P_h(\varphi,\cdot)$ and $P_h(\psi,\cdot)$ overlap (\Cref{lem:segment-overlap}).
The mixing rate is therefore a corollary of \Cref{thm:ergodic}.
We first bound the $\beta$-coefficient of a stationary Markov process by the $\pi$-average distance of its kernel from equilibrium (\Cref{lem:markov-beta}), and then pass from the segment process to the lagged state and to its deterministic subsamples by monotonicity (\Cref{thm:exp-mixing}).
No assumption beyond those of \Cref{thm:ergodic} is needed, and the non-delayed case $q=0$ is included.

\begin{definition}[$\beta$-mixing coefficient]\label{def:beta-mixing}
    We normalize the total-variation distance of two probability measures $\mu,\nu$ on a common measurable space as
    \[
    \norm{\mu-\nu}_{\mathrm{TV}} := \sup_{A}\lvert\mu(A) - \nu(A)\rvert
     = \tfrac12 \int \lvert\mathrm{d}\mu-\mathrm{d}\nu\rvert\ \in[0,1].
    \]
    For sub-$\sigma$-algebras $\mathcal A,\mathcal B\subseteq\mathcal F$, the \emph{$\beta$-mixing (absolute regularity) coefficient} is
    \[
    \beta(\mathcal A,\mathcal B) := \sup\ \tfrac12 \sum_{i=1}^{I} \sum_{j=1}^{J} \left\lvert\Prob(A_i\cap B_j) - \Prob(A_i)\,\Prob(B_j)\right\rvert,
    \]
    the supremum over all finite $\mathcal A$-measurable partitions $\{A_i\}_{i\le I}$ and $\mathcal B$-measurable partitions $\{B_j\}_{j\le J}$ of $\Omega$ \citep{Bradley2005}.
    Since any pair of partitions admissible for $(\mathcal A',\mathcal B')$ is admissible for $(\mathcal A,\mathcal B)$,
    \begin{equation}\label{eq:beta-monotone}
        \mathcal A'\subseteq\mathcal A,\ \mathcal B'\subseteq\mathcal B
        \quad\Longrightarrow\quad
        \beta(\mathcal A',\mathcal B')\le\beta(\mathcal A,\mathcal B).
    \end{equation}
    For a process $(Y_s)_{s\ge s_0}$, its \emph{$\beta$-mixing coefficient at gap $t>0$} is
    \begin{equation}\label{eq:beta-def}
        \beta_Y(t):=\sup_{s\ge s_0}\ \beta\Bigl(\sigma(Y_u:\ s_0\le u\le s),\ \sigma(Y_u:\ u\ge s+t)\Bigr),
    \end{equation}
    and for a sequence $(Y_k)_{k\ge0}$ we set $\beta_Y(m):=\sup_{k\ge0}\beta\bigl(\sigma(Y_j:j\le k),\sigma(Y_j:j\ge k+m)\bigr)$, $m\ge1$.
    By \Eqref{eq:beta-monotone}, $\beta_Y$ is nonincreasing.
    The process is \emph{$\beta$-mixing} if $\beta_Y(t)\to 0$ as $t\to\infty$, and \emph{exponentially $\beta$-mixing} if $\beta_Y(t)\le C_\beta\,e^{-\lambda_{\mathrm{erg}}t}$ for constants $C_\beta<\infty$, $\lambda_{\mathrm{erg}}>0$.
    For a stationary process indexed by $\R$ with past $\sigma(Y_u:u\le s)$, the supremum is redundant and one recovers the classical coefficient $\beta(\mathcal F_{\le 0},\mathcal F_{\ge t})$; the one-sided form \Eqref{eq:beta-def} avoids constructing a two-sided extension of the solution.
\end{definition}

\begin{lemma}[$\beta$-coefficient of a stationary Markov process]\label{lem:markov-beta}
    Let $(\xi_s)_{s\ge0}$ be a time-homogeneous Markov process on a Polish space $E$, Markov with respect to its natural filtration $\mathcal G_s:=\sigma(\xi_u:0\le u\le s)$, with transition kernels $(P_t)_{t\ge0}$ and invariant probability measure $\pi$, and started from $\xi_0\sim\pi$.
    Then, for every $s\ge0$ and $t>0$,
    \begin{equation}\label{eq:beta-markov}
        \beta\Bigl(\mathcal G_s,\ \sigma(\xi_u:u\ge s+t)\Bigr)
        =\int_E\norm{P_t(x,\cdot)-\pi}_{\mathrm{TV}}\,\pi(\mathrm{d}x),
    \end{equation}
    so that, in particular, $\beta_\xi(t)$ is given by the right-hand side.
\end{lemma}
This identity is due to \citet{Davydov1973} for Markov chains on a general state space, and its continuous-time version is standard (see, e.g., \citet{Masuda2007}).
Only the inequality ``$\le$'' is used below.
It follows by conditioning each future event on $\mathcal G_s$: by the Markov property and stationarity, the conditional law of $(\xi_{s+t+r})_{r\ge0}$ is the law of the process started from $P_t(\xi_s,\cdot)$, while its unconditional law is the law of the process started from $\pi$, and a Markov kernel does not increase total variation.
For the reverse inequality, restrict both $\sigma$-algebras to the endpoints using \Eqref{eq:beta-monotone}: $\beta(\sigma(\xi_s),\sigma(\xi_{s+t}))$ is the total-variation distance between the joint law $\pi(\mathrm{d}x)P_t(x,\mathrm{d}y)$ of $(\xi_s,\xi_{s+t})$ and the product law $\pi(\mathrm{d}x)\pi(\mathrm{d}y)$, which equals the right-hand side of \Eqref{eq:beta-markov}.

\begin{theorem}[Exponential $\beta$-mixing of the lagged state]\label{thm:exp-mixing}
    Let Assumptions~\ref{assu:dgp}, \ref{assu:regularity} and \ref{ass:dissip} hold, let $C\ge1$ and $\gamma>0$ be the constants of \Cref{thm:ergodic}, and take the process in its stationary regime, $X_0^{\seg}\sim\pi^{\seg}$ independent of $W$.
    Set
    \[
        M_2^\pi:=\int_{\Cseg}\norm{\eta-c}_{\Cseg}^2\,\pi^{\seg}(\mathrm{d}\eta)<\infty,
        \qquad
        C_\beta:=C\bigl(1+M_2^\pi\bigr),
    \]
    and write $Z_t:=(X_t,X_{t-\tau_1},\dots,X_{t-\tau_q})$, $t\ge0$. Then:
    \begin{enumerate}[label=(\roman*)]
        \item the segment process is exponentially $\beta$-mixing: $\beta_{X^{\seg}}(t)\le C_\beta e^{-\gamma t}$ for all $t>0$;
        \item the lagged state is exponentially $\beta$-mixing: $\beta_{Z}(t)\le C_\beta e^{-\gamma t}$ for all $t>0$, and the same bound holds for $(Z_t)_{t\ge\taumax}$;
        \item for every deterministic grid $0\le t_0<t_1<\cdots$, the sampled sequence $Z_k:=Z_{t_k}$ satisfies, for all $k\ge0$ and $m\ge1$,
        \begin{equation}\label{eq:beta-bound-grid}
            \beta\Bigl(\sigma(Z_j:j\le k),\ \sigma(Z_j:j\ge k+m)\Bigr)\le C_\beta\,e^{-\gamma(t_{k+m}-t_k)};
        \end{equation}
        in particular $\beta_{(Z_k)}(m)\le C_\beta e^{-\gamma\Delta_{\min}m}$ whenever $\min_k(t_{k+1}-t_k)\ge\Delta_{\min}>0$.
    \end{enumerate}
    Consequently, the physical-time mixing estimate \Eqref{eq:finite-beta} used in \Cref{subsec:finite-sample} is implied by Assumptions~\ref{assu:dgp}, \ref{assu:regularity} and \ref{ass:dissip}, with $\lambda_{\mathrm{erg}}:=\gamma$ and $C_\beta$ as above.
\end{theorem}
\begin{proof}
\stp{1: Stationary regime and Markov property}
    By \Cref{thm:ergodic} with $p=1$, $M_2^\pi<\infty$, so the initial condition $X_0^{\seg}\sim\pi^{\seg}$ satisfies the initial condition of Assumption~\ref{assu:dgp}, and \Cref{prop:strong_solution} provides a unique strong solution; invariance of $\pi^{\seg}$ gives $X_s^{\seg}\sim\pi^{\seg}$ for every $s\ge0$.
    Finite stationary drift energy is not needed for mixing; it is imposed separately in \Cref{subsec:finite-sample}.
    As in the proof of \Cref{thm:ergodic}, $(X_t^{\seg})_{t\ge0}$ is a time-homogeneous Markov process with kernels $(P_t)_{t\ge0}$ with respect to the filtration generated by $X_0^{\seg}$ and $W$.
    Since its natural filtration $\mathcal G_s:=\sigma(X_u^{\seg}:0\le u\le s)$ is smaller and $X^{\seg}$ is adapted to it, the tower property shows that $X^{\seg}$ is also Markov with respect to $(\mathcal G_s)$.
    The space $\Cseg$ is a separable Banach space, hence Polish, so \Cref{lem:markov-beta} applies with $E=\Cseg$ and $\pi=\pi^{\seg}$.

\stp{2: Segment process, item (i)}
    Fix $s\ge0$ and $t>0$.
    By \Cref{lem:markov-beta} and \Eqref{eq:ergodic-tv-wasserstein},
    \begin{align}
        \beta\Bigl(\mathcal G_s,\ \sigma(X_u^{\seg}:u\ge s+t)\Bigr)
        &= \int_{\Cseg}\norm{P_t(\eta,\cdot)-\pi^{\seg}}_{\mathrm{TV}}\,\pi^{\seg}(\mathrm{d}\eta)\nonumber\\
        &\le C e^{-\gamma t}\int_{\Cseg}\bigl(1+\norm{\eta-c}_{\Cseg}^2\bigr)\,\pi^{\seg}(\mathrm{d}\eta)
        = C_\beta\,e^{-\gamma t}.\label{eq:beta-segment}
    \end{align}
    Taking the supremum over $s\ge0$ gives (i).

\stp{3: Lagged state, item (ii)}
    The evaluation map $\mathrm{ev}:\Cseg\to\R^{D(q+1)}$, $\mathrm{ev}(\eta):=(\eta(0),\eta(-\tau_1),\dots,\eta(-\tau_q))$, is continuous, and $Z_u=\mathrm{ev}(X_u^{\seg})$ for every $u\ge0$.
    Hence
    \[
        \sigma(Z_u:0\le u\le s)\subseteq\mathcal G_s,
        \qquad
        \sigma(Z_u:u\ge s+t)\subseteq\sigma(X_u^{\seg}:u\ge s+t),
    \]
    and \Eqref{eq:beta-monotone} together with \Eqref{eq:beta-segment} gives $\beta_Z(t)\le C_\beta e^{-\gamma t}$.
    Restricting the index set to $u\ge\taumax$ shrinks both $\sigma$-algebras and the range of the supremum in \Eqref{eq:beta-def}, so the same bound holds for $(Z_t)_{t\ge\taumax}$.

\stp{4: Deterministic subsamples, item (iii)}
    Fix $k\ge0$, $m\ge1$ and set $s:=t_k$, $t:=t_{k+m}-t_k>0$.
    Then $\sigma(Z_j:j\le k)\subseteq\mathcal G_s$ and $\sigma(Z_j:j\ge k+m)\subseteq\sigma(X_u^{\seg}:u\ge s+t)$, so \Eqref{eq:beta-bound-grid} follows from \Eqref{eq:beta-monotone} and \Eqref{eq:beta-segment}.
    No regularity of the spacings is used.
    If all spacings are at least $\Delta_{\min}$, then $t_{k+m}-t_k\ge m\Delta_{\min}$, and taking the supremum over $k$ gives the last claim.
\end{proof}

\begin{remark}[Scope of \Cref{thm:exp-mixing}]\label{rem:mixing-scope}
    For $q=0$ the segment space reduces to $\R^D$ ($\taumax=0$) and \Cref{thm:exp-mixing} covers the Markovian SDE case.
    The additive form of the noise is what makes the total-variation proof valid: if the diffusion coefficient depended on the delayed state, the quadratic variation of the solution could reveal the initial segment, the kernels $P_h(\varphi,\cdot)$ and $P_h(\psi,\cdot)$ could then be mutually singular \citep{HMS2011}, and \Cref{lem:segment-overlap} (hence the bound \Eqref{eq:beta-segment}) could fail.
\end{remark}

\subsection{Finite-sample identifiability}
\label{subsec:finite-sample}

We now move from the population results to what can be guaranteed from a single trajectory observed on a finite grid. The argument compares the empirical excess risk to its population counterpart uniformly over the model class, and concludes with the margin $\gamma(\lambda)$ of \Cref{thm:penalized}.
We observe the stationary solution, with $X_0^{\seg}\sim\pi^{\seg}$ independent of the future Brownian motion $W$.
Write $\Prob_\pi$ and $\E_\pi$ for probability and expectation under this law, including both the initial segment and the subsequent noise.
For $Z_t:=(X_t,X_{t-\tau_1},\ldots,X_{t-\tau_q})$, let $\mu_\infty^q$ be its stationary distribution.
Thus $\occq=\mu_\infty^q$, and the signal strengths in \Eqref{eq:signal} do not depend on $T$ in this section.
The finite-energy condition in Assumption~\ref{assu:dgp} is imposed on this stationary solution:
\begin{equation}\label{eq:finite-stationary-energy}
m_G^2:=\E_\pi\norm{G(Z_0)}^2=\int\norm{G(z)}^2\,\mu_\infty^q(dz)<\infty.
\end{equation}

Observations are available on a deterministic grid $0=t_0<t_1<\cdots<t_n=T$, with $\Delta_k:=t_{k+1}-t_k$ and $\Delta_{\max}:=\max_{k<n}\Delta_k\le1$.
We assume that $\taumax$ is an observation time and that $X_{t_k-\tau_a}$ is observed exactly for every used left endpoint $t_k\ge\taumax$ and every lag $a$.
An equally spaced grid has this property when each delay is an integer multiple of its spacing.
All sums over $k$ below are over $k<n$ with $t_k\ge\taumax$, and we set
\[
Z_k:=Z_{t_k},\qquad T_\circ:=\sum_k\Delta_k=T-\taumax.
\]

\begin{manualassumption}{ass:class}{Finite-sample model class}
Let $\Hclass_n^q$ denote the restricted class of deterministic models derived from Assumption~\ref{assu:risk_minimization}(i) satisfying the following two conditions:
\begin{itemize}
    \item[(i)] \textbf{Local finite covering:} For every radius $\rho \ge 1$ and tolerance $\epsilon > 0$, we can select a finite number of representative functions from within $\Hclass_n^q$ (an internal $\epsilon$-net) such that every function in the class is within a maximum distance of $\epsilon$ from at least one representative over the compact set $K_\rho := \{z \in \R^{D(q+1)} : \norm{z - c^{(q)}} \le \rho\}$.
    The minimal number of representatives required to cover the class this way is denoted by $N_n(\rho, \epsilon)$.
    \item[(ii)] \textbf{Stationary envelope:} The pointwise diameter of the class, defined as $\mathcal{E}(z) := \sup_{\tilde{G}_1, \tilde{G}_2 \in \Hclass_n^q} \norm{\tilde{G}_1(z) - \tilde{G}_2(z)}$, has a finite second moment under the stationary measure: $m_{\mathcal{E}}^2 := \int \mathcal{E}(z)^2 \mu_\infty^q(dz) < \infty$.
\end{itemize}
\end{manualassumption}

The local covering condition makes $\mathcal E$ locally bounded, and it is measurable as a supremum of continuous functions.
Since the true drift is feasible, $\norm{\widetilde G(z)-G(z)}\le\mathcal E(z)$ for every candidate.
Thus the finite second moment of the previous assumption controls prediction errors, not the absolute size of a dissipative drift.
Any deterministic locally bounded measurable upper bound for these errors with finite stationary second moment can be used in its place.
All class-dependent quantities may depend on $n$.

Define the stationary mean-square modulus of the true drift by
\begin{equation}\label{eq:finite-drift-modulus}
\omega_G(h):=\sup_{0\le u\le h}\left(\E_\pi\norm{G(Z_u)-G(Z_0)}^2\right)^{1/2},\qquad 0\le h\le1.
\end{equation}
Continuity and \Eqref{eq:finite-stationary-energy} imply $\omega_G(h)\to0$ as $h\downarrow0$, as shown in the proof below.
By \Cref{thm:exp-mixing}, the total-variation conclusion of \Cref{thm:ergodic} and the stationary moment bound give the physical-time mixing estimate
\begin{equation}\label{eq:finite-beta}
\beta_Z(u)\le C_\beta e^{-\lambda_{\mathrm{erg}}u},\qquad u>0,
\end{equation}
where $\lambda_{\mathrm{erg}}=\gamma$ and $C_\beta=C\left(1+\int\norm{\varphi-c}_{\Cseg}^2\,\pi^{\seg}(d\varphi)\right)$ may be used, with $C,\gamma$ from \Cref{thm:ergodic}.

Consider the empirical criterion, and assume that a measurable global minimizer exists (a finite local cover does not by itself guarantee this):
\begin{equation*} 
\begin{aligned}
\eRisk_n(\widetilde G)&:=\sum_k\Delta_k\bignorm{\frac{X_{t_{k+1}}-X_{t_k}}{\Delta_k}-\widetilde G(Z_k)}^2,\\
(\widehat G,\widehat M)&\in\argmin_{(\widetilde G,\tilde M)\in\Hclass_n^q}\left\{\eRisk_n(\widetilde G)+\lambda\sum_{a=0}^q\norm{\tilde M^a}_0\right\}.
\end{aligned}
\tag{\ref{eq:empirical-risk}}
\end{equation*}
The relevant comparison is between the empirical excess risk $\eRisk_n(\widetilde G)-\eRisk_n(G)$ and
\[
\Risk_T(\widetilde G)=T_\circ\norm{\widetilde G-G}_{L^2(\mu_\infty^q)}^2.
\]

\begin{manualtheorem}{thm:finite-sample}{Finite-sample exact support recovery, $L_0$ estimator}
Let Assumptions~\ref{assu:dgp}, \ref{assu:regularity}, \ref{ass:dissip} and \ref{ass:class} hold for the stationary observations above.
Fix $\delta\in(0,1)$ and choose the deterministic quantities in \Eqref{eq:finite-choices}--\Eqref{eq:finite-entropy-choice} below.
With probability at least $1-\delta$,
\begin{equation}\label{eq:finite-uniform-deviation}
\sup_{\widetilde G\in\Hclass_n^q}\left|\bigl[\eRisk_n(\widetilde G)-\eRisk_n(G)\bigr]-\Risk_T(\widetilde G)\right|\le\varepsilon_n(\delta;\rho,\epsilon),
\end{equation}
where the explicit bound is given in \Eqref{eq:finite-error}.
Suppose additionally that $M\ne0$, $\smin>0$, and
\begin{equation} 
0<\lambda<\frac{T_\circ\smin^2}{\dmax},\qquad
\gamma(\lambda):=\min\{\lambda,T_\circ\smin^2-\lambda\dmax\}>0.
\end{equation}
If
\begin{equation}
\label{eq:condition}
\varepsilon_n(\delta;\rho,\epsilon)<\gamma(\lambda),
\end{equation}
then every measurable minimizer in \Eqref{eq:empirical-risk} satisfies
\begin{equation}\label{eq:finite-support-probability}
\Prob_\pi\!\left(\widehat M^a=M^a\text{ for all }a=0,\ldots,q,\quad
\norm{\widehat G-G}_{L^2(\mu_\infty^q)}^2\le\frac{\varepsilon_n(\delta;\rho,\epsilon)}{T_\circ}\right)\ge1-\delta.
\end{equation}
\end{manualtheorem}
The conclusion applies to both $q=0$ and $q\ge1$.
It identifies the mask exactly, not the drift exactly from finitely many observations.
With $\lambda=\lambda_0T_\circ$ for $0<\lambda_0<\smin^2/\dmax$, recovery follows whenever the displayed bound divided by $T_\circ$ is smaller than $\min\{\lambda_0,\smin^2-\lambda_0\dmax\}$.
Thus consistency requires a joint condition on the observation duration, mesh and class complexity; large $n$ alone does not suffice.

For $\rho\ge1$, define
\begin{equation}\label{eq:finite-envelope-constants}
E_\rho:=1\vee\sup_{z\in K_\rho}\mathcal E(z),\qquad
v_\rho:=\int_{K_\rho^c}\mathcal E(z)^2\,\mu_\infty^q(dz),\qquad
\sigma_{\mathrm{op}}:=\norm{\Sigma}_{\mathrm{op}},\quad
\sigma_F:=\sqrt{\tr(\Sigma\Sigma^\top)}.
\end{equation}
Here $E_\rho$ bounds errors on the truncation region and $v_\rho\to0$ is the remaining stationary squared-error mass.
For an integer $p\ge1$, let
\[
S_p:=\E_\pi\sup_{0\le u\le1}\norm{X_u^{\seg}-c}_{\Cseg}^{2p}<\infty.
\]
Finiteness follows by integrating the finite-horizon estimate of \Cref{lem:moments}(c) against $\pi^{\seg}$ and using its stationary moments.

For $\delta\in(0,1)$, choose $p\ge1$ and $\epsilon>0$, and set
\begin{equation}\label{eq:finite-choices}
\eta:=\delta/6,\qquad J:=\max\{1,\lceil T_\circ\rceil\},\qquad
\rho\ge\max\left\{1,\left(\frac{J(q+1)^pS_p}{\eta}\right)^{1/(2p)}\right\}.
\end{equation}
We define some more quantities: the block length and logarithmic complexity,
\begin{equation}\label{eq:finite-entropy-choice}
b:=\max\left\{1,\frac1{\lambda_{\mathrm{erg}}}\log\frac{C_\beta J}{\eta}\right\},\qquad
N:=N_n(\rho,\epsilon),\qquad
\Lambda:=\log\frac{4N}{\eta}.
\end{equation}

A valid bound in \Eqref{eq:finite-uniform-deviation} is then
\begin{equation}\label{eq:finite-error}
\begin{aligned}
\varepsilon_n(\delta;\rho,\epsilon)
={}&\frac{2m_{\mathcal E}T_\circ}{\eta}\,\omega_G(\Delta_{\max})\\
&+2\sigma_{\mathrm{op}}E_\rho\sqrt{2T_\circ\Lambda}
+\frac{2\epsilon\sigma_F}{\eta}\sum_k\sqrt{\Delta_k}\\
&+E_\rho^2\sqrt{2T_\circ(b+\Delta_{\max})\Lambda}
+4E_\rho\epsilon T_\circ+v_\rho T_\circ.
\end{aligned}
\end{equation}
The first line is the discretisation error.
The second line controls Brownian fluctuations and the approximation by a finite net.
The third line controls the centred time average, its net approximation and the stationary tail outside $K_\rho$.

Before giving the proof, let us describe its structure. The excess empirical risk $\eRisk_n(\widetilde G)-\eRisk_n(G)$ splits into three parts: a time average of $\norm{\widetilde G - G}^2$ along the grid, a stochastic integral against the Brownian increments, and a discretisation error due to the grid. The discretisation error is controlled by the modulus $\omega_G$ (Steps~1--2). We then truncate to the ball $K_\rho$ and replace the class by a finite net (Step~3). The time average concentrates around $\Risk_T$ thanks to the exponential $\beta$-mixing of \Cref{thm:exp-mixing}, via the independent-block argument of \citet{Yu1994} and Hoeffding's inequality (Step~4), while the Brownian term has Gaussian conditional increments and is controlled by a Chernoff bound (Step~5). A union bound over the net and the margin argument of \Cref{thm:penalized} conclude (Step~6).

\begin{proof}[Proof of \Cref{thm:finite-sample}]
\stp{1: The drift modulus needs no growth assumption}
Let $G_A$ be the projection of $G$ onto the closed Euclidean ball of radius $A$ in $\R^D$.
It is continuous and bounded, while $\norm{G-G_A}_{L^2(\mu_\infty^q)}\to0$ by \Eqref{eq:finite-stationary-energy}.
Stationarity and the triangle inequality in $L^2$ give
\[
\norm{G(Z_u)-G(Z_0)}_{L^2(\Prob_\pi)}
\le2\norm{G-G_A}_{L^2(\mu_\infty^q)}
+\norm{G_A(Z_u)-G_A(Z_0)}_{L^2(\Prob_\pi)}.
\]
For fixed $A$, continuity of the paths and bounded convergence make the last term tend to zero as $u\downarrow0$.
First choosing $A$ large and then $u$ small proves $\omega_G(h)\to0$ as $h\downarrow0$; also $\omega_G(h)\le2m_G$.

Fix a deterministic internal net $G_1,\ldots,G_N$ on $K_\rho$ and write $h=\widetilde G-G$ and $h_i=G_i-G$.
The suprema below are measurable: the continuous drifts form a separable space in the topology of uniform convergence on compact sets, and their population risks are continuous on the class by domination with $\mathcal E^2$.
One may therefore evaluate each supremum over a countable dense subset of the class.

\stp{2: Exact expansion and discretisation}
Set $\Delta W_k:=W_{t_{k+1}}-W_{t_k}$.
Inserting the SDDE increment into the square and cancelling the candidate-independent squared-increment term gives
\begin{equation}\label{eq:finite-excess-expansion}
\begin{aligned}
\eRisk_n(\widetilde G)-\eRisk_n(G)-\Risk_T(\widetilde G)
&=A_n(h)-M_n(h)-D_n(h),\\
A_n(h)&:=\sum_k\Delta_k\left(\norm{h(Z_k)}^2-\int\norm{h(z)}^2\,\mu_\infty^q(dz)\right),\\
M_n(h)&:=2\sum_k\inner{h(Z_k)}{\Sigma\Delta W_k},\\
D_n(h)&:=2\sum_k\int_{t_k}^{t_{k+1}}\inner{h(Z_k)}{G(Z_s)-G(Z_k)}\,ds.
\end{aligned}
\end{equation}
Stationarity makes the centring in $A_n$ exact; there is no further Riemann-sum error for its expectation.
For every candidate,
\[
|D_n(h)|\le\mathcal D_n:=2\sum_k\int_{t_k}^{t_{k+1}}\mathcal E(Z_k)\norm{G(Z_s)-G(Z_k)}\,ds.
\]
Cauchy--Schwarz and stationarity imply
\[
\E_\pi\mathcal D_n\le2m_{\mathcal E}\sum_k\int_0^{\Delta_k}\omega_G(u)\,du
\le2m_{\mathcal E}T_\circ\omega_G(\Delta_{\max}).
\]
Markov's inequality therefore bounds $\sup_h|D_n(h)|$ by the first line of \Eqref{eq:finite-error}, outside an event of probability at most $\eta$.
If this expectation bound is zero, $\mathcal D_n=0$ almost surely and no exceptional event is needed.

\stp{3: Localisation and the squared-loss net}
Let
\[
\mathcal L_\rho:=\left\{\sup_{\taumax\le t\le T}\norm{Z_t-c^{(q)}}\le\rho\right\}.
\]
Since $\norm{Z_t-c^{(q)}}\le\sqrt{q+1}\norm{X_t^{\seg}-c}_{\Cseg}$, a cover of $[\taumax,T]$ by $J$ intervals of length at most one, stationarity and Markov's inequality give
\begin{equation}\label{eq:finite-localisation}
\Prob_\pi(\mathcal L_\rho^c)\le\frac{J(q+1)^pS_p}{\rho^{2p}}\le\eta.
\end{equation}

Define $f_h(z):=\norm{h(z)}^2\mathbf1_{K_\rho}(z)\in[0,E_\rho^2]$ and
\[
A_n^\rho(h):=\sum_k\Delta_k\left(f_h(Z_k)-\int f_h(z)\,\mu_\infty^q(dz)\right).
\]
On $\mathcal L_\rho$, the empirical squared losses equal their truncated versions, so $|A_n(h)|\le|A_n^\rho(h)|+T_\circ v_\rho$.
For each $h$, choose a centre $h_i$ with $\sup_{K_\rho}\norm{h-h_i}\le\epsilon$.
Both errors are bounded by $E_\rho$ on this set, hence $\sup_z|f_h(z)-f_{h_i}(z)|\le2E_\rho\epsilon$ and
\begin{equation}\label{eq:finite-loss-net}
\sup_h|A_n(h)|\le\max_{i\le N}|A_n^\rho(h_i)|+4E_\rho\epsilon T_\circ+T_\circ v_\rho
\quad\text{on }\mathcal L_\rho.
\end{equation}

\stp{4: Time-average fluctuations on an irregular grid}
Partition the used left endpoints into bins $I_j=[\taumax+jb,\taumax+(j+1)b)$, for $0\le j<J_b:=\max\{1,\lceil T_\circ/b\rceil\}$, and set $w_j:=\sum_{k:t_k\in I_j}\Delta_k$.
The intervals with left endpoints in a nonempty bin are consecutive, and only the last can extend beyond its right boundary, by at most $\Delta_{\max}$.
Consequently,
\begin{equation}\label{eq:finite-bin-weights}
0\le w_j\le b+\Delta_{\max},\qquad
\sum_jw_j=T_\circ,\qquad
\sum_jw_j^2\le(b+\Delta_{\max})T_\circ.
\end{equation}
Successive nonempty bins of the same parity have their observations separated by at least $b$ in physical time.
For either parity, \Eqref{eq:finite-beta} bounds the total-variation distance between its joint block law and the product of its block marginals by $(m-1)_+C_\beta e^{-\lambda_{\mathrm{erg}}b}$, where $m$ is the number of its nonempty bins.
Indeed, successively separating the next block from all previous blocks costs at most the corresponding $\beta$ coefficient, and the triangle inequality adds these costs.
This comparison concerns the observation blocks themselves, so its error is not multiplied by the number $N$ of candidate functions.
It is the independent-block argument of \citet{Yu1994}, here applied in physical time.

Under the product law for one parity, the variables
\[
V_{j,i}:=\sum_{k:t_k\in I_j}\Delta_k\left(f_{h_i}(Z_k)-\int f_{h_i}(z)\,\mu_\infty^q(dz)\right)
\]
are independent and centred, with range length at most $E_\rho^2w_j$.
Hoeffding's inequality gives, for $x>0$,
\[
\Prob_{\mathrm{prod}}\!\left(\left|\sum_{j\text{ of one parity}}V_{j,i}\right|>x\right)
\le2\exp\left(-\frac{2x^2}{E_\rho^4\sum_jw_j^2}\right).
\]
An empty parity has sum zero and needs no bound.
Using \Eqref{eq:finite-bin-weights}, taking $x=E_\rho^2\sqrt{T_\circ(b+\Delta_{\max})\Lambda/2}$, and taking the union over the $N$ centres and two parities yield
\begin{equation}\label{eq:finite-empirical-tail}
\Prob_\pi\!\left(\max_{i\le N}|A_n^\rho(h_i)|>E_\rho^2\sqrt{2T_\circ(b+\Delta_{\max})\Lambda}\right)
\le4Ne^{-\Lambda}+J_bC_\beta e^{-\lambda_{\mathrm{erg}}b}\le2\eta.
\end{equation}
The last inequality uses $b\ge1$, hence $J_b\le J$, and \Eqref{eq:finite-entropy-choice}.
Together with \Eqref{eq:finite-loss-net}, this supplies the third line of \Eqref{eq:finite-error}.

\stp{5: Brownian fluctuations and their net approximation}
Use the predictable truncation
\[
M_n^\rho(h):=2\sum_k\inner{h(Z_k)\mathbf1_{K_\rho}(Z_k)}{\Sigma\Delta W_k}.
\]
For fixed $h_i$, conditional on the past at $t_k$, the $k$th summand is centred Gaussian with variance at most $4\sigma_{\mathrm{op}}^2E_\rho^2\Delta_k$.
Iterating the conditional moment-generating functions gives
\[
\E_\pi e^{sM_n^\rho(h_i)}\le\exp\left(2s^2\sigma_{\mathrm{op}}^2E_\rho^2T_\circ\right),\qquad s\in\R.
\]
The Gaussian Chernoff bound and a union over the fixed net imply
\begin{equation}\label{eq:finite-martingale-tail}
\Prob_\pi\!\left(\max_{i\le N}|M_n^\rho(h_i)|>2\sigma_{\mathrm{op}}E_\rho\sqrt{2T_\circ\Lambda}\right)
\le2Ne^{-\Lambda}=\eta/2.
\end{equation}
The net remainder is controlled pathwise, not by treating a data-selected remainder as a fixed martingale:
\[
|M_n^\rho(h)-M_n^\rho(h_i)|\le2\epsilon\sum_k\norm{\Sigma\Delta W_k},\qquad
\E_\pi\sum_k\norm{\Sigma\Delta W_k}\le\sigma_F\sum_k\sqrt{\Delta_k}.
\]
Markov's inequality therefore bounds the remainder simultaneously for all candidates by $2\epsilon\sigma_F\eta^{-1}\sum_k\sqrt{\Delta_k}$, outside an event of probability at most $\eta$.
On $\mathcal L_\rho$, $M_n(h)=M_n^\rho(h)$ for every $h$, so this and \Eqref{eq:finite-martingale-tail} give the second line of \Eqref{eq:finite-error}.

\stp{6: Uniform deviation and mask recovery}
The exceptional probabilities in Steps 2--5 sum to at most
\[
\eta+\eta+2\eta+\eta/2+\eta=\tfrac{11}{2}\eta<\delta.
\]
Adding the three bounds in \Eqref{eq:finite-excess-expansion} proves \Eqref{eq:finite-uniform-deviation} on the complementary event.
On that event, feasibility of $(G,M)$ and empirical optimality imply
\[
\Risk_T(\widehat G)+\lambda\sum_{a=0}^q\norm{\widehat M^a}_0
-\lambda\sum_{a=0}^q\norm{M^a}_0
\le\varepsilon_n(\delta;\rho,\epsilon).
\]
The population margin of \Cref{thm:penalized}, applied to $\mu_\infty^q=\occq$, makes the left-hand side at least $\gamma(\lambda)$ for every wrong mask.
This uses its lower bound for every feasible wrong mask, not a population-minimisation assumption on the empirical estimator.
Condition \Eqref{eq:condition} therefore excludes every wrong mask on the same event.
Once the masks agree, their penalties cancel, leaving $\Risk_T(\widehat G)\le\varepsilon_n(\delta;\rho,\epsilon)$.
Dividing by $T_\circ$ proves \Eqref{eq:finite-support-probability}.
\end{proof}

We showed in the proof that the discretisation term is $2m_{\mathcal E}T_\circ\omega_G(\Delta_{\max})/\eta$.
When the stronger estimate $\omega_G(h)\le C_G\sqrt h$ is available, Step 2 improves this term to $4m_{\mathcal E}C_G(3\eta)^{-1}\sum_k\Delta_k^{3/2}$.
It does not become $O(T_\circ\Delta_{\max})$ merely from a fourth drift moment.
On an equally spaced used grid, we have the equality $\sum_k\sqrt{\Delta_k}=T_\circ/\sqrt\Delta$, so the martingale net remainder divided by $T_\circ$ is proportional to $\epsilon/(\eta\sqrt\Delta)$.
On a general grid, the displayed sum must be retained; it is not bounded using $\Delta_{\max}$ alone.
The factors $1/\eta=6/\delta$ together with the local envelope $E_\rho$, which may grow with the radius $\rho\ge(J(q+1)^pS_p/\eta)^{1/(2p)}$ and hence with $T_\circ$, make \Eqref{eq:finite-error} a qualitative consistency statement rather than a sharp rate.


\paragraph{Optimization accuracy and sampling endpoints.}
If the returned feasible pair has a certified empirical objective gap at most $\zeta_n\ge0$, replace \Eqref{eq:condition} by $\varepsilon_n+\zeta_n<\gamma(\lambda)$ and the prediction bound by $(\varepsilon_n+\zeta_n)/T_\circ$.
This follows by adding $\zeta_n$ to the final optimality comparison; an arbitrary output of a nonconvex optimization algorithm is not automatically certified.
If $\taumax$ is not observed, let $a$ be the first used left endpoint, replace $T_\circ$ by $T-a$, and use $\Risk_{a,T}(\widetilde G):=(T-a)\norm{\widetilde G-G}_{L^2(\mu_\infty^q)}^2$ throughout this section and in its population margin.
The proof is otherwise unchanged.
Interpolated lag values require an additional error analysis and are not covered by the exact-observation statement above.

\subsection{Dissipativity and ergodicity of the learned system}
\label{subsec:learned-ergodic}

In this section, we analyze the ergodicity of both the data-generating process and the learned simulator, providing theoretical insights beyond the empirical assumptions made in our experiments. Ergodicity of the true data-generating process guarantees that a single, long trajectory visits all states in proportion to their true probabilities. This allows us to learn the invariant measure because time averages converge to ensemble averages. Conversely, ergodicity of the learned model ensures stability during forward simulation: rather than drifting indefinitely or exploding, the generated paths settle into their own stationary distribution (which ideally matches the true one). In what follows, we discuss sufficient conditions and architectural choices to guarantee the ergodicity of the learned system.

We parameterize the drift using an MLP with $\tanh$ hidden activations and an affine output layer. Provided there are no skip connections to the final layer, the output of this architecture is globally bounded. Consequently, when used as the drift in \Eqref{eq:sde_model}, the resulting process is non-explosive. However, a globally bounded drift cannot satisfy the dissipation condition of Assumption~\ref{ass:dissip} outside a compact set, nor can it uniformly approximate a drift that does over the entire domain. As a result, \Cref{thm:transfer} does not apply to an unconstrained setup. We emphasize that this highlights a limitation of the sufficient condition rather than the model itself, as bounded drifts can still be ergodic. For instance, the SDE $d X_t = -\tanh(X_t)\, dt + \sigma \, dW_t$ yields an integrable invariant density proportional to $(\cosh x)^{-2/\sigma^2}$, and its drift can be exactly represented by our unconstrained chosen architecture.

In what follows, we describe several fixes and results to maintain dissipativity and thus ergodicity of the learned drift. We chose for our experiments to apply the post-training exterior confinement defined by \Eqref{eq:post-hoc-network} which imposes no restriction to the architecture during the training and only applies a restoring force outside a compact set that encompasses all of our samples.

Although an unconstrained MLP architecture does not inherently satisfy these global dissipation conditions, we nevertheless provide a conditional ergodicity result in \Cref{thm:transfer}. This global perturbation result becomes highly relevant when the learned and true drifts exhibit compatible tail behaviors --- for instance, when they share a dissipative physical baseline and differ only by bounded residuals. Thus, we retain this theorem to address scenarios equipped with such structural priors, with the caveat that its hypotheses do not automatically emerge simply from fitting an unconstrained neural network.

\begin{theorem}[Conditional ergodicity transfer under bounded drift error]
\label{thm:transfer}
Suppose Assumptions~\ref{assu:dgp}, \ref{assu:regularity} and \ref{ass:dissip} hold for the true system, with \(B=\sum_{a=1}^q(\beta_a+\varepsilon_a)<\beta_0\). Let \(\widetilde G\) be locally Lipschitz, of at most linear growth, and satisfy
\[
 \sup_{z\in\R^{D(q+1)}}\norm{\widetilde G(z)-G(z)}\le\delta_0<\infty.
\]
For any \(0<\eta<2(\beta_0-B)\), it satisfies Assumption~\ref{ass:dissip} with the same centre and radius and with
\begin{align*}
 \widetilde\beta_0&=\beta_0-\eta/2,&
 \widetilde\alpha&=\alpha+\delta_0^2/(2\eta),&
 \widetilde\beta_a&=\beta_a,\\
 \widetilde C_R&=C_R+\delta_0^2/(2\eta)+\eta R^2/2,&
 \widetilde\varepsilon_a&=\varepsilon_a\quad(a\ge1).
\end{align*}
With the same invertible additive diffusion, the learned SDDE is non-explosive and its segment process has a unique invariant probability measure, all polynomial stationary segment moments, and geometric convergence in total variation as in \Cref{thm:ergodic}, with model-dependent constants.
\end{theorem}
\begin{proof}
Writing \(u=x_0-c\), Cauchy--Schwarz and Young's inequality give
\[
 \inner{u}{\widetilde G(z)-G(z)}
 \le\delta_0\norm u\le\frac\eta2\norm u^2+\frac{\delta_0^2}{2\eta}.
\]
Add this to the exterior and interior inequalities of Assumption~\ref{ass:dissip}; in the latter, use \(\norm u^2\le R^2\). The displayed constants follow and retain strict domination. The resulting radial bound also gives the Khasminskii condition. Local Lipschitzness gives well-posedness by localisation, and the second-moment estimate together with at-most-linear growth gives finite drift energy for square-integrable initial histories independent of the future noise. The moment and kernel-overlap arguments establishing \Cref{thm:ergodic} then apply to \(\widetilde G\).
\end{proof}

An important observation is that the constant $\delta_0$ need not be small for existence of the ergodic behavior, although increasing it can have drawbacks on statistical accuracy, and worsen moments and mixing results. Note that the result above can be strengthened into an affine error bound with sufficiently small slope, $\norm{\widetilde G(z)-G(z)}\le\delta_0 + \theta \norm{z - c^{(q)}}, \, \theta \ge 0$.

For the learned process to remain stable, we only need to bound the network's error in the outward radial direction. Large errors that act tangentially, or those that provide additional inward pull, will not cause the system to diverge. We therefore turn to conditions that can be imposed directly on the learned drift rather than on its unknown error at infinity.

\paragraph{Dissipative architecture with bounded neural residual} A simple and natural extension of the unconstrained architecture is to retain the MLP with tanh activations as nonlinear residuals and add an explicit instantaneous damping:
\begin{equation}
    \label{eq:learned-damped-architecture}
    \widetilde{G}_\theta(z) = - K (x_0 - c) + N_\theta(z), \quad K = \textrm{diag}(\kappa_1, \dots, \kappa_D), \ \kappa_i \ge \kappa_* > 0
\end{equation}

\begin{proposition}[Confinement by construction]
    \label{prop:learned-damping}
    Let \(N_\theta\) be locally Lipschitz and globally bounded, with \(\sup_z\norm{N_\theta(z)}\le M_\theta\). For any fixed delays and an invertible additive diffusion as in Assumption~\ref{assu:dgp}, the drift in \Eqref{eq:learned-damped-architecture} defines a non-explosive SDDE whose segment process has a unique invariant law, all polynomial stationary segment moments, and geometric convergence in total variation.
\end{proposition}
\begin{proof}
    For every lagged state \(z\), with \(u=x_0-c\),
    \begin{equation}
     \inner{u}{\widetilde G_\theta(z)}
     \le-\kappa_*\norm u^2+M_\theta\norm u
     \le-\frac{\kappa_*}{2}\norm u^2+\frac{M_\theta^2}{2\kappa_*}.
     \label{eq:learned-damping-bound}
    \end{equation}
    This gives Assumption~\ref{ass:dissip} with zero delayed coefficients; the interior radial bound follows by increasing its constant on any chosen ball. Local Lipschitzness and linear growth ensure the remaining well-posedness and energy conditions. The conclusion is a direct application of \Cref{thm:ergodic}.
\end{proof}

\paragraph{Post-training exterior confinement} The same idea can be applied after training in such a way that the fitted drift remains unchanged inside a radius vector $R = (R_1, \dots, R_D)^\top$ with strictly positive components, $R_j > 0$. Set $u = x_0 - c$ and define $\Pi_R(u)= (\min\{R_j, \max\{-R_j, u_j\}\})_j$ the projection onto the box $\prod_j [-R_j, R_j]$. Then, define the post-hoc network:
\begin{equation}
    \label{eq:post-hoc-network}
    \widetilde{G}_{\theta,R}(z) = N_\theta(z) - K \cdot(u - \Pi_R(u))
\end{equation}
Whenever $|x_{0,j} - c_j| \le R_j$ for every $j$, the network is left unchanged for all delayed inputs. This applies a linear restoring drift only to coordinates outside the given box.
Furthermore, since we can rearrange \Eqref{eq:post-hoc-network} as $\widetilde{G}_{\theta,R}(z) = -K \cdot u + (N_\theta(z) + K \cdot \Pi_R(u))$ and $\sup_z \norm{N_\theta(z) + K \Pi_R(u)} \le M_\theta + \norm{K \cdot R}$, \Cref{prop:learned-damping} applies. Note, nonetheless, that this changes in general the invariant law of the learned process and thus the damping matrix and radii considered should be taken into consideration and their effect on long-run statistics assessed.

The preceding results describe some constructions to exhibit the existence of the learned equilibrium, however they say little about whether or not it agrees with the true one. Under the same invertible additive noise, we can find a stationary error bound on the invariant distributions in total variation given the true invariant measure.

\begin{proposition}[Stationary error bound on invariant distributions]
    \label{prop:learned-invariant-comparison}
    Suppose the true process is stationary with segment law \(\pi^{\seg}\). Let \(\widetilde G\) be locally Lipschitz and define a non-explosive SDDE with the same lags and diffusion. Suppose its invariant law \(\widetilde\pi^{\seg}\) and kernels obey
    \begin{equation}
        \label{eq:mixing}
         \norm{\widetilde P_t(\varphi,\cdot)-\widetilde\pi^{\seg}}_{\mathrm{TV}}
         \le\widetilde C e^{-\widetilde\gamma t}(1+V(\varphi)),
         \quad V(\varphi)=\norm{\varphi-c}_{\Cseg}^2,
         \quad \pi^{\seg}V<\infty.
    \end{equation}
    where the former inequality holds when the learned drift satisfies for example Assumptions~\ref{assu:dgp}, \ref{assu:regularity} and \ref{ass:dissip} as in \Cref{thm:ergodic}.
    If \(e_\Sigma^2=\int\norm{\Sigma^{-1}(\widetilde G-G)(z)}^2\mu_\infty^q(dz)<\infty\), then for every \(t>0\), with total variation normalized as \(\sup_A|\nu(A)-\mu(A)|\),
    \begin{equation}
     \label{eq:learned-invariant-comparison}
     \norm{\pi^{\seg}-\widetilde\pi^{\seg}}_{\mathrm{TV}}
     \le \frac{e_\Sigma\sqrt t}{2}
           +\widetilde C e^{-\widetilde\gamma t}(1+\pi^{\seg}V).
    \end{equation}
\end{proposition}
\begin{proof}
    \stp{1: Common starting equilibrium law}
    Fix \(t>0\) and start both equations with an initial segment of law \(\pi\), independent of future noise. The true terminal-segment law remains \(\pi\), whereas the learned law is \(\pi \widetilde P_t\). Hence
    \begin{equation}
     \norm{\pi^\seg-\widetilde\pi^\seg}_{\rm TV}
     \le\underbrace{\norm{\pi^\seg-\pi^\seg \widetilde P_t}_{\rm TV}}_{\text{finite-time drift comparison}}
       +\underbrace{\norm{\pi^\seg \widetilde P_t-\widetilde\pi^\seg}_{\rm TV}}_{\text{learned mixing}}.
     \label{eq:triangle}
    \end{equation}
    Let \(\mathbb P\) and \(\mathbb Q\) be the true and learned laws of the full path on \([-\tau_{\max},t]\), including its initial history. On this path space set
    \[
     v(z)=\Sigma^{-1}(G(z)-\widetilde G(z)),\qquad A_s=\int_0^s\norm{v(Z_r)}^2\,dr.
    \]
    Stationarity under \(\mathbb P\) gives the essential identity
    \begin{equation}
     \E_{\mathbb P}A_t
     =\int_0^t\E_{\mathbb P}\norm{v(Z_s)}^2\,ds=t e_\Sigma^2.
     \label{eq:energy}
    \end{equation}
    \stp{2: Stopped change of drift}
    Finite mean energy does not imply a global Novikov condition. Instead, stop at
    \(\sigma_m=\inf\{s\in[0,t]:A_s\ge m\}\), with the convention \(\inf\varnothing=\infty\). Under \(\mathbb Q\), let \(\widetilde B\) be the driving Brownian motion and define
    \begin{equation}
     \frac{d\mathbb P_m}{d\mathbb Q}
     =\exp\left\{\int_0^{t\wedge\sigma_m}v(Z_s)\cdot d\widetilde B_s
                 -\frac12 A_{t\wedge\sigma_m}\right\}.
     \label{eq:girsanov}
    \end{equation}
    Because \(A_{t\wedge\sigma_m}\le m\), this exponential is a true martingale, also conditionally on the initial history. Girsanov's theorem therefore preserves the initial law \(\pi\) and makes \(\mathbb P_m\) follow drift \(G\) until \(\sigma_m\), then drift \(\widetilde G\) afterwards.
    
    Under \(\mathbb P_m\), the log-density in \Eqref{eq:girsanov} is a mean-zero stochastic integral plus \(A_{t\wedge\sigma_m}/2\). Consequently, the Kullback-Leibler divergence can be written as
    \begin{equation}
     \KL(\mathbb P_m\Vert\mathbb Q)
     =\frac12\E_{\mathbb P_m}A_{t\wedge\sigma_m}
     =\frac12\E_{\mathbb P}A_{t\wedge\sigma_m}
     \le\frac{t e_\Sigma^2}{2}.
     \label{eq:entropy}
    \end{equation}
    The second equality holds because the auxiliary and true equations agree up to the stopping time: pathwise uniqueness from \Cref{prop:strong_solution} gives identical stopped path laws. This is why the error remains evaluated under the \emph{true} law.
    
    If we couple these two equations with the same history and Brownian motion, their paths can differ only if \(\sigma_m\le t\), so
    \[
     \norm{\mathbb P-\mathbb P_m}_{\rm TV}
     \le\mathbb P(A_t\ge m)\le\frac{t e_\Sigma^2}{m}.
    \]
    Pinsker's inequality and \Eqref{eq:entropy} now imply
    \[
     \norm{\mathbb P-\mathbb Q}_{\rm TV}
     \le\frac{t e_\Sigma^2}{m}
         +\sqrt{\frac12\KL(\mathbb P_m\Vert\mathbb Q)}
     \le\frac{t e_\Sigma^2}{m}+\frac{e_\Sigma\sqrt t}{2}.
    \]
    Let \(m\to\infty\) with \(t\) fixed. Extracting a terminal segment cannot increase total variation, and its laws under \(\mathbb P,\mathbb Q\) are \(\pi^\seg,\pi^\seg \widetilde P_t\), respectively. Thus
    \begin{equation}
     \norm{\pi^\seg-\pi^\seg \widetilde P_t}_{\rm TV}\le\frac{e_\Sigma\sqrt t}{2}.
     \label{eq:finite}
    \end{equation}
    
    \stp{3: Conclusion}
    Integrating \Eqref{eq:mixing} against its initial law \(\pi^\seg\) gives
    \[
     \norm{\pi^\seg \widetilde P_t-\widetilde\pi^\seg}_{\rm TV}
     \le\int\norm{\widetilde P_t(\varphi,\cdot)-\widetilde\pi^\seg}_{\rm TV}\,\pi^\seg(d\varphi)
     \le \widetilde C e^{-\widetilde\gamma_t}(1+\pi^\seg V).
    \]
    Together with \Eqref{eq:finite} and \Eqref{eq:triangle}, this proves \Eqref{eq:learned-invariant-comparison}.
\end{proof}

\section{Comparison with adaptive group lasso assumptions}
\label{subsec:assumptions_bellot}

\citet{bellot2022ngm} prove local consistency of the \emph{adaptive group lasso}, for Markovian systems. They show similar identifiability proofs using the adaptive group lasso to select input features. AGL also requires an initial consistent estimator to build the adaptive weights. 
In this work, we derive identifiability guarantees for SDDE-driven systems, which extend \citet{bellot2022ngm} results. Additionally, we use the $L_0$ penalty instead of AGL. The $L_0$ route is in two ways
cleaner: no initial estimator is needed, and the population theorem (\ref{thm:penalized}) already supplies an explicit
$\lambda$ window. 
Both results suppose an optimization oracle, which is standard in the literature, but something that is most often not available in practice. 

Our theorems assume global minimization of a nonconvex, combinatorial objective.
\citet{bellot2022ngm} assume that they can find exact solutions to their problem, which is least squares over a \emph{neural network} class with an AGL penalty, and hence also nonconvex. Both results presuppose an
optimization oracle --- theirs for a continuous nonconvex problem, ours for a combinatorial one ---.
Empirically, we find that even on Markovian systems (part of the CausalDynamics benchmark, the Dysts benchmark), the relaxed \Lzero penalty performs better than the AGL-penalized objective. 
Our minimum-signal condition $\smin > 0$ plays exactly the role of their eq.~(10), and our
effective sample size $\lambda_{\mathrm{erg}} T$ plays the role of their $n/\norm{\alpha}_2$.

\section{Implementation details and hyperparameters}
\label{sec:hyperparams}

\subsection{Implementation details}
\label{sec:implementation_details}

The \Lzero-norm is not differentiable. To allow for score-based optimization, we implement the relaxed $L_0$ penalization proposed by \citet{louizos2018l0}. More precisely, the matrix M is a square logits matrix parametrized as the sigmoid of differentiable parameters in $\mathbb{R}$. During training, at each batch, we draw samples of M through a hard-concrete distribution, a continuous relaxation of a binary random variable. We use a low temperature parameter $\beta = 0.33$ to sharpen the hard-concrete distribution and force to turn inputs on or off. Samples are drawn on the $[-0.1, 1.1]$ interval before being clamped to $[0, 1]$ i.e. most samples will be exactly equal to 0 or 1. At inference, we fix M to be 0 or 1 if the probabilities (obtained from the logits) are lower or higher than 0.5. To have a smoothed convergence, we draw multiple samples of M at each batch before optimizing for its differentiable parameters. 
We also warm-up the model and train the neural network without penalty for 100 warm-up iterations before adding the penalty to the loss. We then train until the loss plateaus and does not decrease for 200 patience iterations. These parameters and optimization procedure are standard when using \Lzero-penalty \citep{louizos2018l0}.

The functions $f$ are small multi-layer perceptrons (MLPs) and we use a gradient penalty to enforce that our functions are Lipschitz. To make sure that smooth dimensions e.g. with close to 0 gradient are not approximated by the mean and then do not have any drivers, we renormalize each dimension by the gradient mean and standard deviation. 

C-NODE follows a similar implementation, where the $L_0$ is replaced by \Lone norm. In practice, we just apply a \Lone penalty on the logits matrix $M$. AGL implementation is similar, but uses an AGL penalty instead. We find that using a gradient penalty slightly improve the results for both, so we included it in the method's implementation. Except the penalty, everything is kept the same between methods. 

\section{CausalDynamics benchmark} 
\label{sec:full_results}

\subsection{Metric computation}
\label{subsec:metric_computation}

\Lzero-NDDE and ``No Grad. Pen.'' use a logit matrix to represent probabilities of inputting a feature or not. It is thus straightforward to calculate SHD, AUROC and AUPRC by comparing this matrix to the ground truth logits matrix. AGL and \Lone penalty instead assign weights to each input feature. While these weights can still be compensated for by the neural network, we used this weight matrix to compute AUROC / AUPRC, treating the weights as feature importance. For SHD instead, we used a low threshold ($10^{-8}$) to decide whether an input feature is turned off or not. Changing this threshold did not affect performance, as we observed that the weights were generally either 0 or above $10^{-3}$ in practice. 

\subsection{Hyperparameter search}
\label{subsed:hyperparam_search}

The causal methods evaluated in CausalDynamics have various numbers of hyperparameters. While methods such as PCMCI rely crucially on a single parameter e.g. the conditional independence test threshold, others such as TSCI or NGC depend on many parameters e.g. neural network architectures. \citet{herdeanu2025causaldynamicslargescalebenchmarkstructural} thus decided to keep default parameters except the main parameter and perform hyperparameter search on this parameter only. We do the same and keep default parameters, highlighted in \Cref{subsec:final_hyperparams}, except for the sparsity penalty coefficient, which controls for the final level of sparsity.

For all methods, we run our model on the 10 time series of the Lorenz84 system with the following penalty coefficients and select the one maximizing AUROC: (0.0001, 0.00025, 0.0005, 0.001, 0.0025, 0.005, 0.01, 0.025, 0.05, 0.1, 0.25, 0.5, 1, 2.5, 5, 10, 25, 50, 100).

\subsection{Final hyperparameters}
\label{subsec:final_hyperparams}

\begin{table*}[hbt!]\centering
\begin{tabular}{c c}

\toprule
\textbf{Parameter} & \textbf{Value} \\
\midrule
\multicolumn{2}{c}{\textbf{Architecture, functions $f_j$}} \\
\midrule

Number of layers & 2 \\
Number of units per layer & 8 \\
\midrule
\multicolumn{2}{c}{\textbf{Optimization parameters}} \\
\midrule
Gradient penalty coefficient & 100 \\
Warm-up iterations & 100 \\
Patience iterations & 200 \\
Learning rate & 0.001 \\
Batch size & 128  \\
\midrule
\multicolumn{2}{c}{\textbf{$L_0$ implementation parameters} (\Lzero-NDDE, No grad. pen.)} \\
\midrule
Temperature $\beta_{\text{temp}}$ & 0.33 \\
Approximate distribution interval & $[-0.1, 1.1]$ \\
Number of samples & 3 \\
\midrule
\multicolumn{2}{c}{\textbf{Sparsity penalty coefficients}} \\
\midrule
$L_0$ & 0.025 \\
AGL & 5 \\
C-NODE (L1) & 50 \\
No grad. pen. & 0.01 \\

\end{tabular}
\caption{\textbf{Final hyperparameters used for \Lzero-NDDE, C-NODE, AGL and No grad. pen.} The last 4 lines show the penalty coefficient for each method, as it is the only parameter tuned for each the method. Note that No grad. pen. does not use gradient penalty.}
\label{tab:hyperparams_causaldynamics}
\end{table*}

\subsection{Full CausalDynamics results}

\begin{table*}[h!]\centering
\resizebox{\textwidth}{!}{
\begin{tabular}
{c c ccc ccccc cc c}
\toprule
& \multirow{2}{*}{\textbf{Experiment}}  & \multicolumn{3}{c}{\textbf{Simple}} & \multicolumn{5}{c}{\textbf{Coupled}} & \multicolumn{2}{c}{\textbf{Climate}} & \multirow{2}{*}{\textbf{Avg. rank}} \\
\cmidrule(lr){3-5} \cmidrule(lr){6-10} \cmidrule(lr){11-12}
& & Default & Confounder & Noise & Default & Noise & Confounder & Time-lag & Standardize & MAOOAM & ENSO \\
\midrule

 \multirow{13}{*}{$\vcenter{\hbox{\rotatebox{90}{\textbf{SHD}}}}$}
& PCMCI+ & 41.04 & 23.02 & 47.64 & 224.80 & 183.90 & 324.63 & 327.72 & 228.32 & 80.00 & 529.36 & 7.30 \\
& F-PCMCI & 35.30 & 21.07 & 45.09 & 192.90 & \underline{149.60} & 195.74 & 350.61 & \underline{201.79} & 130.00 & 530.27 & 5.80 \\
& VARLiNGAM & 35.69 & 22.04 & 42.84 & 311.45 & 248.90 & 159.63 & 449.33 & 349.63 & 130.00 & 453.00 & 8.00 \\
& DYNOTEARS & 52.37 & 21.74 & 51.68 & 181.50 & 180.95 & 248.32 & 261.44 & 243.84 & 94.00 & 589.36 & 6.60 \\
& NGC & 28.91 & 19.96  & 28.84 & 840.95 & 842.55 & 670.53 & 793.67 & 840.26 & \underline{31.00} & 337.09 & 7.90 \\
& TSCI & 52.50 & 22.02 & 60.50 & 174.70 & 173.50 & 265.26 & 244.56 & 244.42 & 108.00 & 666.27 & 7.20 \\
& CUTS+ & 48.11 & 24.04 & 61.32 & \textbf{152.00} & 150.50 & 272.68 & 247.22 & 310.63 & 130.00 & 608.73 & 7.50 \\
& RCD & 61.85 & 26.74 & 61.38 & \underline{157.05} & 155.65 & \underline{136.53} & \textbf{201.11} & \textbf{159.84} & 130.00 & 665.36 & 6.80 \\
& GRaSP & 59.04 & 27.09 & 60.18 & 215.94 & 842.55 & \textbf{136.00} & 793.67 & 840.26 & 126.00 & 666.27 & 9.70 \\
& TCDF & 59.46 & 26.61 & 61.43 & 252.00 & 842.55 & 670.53 & 793.67 & 840.26 & 130.00 & 666.27 & 11.65 \\
 & AGL & 26.60 & \underline{16.96} & \underline{21.28} & 319.04$^*$ & 222.41$^*$ & 246.37$^*$ & 315.54$^*$ & 312.31$^*$ & \textbf{30.00}$^*$ & 337.09 & 5.55 \\
 & C-NODE (L1) & \textbf{24.07} & 17.10 & \textbf{20.5} & 253.91$^*$ & 176.30$^*$ & 179.58$^*$ & 273.57$^*$ & 258.09 & 42.00 & \underline{335.36} & \underline{4.50} \\
& \textbf{\Lzero-NDDE (Ours)} & \underline{26.42} & \textbf{16.25} & 23.67 & 205.35 & \textbf{144.20} & 143.33 & \underline{222.34} & 212.11 & 40.00 & \textbf{325.54} & \textbf{2.50} \\

\midrule

 \multirow{13}{*}{$\vcenter{\hbox{\rotatebox{90}{\textbf{AUROC / AUPRC} }}}$} 

& PCMCI+ & .52 / .71  & .49 / .59 & .50 / .69 & \textbf{.67} / .25 & \underline{.64} / .25 & .58 / .20 & .58 / .24 & \textbf{.69} / .27 & .69 / .88 & \underline{.57 / .70} & 4.10 / 5.25 \\
& F-PCMCI & .51 / .70 & .50 / .59  & .52 / .70 & \textbf{.67} / .27 & .57 / .21 & .55 / .19 & \underline{.59} / .24 & \underline{.68} / .28 & .50 / .81 & \underline{.57 / .70} & 4.80 / 6.05 \\
& VARLiNGAM & .50 / .69 & .48 / .58 & .53 / .70 & .60 / .19 & .57 / .18 & .51 / .17 & .54 / .22 & .60 / .19 & .50 / .81 & .56 / .69 & 6.90 / 8.55 \\
& DYNOTEARS & .43 / .67 & .52 / .64 & .48 / .68 & .59 / .21 & .57 / .20 & .49 / .17 & .53 / .22 & .65 / .23 & .64 / .86 & .55 / .69 & 7.75 / 7.95 \\
& NGC & .50 / .69 & .50 / .58  & .50 / .68 & .50 / .15 & .50 / .15 & .50 / .16 & .50 / .20 & .50 / .15 & .50 / .81 & .50 / .67 & 9.85 / 10.90 \\
& TSCI & .46 / .68 & .53 / .66 & .49 / .68 & .60 / .23 & .53 / .17 & .51 / .18 & .53 / .21 & .65 / .23 & .58 / .84 & .50 / .67 & 7.50 / 8.00 \\
& CUTS+ & .50 / .69 & .50 / .58 & .50 / .68 & .50 / .15 & .50 / .15 & .49 / .16 &  .50 / .20 &  .50 / .15 &  .50 / .81 & .50 / .67 & 10.10 / 10.90 \\
& RCD & .50 / .69 & .50 / .58 & .50 / .68 & .50 / .15 & .50 / .15 & .51 / .18 & .50 / .20 & .50 / .16 &  .50 / .81 &  .50 / .67 & 9.55 / 10.25 \\
& GRaSP & .52 / .71 & .55 / .64 & .49 / .68 & .49 / .15 & .50 / \textbf{.50} & .50 / .17 & .50 / \textbf{.50} & .50 / \textbf{.50} & .48 / .81 & .50 / .50 & 9.60 / 6.70 \\
& TCDF & .51 / .70 & .50 / .58 & .50 / .68 & .48 / .15 & .50 / \textbf{.50} & .50 / \textbf{.50} & .50 /  \textbf{.50}& .50 / \textbf{.50} & .50 / .81 &  .50 / .50 & 9.85 / 6.65 \\
 & AGL & .52 / .71$^*$ & \underline{.56} / .64 & .53 / .74$^*$ & .53 / .33$^*$ & .53 / .33$^*$ & .54 /  .35$^*$ &.52 / .35$^*$ & .54 / .34$^*$ & .56 / .86$^*$ & .54 / .69$^*$ & 5.80 / 4.35 \\
  & C-NODE (L1) & .64 / \underline{.79} & .54 / \underline{.65} & \underline{.61 / .79}$^*$ & .64 / \textbf{.43} & .63 / .43 & \textbf{.67} / \underline{.47} & \underline{.59} / .43 & \underline{.68} / .47 & \textbf{.88 / .97} & .49 / .66$^*$ & \underline{3.45 / 3.35} \\
& \textbf{\Lzero-NDDE (Ours)} & \textbf{.64 / .80} & \textbf{.58 / .66} & \textbf{.63 / .80} & .64 / \textbf{.43} & \textbf{.65} / .44 & \textbf{.67} / \underline{.47} & \textbf{.60} / .43 & .66 / \underline{.45} & \underline{.86 / .96} & \textbf{.58 / .75} & \textbf{1.75 / 2.10} \\
\bottomrule
\end{tabular}
}
\caption{\textbf{\Lzero-NDDE outperforms competing methods most of the time on the CausalDynamics benchmark.} This table is similar to \Cref{tab:causaldynamics_results}  with additional methods being shown.}
\label{tab:full_causal_dyn_new}
\end{table*}

\newpage

\subsection{No gradient penalty}

\begin{table*}[h!]\centering
\begin{footnotesize}
\resizebox{\textwidth}{!}{
\begin{tabular}{c|c|ccc|ccccc|cc|}

& \multirow{2}{*}{Experiment}  & \multicolumn{3}{c|}{Simple} & \multicolumn{5}{c|}{Coupled} & \multicolumn{2}{c|}{Climate} \\
& & Default & Confounder & Noise & Default & Noise & Confounder & Time-lag & Standardize & MAOOAM & ENSO \\
\hline

 \multirow{2}{*}{$\vcenter{\hbox{{SHD}}}$}
 & No Grad. Pen. & 26.95 & 18.21 & 24.47 & 215.08 & 142.67 & 145.08 & 233.58 & 220.68 & 49.00 & 341.27 \\
& {\Lzero-NDDE (Ours)} & 26.42 & {16.25} & 23.67 & 205.35 & {144.20} & 143.33 & 222.34 & 212.11 & 40.00 & {325.54} \\

\hline


AUROC / & No Grad. Pen. & {.64 / .80} & {.58} / .65 & .61 / .79 & .63 / .42 & .64 / .44 & {.67 / .47} & .59 / .41 & .65 / .44 & .85 / .96 & .52 / .68 \\
AUPRC & {\Lzero-NDDE (Ours)} & {.64 / .80} & {.58 / .66} & {.63 / .80} & .64 / {.43} & {.65} / .44 & {.67 / .47} & {.60 / .43} & .66 / .45 & .86 / .96 & {.58 / .75} \\
\end{tabular}}
\end{footnotesize}

\caption{{\Lzero-NDDE is slightly stronger with gradient penalty.} Gradient penalty helps avoid overfitting spurious correlations and leads to slightly better performance on the CausalDynamics benchmark.}
\label{tab:causaldynamics_results_full}
\end{table*}


\section{Dysts benchmark}

\subsection{Train and test systems}
\label{subsec:hyperparam_search_dysts}

We use 40 systems in total, split into a hyperparameter search set (5 systems), and a test set (the other 35 systems). For each system, we generate one training trajectory, and 3 test trajectories. The train trajectories last 50 Lyapunov times and the test trajectories last 250 Lyapunov times, to ensure that the trajectory visits the entire space. Trajectories are sampled at the native temporal resolution. To avoid too short or too long trajectories in training, we enforce the training trajectory to have a minimum of 1000 and a maximum of 10000 timesteps. We use a stochastic forcing of level 0.05 times the variance of each dimension of the training trajectory. Each test and training trajectory is normalized by subtracting the mean and dividing by the training trajectory standard deviation per dimension. 

We studied 40 systems in total, which are the 40 Dysts systems for which the Jacobian is accessible. For these 40 systems, a method is implemented to access the Jacobian at any point along any trajectory. This is necessary to compute metrics. Among these 40 systems, we randomly chose 5 systems for doing hyperparameter search, and 35 systems for testing. On the 5 hyperparameter search systems, we ran the methods with the following sets of hyperparameters, and selected the ones that maximized the valid prediction time on the three test trajectories.

For all methods, we wait a certain number of patience iterations (200) to declare convergence i.e. when the penalized loss does not improve for 200 iterations, we stop sparsifying. After this, we fix the input mask i.e. the learned matrices $\tilde{M}$, and then keep training the neural networks to learn the dynamics for a certain number of iterations \texttt{n\char`_iter\char`_after\char`_convergence}. Also, we wait a number \texttt{niter\char`_warmup} of warm-up iterations before sparsifying.

We chose the best values among the following values for the following parameters:

\begin{itemize}
    \item Sparsity penalty coefficient: $[0.01, 0.05, 0.1, 0.5, 1, 5, 10, 50, 100]$
    \item Gradient penalty coefficient: $[0, 10, 100, 1000, 10_000]$
    \item \texttt{n\char`_iter\char`_after\char`_convergence}: $[0, 50, 100, 200, 500, 1000]$
    \item \texttt{niter\char`_warmup}: $[0, 25, 50, 100]$
\end{itemize}

For all methods, these four parameters are key as they control the sparsity and the regularization of the methods and their final values can be found in \Cref{tab:hyperparams_dysts}. All other parameters are kept the same as in the CausalDynamics experiments (\Cref{tab:hyperparams_causaldynamics}).

\begin{table*}[hbt!]\centering
\begin{tabular}{c c c c c c}

\toprule
Optimization parameters & \Lzero-NDDE & NODE & C-NODE & AGL & PathReg \\
\midrule
Gradient penalty coefficient & 0 & 100 & 0 & 1000 & 1000 \\
\texttt{niter\char`_warmup} & 0 & 0 & 100 & 50 & 200 \\
\texttt{n\char`_iter\char`_after\char`_convergence} & 200 & 0 & 500 & 1000 & 500 \\
Sparsity penalty coefficients & 0.05 & 0 & 0.05 & 50 & 1 \\
\bottomrule
\end{tabular}
\caption{\textbf{Final hyperparameters used for \Lzero-NDDE, C-NODE, AGL and PathReg}}
\label{tab:hyperparams_dysts}
\end{table*}

The systems chosen randomly to perform hyperparameter search are the following: 
MooreSpiegel,
Lorenz,
Laser,
RikitakeDynamo,
Chua

The systems used for evaluating the methods are the following: 
BurkeShaw,
Chen,
ChenLee,
Coullet,
DequanLi,
Duffing,
Finance,
GenesioTesi,
Hadley,
JerkCircuit,
KawczynskiStrizhak,
LiuChen,
Lorenz84,
LuChen,
LuChenCheng,
PanXuZhou,
PehlivanWei,
QiChen,
RayleighBenard,
Rucklidge,
Sakarya,
SanUmSrisuchinwong,
ShimizuMorioka,
SprottA,
SprottB,
SprottC,
SprottE,
SprottI,
SprottTorus,
Thomas,
ThomasLabyrinth,
Tsucs2,
WangSun,
YuWang,
YuWang2,
ZhouChen

\subsection{Metrics}
\label{subsec:dysts_metric_computation}

As we evaluate driver identifiability on the CausalDynamics benchmark, we are now interested in evaluating the dynamics. More precisely, we want to compare methods on short-term horizon prediction performance, attractor and geometry reconstruction, and long-term statistics. At evaluation, we predict autoregressively 3 trajectories from the initial conditions of the 3 test conditions (only the initial conditions are given), and generate trajectories of the same length as the test trajectories. We then compare the predicted vs true trajectories and compute several metrics.  

For short-term prediction, we compute the valid prediction time at 1, normalized by the first Lyapunov exponent (valid prediction lyapunov time, VLPT) i.e. we compute the number of Lyapunov timesteps before the normalized trajectory is further than 1 away from the true normalized trajectory. We also compute the normalized root mean square error (NRMSE), by averaging the RMSE across the first 5 Lyapunov timesteps before normalizing by the trajectory mean and standard deviation. 

For geometry and attractor reconstruction, we compute the first Lyapunov exponent error (FLEE) i.e. the squared difference between the first Lyapunov exponent of the learned vs. true system. The Lyapunov exponents are computed using the Benettin algorithm \citep{benettin_algorithm}, which requires access to the Jacobian which is possible in the selected Dysts systems, and for the learned models which are by construction differentiable. This error is small when the chaoticity of the system is well approximated. We additionally compute the Kaplan-Yorke dimension error ($D_{KY}$, \citealp{dky_error}). Computed using the Lyapunov exponents, the $D_{KY}$ gives us the dimension of the attractor i.e. the fractal dimension, which would be 2 in the case of the classic Lorenz system. The reported error is the squared error between the estimated system vs. true system's $D_{KY}$.

We then evaluate whether the learned system correctly captures long-term statistics, in other words whether the predicted trajectories converge, and if so whether they converge to the right equilibrium distribution. We report the log-spectral distance (LSD), which is the squared difference between the true and estimated log power spectral density, to indicate whether the system captured the right oscillation and periodicity properties of the system. We additionally report the $D_{stsp}$, a state-space divergence which estimates the KL divergence between the learned and true system invariant distribution from a set of predicted and true trajectories \citep{dstsp}. We compute it by placing a Gaussian kernel at each trajectory point (for both true and generated trajectories) and estimating the KL divergence between the two Gaussian Mixture Models via Monte Carlo sampling.

\subsection{Results}
\label{subsec:dysts_results_figure}

On top of the quantitative results shown in the paper, we show, in \Cref{fig:full_dysts_results} the predicted (red) vs true (black) trajectories for 26 systems, for the three best methods, \Lzero-NDDE, NODE and C-NODE.

\begin{figure}[htbp]
    \centering
    \includegraphics[width=\textwidth]{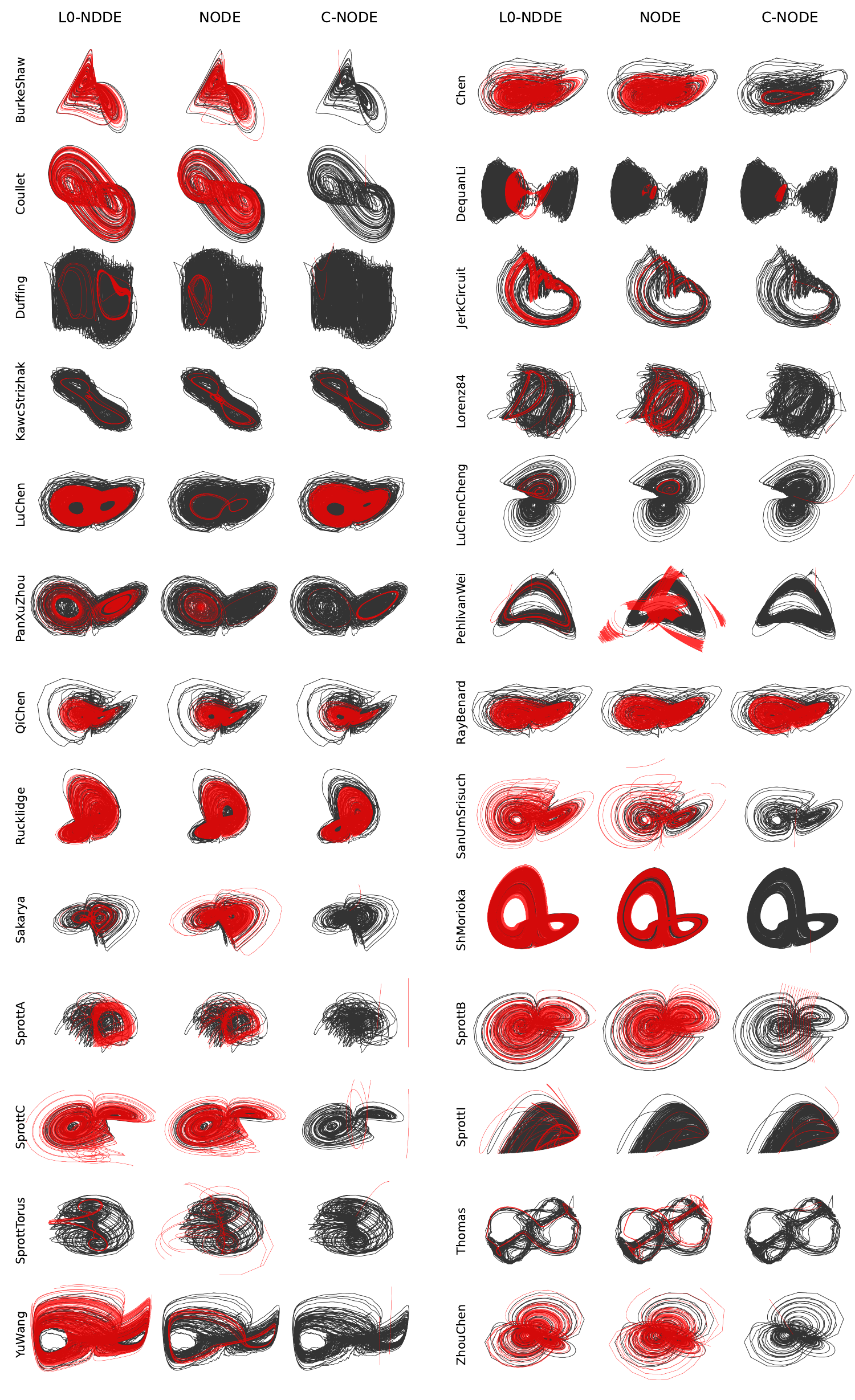}
    \vspace{-2em}
    \caption{Example predicted trajectories (red) vs. true trajectories (black) on 26 Dysts systems, for \Lzero-NDDE (left), NODE (middle), and C-NODE (right).}
    \label{fig:full_dysts_results}
\end{figure}

We then analyze the performance of the different methods on this benchmark, in \Cref{fig:dysts_statistical_association}.

\begin{figure}[htbp]
    \centering
    \includegraphics[width=\textwidth]{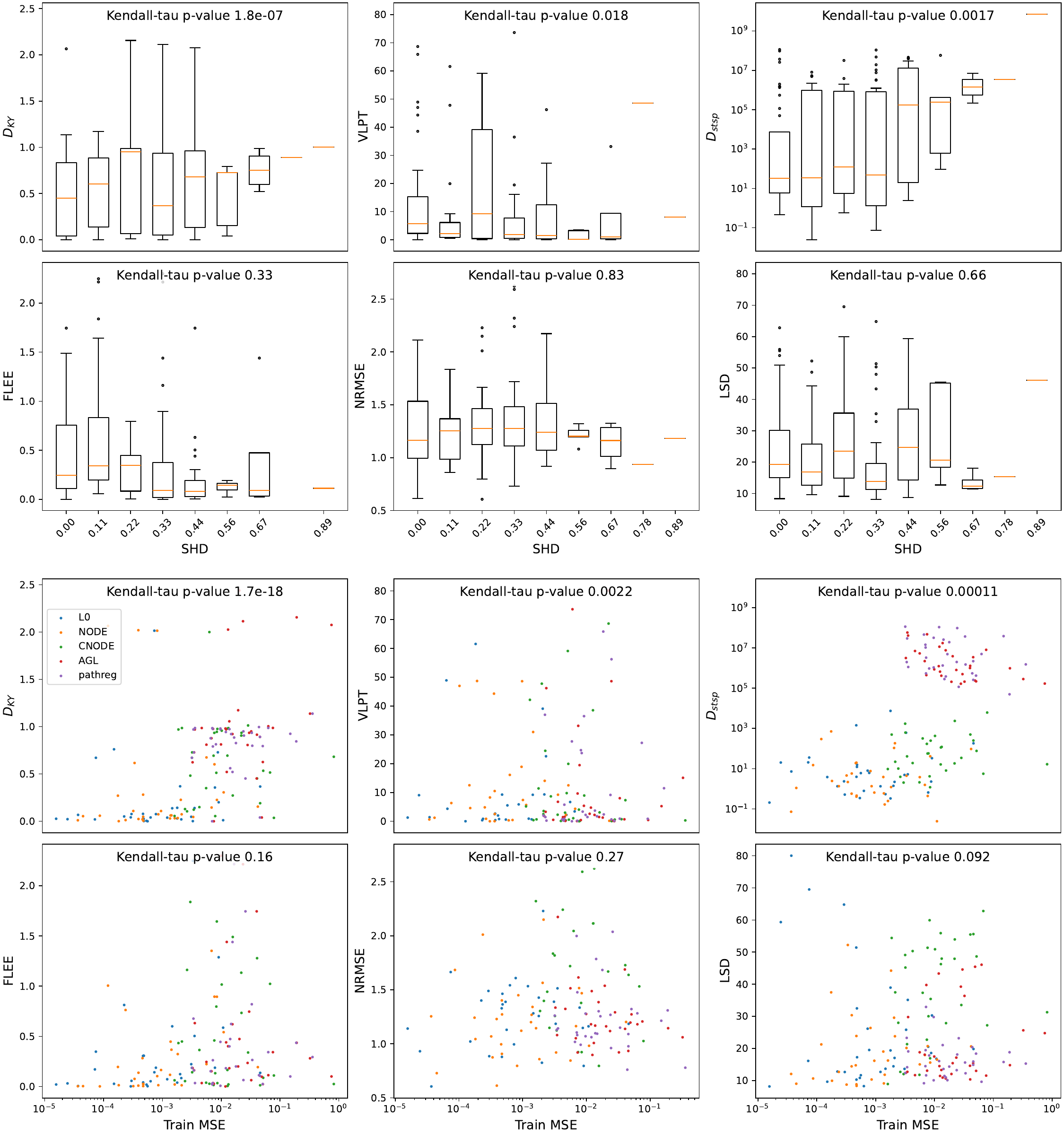}
    \vspace{-2em}
    \caption{\textbf{Lower training error and lower SHD are both associated with higher performance.} We perform an analysis to better understand why \Lzero-NDDE and NODE outperform other methods on the Dysts datasets. We compute, for all methods and all datasets, the training error (Train MSE) and SHD. We then plot the 6 metrics versus the SHD (upper half) and versus the training error (lower half). We see that $D_{KY}$ and $D_{stsp}$ increase, while VLPT decreases with SHD and training error. These relationships are statistically significant according to a Kendall-$\tau$ test, a standard rank-correlation test for investigating associations between two sets of values, without assuming Gaussianity of the samples or needing large samples. The relationship is also significant between LSD and training error. This indicates that a lower training error is associated with better reconstruction, indicating why NODE performs well, as the absence of input feature regularization allows it to overfit to the training data; and why \Lzero-NDDE outperforms other methods, as it achieves lower SHD. Moreover, the \Lzero-penalty has the advantage that it does not perform weight decay, but explicitly turns features on or off, allowing the rest of the network to approximate the dynamics very well given the set of selected features.}
    \label{fig:dysts_statistical_association}
\end{figure}

\end{document}